\documentclass[letterpaper]{article} % DO NOT CHANGE THIS
\usepackage{aaai23}  % DO NOT CHANGE THIS
\usepackage{times}  % DO NOT CHANGE THIS
\usepackage{helvet}  % DO NOT CHANGE THIS
\usepackage{courier}  % DO NOT CHANGE THIS
\usepackage[hyphens]{url}  % DO NOT CHANGE THIS
\usepackage{graphicx} % DO NOT CHANGE THIS
\usepackage[sorting = none, backend = bibtex, style=numeric-comp,maxbibnames=99,firstinits=true]{biblatex}
\usepackage{caption} % DO NOT CHANGE THIS AND DO NOT ADD ANY OPTIONS TO IT
\usepackage{microtype}
\usepackage{graphicx}
\usepackage{subfigure}
\usepackage{booktabs}
\usepackage{multirow}
\usepackage{makecell}
\usepackage{enumitem}
\usepackage{float}
\usepackage{amsmath}
\usepackage{amssymb}
\usepackage{mathtools}
\usepackage{amsthm}
\usepackage{amsfonts}
\usepackage{dsfont}
\usepackage{color}
\usepackage{hyperref}
\usepackage[linesnumbered, ruled, vlined]{algorithm2e} % For algorithms

\usepackage{algorithmicx}

\theoremstyle{plain}
\newtheorem{theorem}{Theorem}

\newtheorem{lemma}[theorem]{Lemma}
\newtheorem{corollary}{Corollary}
\theoremstyle{definition}

\newtheorem{assumption}{Assumption}
\theoremstyle{remark}

\DeclareMathOperator*{\argmax}{arg\,max}
\DeclareMathOperator*{\argmin}{arg\,min}

\usepackage{newfloat}
\usepackage{listings}
\DeclareCaptionStyle{ruled}{labelfont=normalfont,labelsep=colon,strut=off} % DO NOT CHANGE THIS
\floatstyle{ruled}
\newfloat{listing}{tb}{lst}{}
\floatname{listing}{Listing}
\title{User-Oriented Robust Reinforcement Learning}
\author{
    Haoyi You,
    Beichen Yu,
    Haiming Jin\thanks{Corresponding author.},
    Zhaoxing Yang,
    Jiahui Sun
}
\affiliations{
 Shanghai Jiao Tong University, Shanghai, China\\
    \{yuri-you, polarisybc, jinhaiming, yiannis, jhsun1997\}@sjtu.edu.cn
}

\usepackage{bibentry}
\begin{document}

\maketitle
% \vspace{-100cm}

\begin{abstract}
\vspace{-0.05cm}
Recently, improving the robustness of policies across different environments attracts increasing attention in the reinforcement learning (RL) community. Existing robust RL methods mostly aim to achieve the max-min robustness by optimizing the policy's performance in the worst-case environment. However, in practice, a user that uses an RL policy may have different preferences over its performance across environments. Clearly, the aforementioned max-min robustness is oftentimes too conservative to satisfy user preference. Therefore, in this paper, we integrate user preference into policy learning in robust RL, and propose a novel \textit{User-Oriented Robust RL (UOR-RL)} framework. Specifically, we define a new \textit{User-Oriented Robustness (UOR)} metric for RL, which allocates different weights to the environments according to user preference and generalizes the max-min robustness metric. To optimize the UOR metric, we develop two different UOR-RL training algorithms for the scenarios with or without a priori known environment distribution, respectively. Theoretically, we prove that our UOR-RL training algorithms converge to near-optimal policies even with inaccurate or completely no knowledge about the environment distribution. Furthermore, we carry out extensive experimental evaluations in 6 MuJoCo tasks. The experimental results demonstrate that UOR-RL is comparable to the state-of-the-art baselines under the average-case and worst-case performance metrics, and more importantly establishes new state-of-the-art performance under the UOR metric. 
\end{abstract}

\vspace{-0.3cm}
\section{Introduction}
Recently, reinforcement learning (RL) raises a high level of interest in both the research community and the industry due to its satisfactory performance in a variety of decision-making tasks, such as playing computer games \cite{mnih2013playing}, autonomous driving \cite{kiran2021deep}, robotics \cite{kober2013reinforcement}. Among existing RL methods, model-free ones, such as DQN \cite{mnih2013playing}, DDPG \cite{silver2014deterministic}, PPO \cite{schulman2017proximal}, which typically train policies in simulated environments, have been widely studied. However, it is highly possible that there exist discrepancies between the training and execution environments, which could severely degrade the performance of the trained policies. Therefore, it is of significant importance to robustify RL policies across different environments. 
%improve the robustness of 
%game of Go \cite{,silver2016mastering}. 
%However, in many applications of RL, due to perturbation or dynamic or uncertain environment, it is frequent that the trained policy becomes unstable in the real scenarios. Therefore, enhancing robustness of the policy is coming into view and becoming a major goal in RL region. 

Existing studies of RL robustness against environment discrepancies \cite{RobustMDP,rajeswaran2016epopt,action_robust_1,curi2021combining} mostly aim to achieve the max-min robustness by optimizing the  performance of the policy in the worst-case environment. However, such max-min robustness could oftentimes be overly conservative, since it only concentrates on the performance of the policy in the worst case, regardless of its performance in any other case. As a matter of fact, it is usually extremely rare for the worst case (e.g., extreme weather conditions in autonomous driving, power failure incidence in robot control) to happen in many applications. Therefore, we should take the environment distribution into consideration and pay more attention to the environments with higher probabilities though they are not the worst.

% , such as the extreme weather conditions in autonomous driving, power failure incidence in robot control.
%the chance of the worst case, e.g., the extreme weather conditions in autonomous driving or power failure coincidence in robot controlling, may be extremely small, and even equal to zero.
%sometimes not the most ideal metric 
%focus on the worse-case environment, regarding improving the performance of the policy under such environment as their metrics (max-min metric)  .  
%The max-min metric is widely-used in robustness RL but have two main drawbacks, which are concerned in our paper. 
%One of the main drawbacks is that the max-min metric is commonly overly conservative in RL, 
%which makes the max-min robustness useless in real applications, .

Besides, a user that uses an RL policy for real-world tasks may have different preferences over her performance across environments. Furthermore, for the same decision-making task, the preferences of different users may vary, resulting that they are in favor of different policies. For instance, in computer games, some users prefer to attack others and take an aggressive policy, while others may prefer a defensive policy. Therefore, user preference is a crucial factor that should be considered in RL policy training, which is ignored by the max-min robustness. 
%as for the other and the most significant drawback, the max-min metric ignores the preference of the user, which usually varies in different decision-making problems. For instance, in aircraft controlling, for safety consideration, the user cares much about the worst case, while in gaming processes, the user may focus on those worst 20\% to 30\% cases rather than worst one, under such processes the max-min robustness is unable to satisfies the user's preference. 

Due to the significance of user preference and environment distribution, we design a new \textit{User-Oriented Robustness (UOR)} metric, which integrates both user preference and environment distribution into the measurement of robustness. Specifically, the UOR metric allocates different weights to different environments, based on user preference, environment distribution, and relative policy performance. In fact, the max-min robustness is a special case of the UOR metric that the user prefers extremely conservative robustness, and thus, the UOR metric is a generalization of the max-min robustness. Hence, in this paper, we focus on optimizing the UOR metric during policy training, which can help obtain policies better aligned with user preference.
% The UOR metric is of great facilitation to help train the policy fitting user preference. 
% Owing to the defects of current robustness RL, we propose User-Oriented Robustness (UOR) framework, which firstly integrate user's preference into robustness consideration. The UOR metric $\mathcal{E}$ considers environment parameter distribution and user's preference, and  
% Based on such robustness, we define a new User-Oriented Robustness parameterized MDP (UOR-PMDP) based parameterized MDP, which proposes a new UOR metric $\mathcal{E}$. 

To optimize the UOR metric, we propose the \textit{User-Oriented Robust RL (UOR-RL)} framework. One of the training algorithms of the UOR-RL framework, namely \textit{Distribution-Based UOR-RL (DB-UOR-RL)}, takes the environment distribution as input to help optimize the UOR metric. In real-world applications, however, the environment distribution may sometimes be unknown to the user. To tackle such case, we design another training algorithm, namely \textit{Distribution-Free UOR-RL (DF-UOR-RL)}, which works even without any knowledge of the environment distribution. Both algorithms evaluate the UOR metric and use it to update the policy, while they differ in their approaches for UOR metric evaluation, because of the different prior knowledge of the environment distribution.
%As whether with the prior knowledge of the environment parameter distribution varies in different applications, the UOR-RL framework contains two training algorithms, namely \textit{Distribution-Based UOR-RL (DB-UOR-RL)} and \textit{Distribution-Free UOR-RL (DF-UOR-RL)}, to tackle two cases that with and without distribution, respectively.  
%propose two \textit{User-Oriented Robust RL (UOR-RL)} training algorithms, namely \textit{Distribution-Based UOR-RL (DB-UOR-RL)} and \textit{Distribution-Free UOR-RL (DF-UOR-RL)}, for the cases with and without apriori known environment parameter distribution, respectively. Both two algorithms mainly aim at evaluating the UOR metric and use it to update the policy, while the two algorithms differ in evaluation the UOR metric because of the different knowledge of environment parameter distribution, .
%Moreover, we provide sufficient analysis of the two algorithms. 

Theoretically, under several mild assumptions, we prove that UOR-RL guarantees the following series of desirable properties. For DB-UOR-RL, we prove that with $O(\frac{1}{\epsilon^d})$ computational complexity, where $d$ denotes the dimension of the parameter that parameterizes the environment, the output policy of DB-UOR-RL is $\epsilon$-suboptimal to the optimal policy under the UOR metric. Furthermore, even when DB-UOR-RL takes an inaccurate empirical environment distribution as input, we prove that, as long as the total variation distance between the empirical distribution and the accurate one is no larger than $O(\epsilon^d)$, the output policy of DB-UOR-RL is still guaranteed to be $\epsilon$-suboptimal to the optimal policy. For DF-UOR-RL, though without any prior knowledge of the environment distribution, our proof shows that DF-UOR-RL could still generate an $\epsilon$-suboptimal policy with $O(\frac{1}{\epsilon^{2d+4}})$ computational complexity.
%we provide theorem to guarantee that the output policy is $\epsilon$-suboptimal compared to the optimal one with computational complexity not exceeding $\frac{1}{\epsilon^{2d+4}}$. 
%theoretical results about UOR-RL. 
%Under several mild assumptions, we have following theoretical results. 
%However, considering that in the real application of DB-UOR-RL Algorithm, the distribution function may be inaccurate, we also prove that as long as the distance between inaccurate distribution and accurate one is not larger than $O(\frac{1}{\epsilon}$, the output 
%policy is still $\epsilon$-optimal to the optimal one. 
%Theoretically, we prove the theorems to guarantee near-optimality of the output policy from the two algorithms. Besides, in the real application, we also provide the theorem and corollary to guarantee near-optimality with the inaccurate distribution  
%In another aspect, we carry out experiments to evaluate our algorithms in robot controlling environment, MuJoCo, and compare our algorithm to the state-of-the-art baselines. The experimental results are promising that for the traditional metrics, average and worst-case metrics, our algorithms is comparable to state-of-the-art metrics, and for UOR metric, our algorithm establishes new state-of-the-art performance. Furthermore, our experiments compare the policy performance of different user's preference, which will be really benefits to help users to set there preference.

% Main contributions in the paper are summarized as follows.
The contributions of this paper are summarized as follows.
% Our main contributions are summarized as follows.
\begin{itemize}[leftmargin=*]
    \item We propose a user-oriented metric for robustness measurement, namely UOR, allocating different weights to different environments according to user preference. To the best of our knowledge, UOR is the first metric that integrates user preference into the measurement of robustness in RL.
    \item We design two UOR-RL training algorithms for the scenarios with or without a priori known environment distribution, respectively. Both algorithms take the UOR metric as the optimization objective so as to obtain policies better aligned with user preference. 
    \item We prove a series of results, through rigorous theoretical analysis, showing that our UOR-RL training algorithms converge to near-optimal policies even with inaccurate or entirely no knowledge about the environment distribution. 
    \item We conduct extensive experiments in 6 MuJoCo tasks. The experimental results demonstrate that UOR-RL is comparable to the state-of-the-art baselines under the average-case and worst-case performance metrics, and more importantly establishes new state-of-the-art performance under the UOR metric.
    \end{itemize}
%\paragraph{Establish User-Oriented Robustness Reinforcement Learning Framework:} 
%After analyzing the works of robustness reinforcement learning, we propose User-Oriented Robustness Reinforcement Learning, which is the first to consider user's preference in Robustness RL. To formulate UOR, we define UOR-PMDP as a new decision-making problem and propose a new UOR metric measure UOR.  
%\paragraph{Propose UOR RL training Algorithms:}
%To solve UOR-PMDP, we propose two RL training algorithms, Distribution-Based UOR RL algorithm and Distribution-Free UOR RL algorithms, which are respectively used for the case with and without environment distribution.
%\paragraph{Theoretically Guarantee the near-optimality of Algorithms:} In theoretical aspect, we discuss the performance of our algorithms in three different situations, with accurate, inaccurate and no environment apriori distribution. Under such three situations, we prove the corresponding theorems and corollary to guarantee the near-optimality of our algorithms. 
%\paragraph{Practical Implementation and Evaluation of DB-UOR-RL and DF-UOR-RL Algorithm} In experimental aspect, we conduct our two Algorithms in robot controlling tasks, MoJuKo, comparing to state-of-the-art works. For the traditional RL metrics, our algorithms achieve the state-of-the-art performance and in UOR metrics, our algorithms exceeds performance of recent works and establish new state-of-the-art performance.

\section{Problem Statement}
\subsection{Preliminary}
\label{Preliminary}
We introduce parameterized Markov Decision Process (PMDP) \cite{rajeswaran2016epopt} represented by a 6-tuple $(\mathcal{S},\mathcal{A},\gamma,\mathbb{S}_{0},T,R)$, as it is the basis of the UOR-PMDP that will be defined in Section~\ref{subsection UOR-PMDP}. $\mathcal{S}$ and $\mathcal{A}$ are respectively the set of states and actions. $\gamma$ is the discount factor. $\mathbb{S}_0\in \Delta(\mathcal{S})$ denotes the initial state distribution\footnote{We use $\Delta(\mathcal{X})$ to denote the set of all distributions over set $\mathcal{X}$.}. Furthermore, different from the traditional MDP, a PMDP's transition function $T:\mathcal{S}\times\mathcal{A}\times\mathcal{P}\rightarrow\Delta(\mathcal{S})$ and reward function $R:\mathcal{S}\times \mathcal{A}\times\mathcal{S}\times\mathcal{P}\rightarrow\Delta(\mathbb{R})$ take an additional environment parameter $p$ as input, with $\mathbb{R}$ denoting the set of all real numbers. The environment parameter $p$ is a random variable in range $\mathcal{P}\subset \mathbb{R}^d$, following a probability distribution $\mathbb{D}$. In PMDP, the parameter $p$ is sampled at the beginning of each trajectory, and keeps constant during the trajectory. We consider the scenario where $p$ is unknown to the user during execution, but our work can extend to the scenario with known $p$ through regarding the environment parameter $p$ as an additional dimension of state.
% Furthermore, we consider the scenario where $p$ is known to the user during execution, as such parameter could oftentimes be readily observed in practice\footnote{For example, the value of the friction in robot control can be measured by specific sensors, and the weather conditions can be easily obtained in autonomous driving.}.

% As the parameter $p$ is known to the user, we choose to take $p$ as an input of the policy that $\pi:\mathcal{S}\times\mathcal{P}\rightarrow\Delta(\mathcal{A})$ and let $\Pi$ denote the set of all policies. 
A policy $\pi:\mathcal{S}\rightarrow\Delta(\mathcal{A})$ is a mapping from a state to a probability distribution over actions.
Furthermore, the expected return $J(\pi,p)$ is defined as the expected discounted sum of rewards when the policy $\pi$ is executed under the environment parameter $p$, i.e., 
\begin{equation}
J(\pi,p)=\mathbb{E}\left[\sum_{t=0}^\infty \gamma^t R(s_{t},a_t,s_{t+1},p)|\pi\right].
\end{equation}

\subsection{User-Oriented Robust PMDP}
\label{subsection UOR-PMDP}
\subsubsection{Definition}
%In this paper, we propose the concept of \textit{User-Oriented Robustness (UOR)} which allocates different weights to the environment parameters according to user preference.
%puts a high focus on those environment parameters where the agent performs poorly, and the degree of the focus is determined by user preference. 
%allocates more weights to the environment parameters where agent performs more poorly, according to user preference.
%allocates different weights to the environment parameters according to user preference and the performance of the agent under the parameter. , which considers UOR over PMDP

In this paper, we propose \textit{User-Oriented Robust PMDP (UOR-PMDP)}, which is represented by an 8-tuple $(\mathcal{S},\mathcal{A},\gamma,\mathbb{S}_{0},T,R,h,W)$. The first six items of UOR-PMDP are the same as those of PMDP. Furthermore, we introduce the last two items, namely the ranking function $h$ and the preference function $W$, to formulate our new \textit{User-Oriented Robustness (UOR)} metric that allocates more weights to the environment parameters where the policy has relatively worse performance, according to user preference.
%allocates weights to the environment parameters according to user preference over the relative performance of the policy under the environment parameters.

As UOR requires assessing the relative performance of the policy under different environment parameters, we define the \textit{ranking function} $h:\Pi\times\mathcal{P}  \rightarrow [0,1]$ as 
\begin{equation}
    h(\pi,p)=\int_{\mathcal{P}}\mathbb{D}(p)\cdot\mathds{1}\left[J(\pi,p')\le J(\pi,p)\right]\mathrm{d}p',
\end{equation}
which represents the probability that the performance of policy $\pi$ under an environment parameter $p'$ sampled from $\mathbb{D}$ is worse than that under the environment parameter $p$. 

To represent a user's preference, we define the \textit{preference function} $W:[0,1]\rightarrow \mathbb{R_+}\cup\{0\}$ which assigns weights to environment parameters with different rankings. Specifically, given $\pi$, the weight assigned to environment parameter $p$ is set as $W(h(\pi,p))$. Moreover, we require function $W$ to be non-increasing, since UOR essentially puts more weights on environment parameters with lower rankings. 

In practice, to make it more convenient for a user to specify her preference, we could let the preference function $W$ belong to a family of functions parameterized by as few as only one parameter. For example, by setting the preference function $W$ as
 \begin{equation}
    W(x)=(k+1)\cdot (1-x)^k,
    \label{fomula robust weight function}
\end{equation}
a single robustness degree parameter $k\in \mathbb{R_+}$ suffices to completely characterize the user preference.

In terms of the objective, the UOR-PMDP aims to maximize the \textit{UOR metric} $\mathcal{E}$ defined as
\begin{equation}
    \mathcal{E}(\pi)=\int_\mathcal{P}\mathbb{D}(p)\cdot J(\pi,p)\cdot W(h(\pi,p))\mathrm{d}p, 
\end{equation}
which is the expectation of the weighted sum of $J(\pi,p)$ over the distribution $\mathbb{D}$. That is, the optimal policy $\pi^\ast$ of the UOR-PMDP satisfies
\begin{equation}
    \pi^\ast=\argmax_{\pi\in \Pi}\mathcal{E}(\pi).
\end{equation}
\subsubsection{Properties}
In fact, our UOR metric generalizes the worst-case performance robustness metric $\min_{p\in\mathcal{P}} J(\pi,p)$ \cite{RobustMDP,rajeswaran2016epopt,action_robust_1,curi2021combining}, and the average-case metric $\mathbb{E}_{p\sim \mathbb{D}}\left[J(\pi,p)\right]$ without robustness consideration \cite{dasilva2012learning,tobin2017domain}.

As $W$ is non-increasing, it has the following two extreme cases. For one extreme case, $W$ concentrates on zero. That is, $W(x)=\delta(x)$, where $\delta$ denotes the Dirac function, and consequently the UOR metric becomes 
\begin{equation*}
    \mathcal{E}(\pi)=\int_\mathcal{P}\mathbb{D}(p)\cdot J(\pi,p)\cdot\delta(h(\pi,p))\mathrm{d}p=\min_{p\in \mathcal{P}} J(\pi,p).
\end{equation*}
For the other extreme case, $W$ is uniform in $[0,1]$. That is, $W(x)\equiv1$, and consequently the UOR metric becomes 
\begin{equation*}
    \mathcal{E}(\pi)=\int_\mathcal{P}\mathbb{D}(p)\cdot J(\pi,p)\mathrm{d}p=\mathbb{E}_{p\sim \mathbb{D}} \left[J(\pi,p)\right].
\end{equation*}

\section{Solutions for UOR-PMDP}
In this section, we present two \textit{User-Oriented Robust RL (UOR-RL)} algorithms to solve UOR-PMDP in the scenarios with and without a priori known environment parameter distribution, respectively. 
\subsection{Distribution-Based UOR-RL}
\subsubsection{Algorithm Design}
We consider the scenario that the distribution $\mathbb{D}$ is known before training, and propose the \textit{Distribution-Based UOR-RL (DB-UOR-RL)} training algorithm in Algorithm \ref{Algorithm with model} which makes use of the distribution $\mathbb{D}$ during the training period.

\begin{algorithm}[t]%\vspace{10pt}
\caption{DB-UOR-RL Algorithm}\label{Algorithm with model}
	\small
% 	\noindent 	
	\tcp{Initialization.}
	Initialize policy $\pi_{\theta^0}$ and block diameter upper bound $\delta$\;\label{a1l1}
	$\{\mathcal{P}_1,\mathcal{P}_2,\cdots,\mathcal{P}_n\}\leftarrow \text{Set\_Division}(\mathcal{P},\delta)$\;\label{a1l2}
	\ForEach{Block $\mathcal{P}_j$\label{a1l3}}{
	$p_j\leftarrow$ Arbitrarily chosen element in $\mathcal{P}_j$\;\label{a1l4}
	$m_j\leftarrow\int_{\mathcal{P}_j}\mathbb{D}(p)\mathrm{d}p$\;\label{a1l5}
	}
	\tcp{Policy Training.}
	\ForEach{Iteration $i=0$ to max-iterations\label{a1l6}}{
	\ForEach{Block $\mathcal{P}_j$\label{a1l7}}{

	Execute policy
	$\pi_{\theta^i}$ under $p_j$ and evaluate the empirical expected return $\hat{J}(\pi_{\theta^i},p_j)$\;\label{a1l8} 
	}
	Sort the sequence $\{\hat{J}(\pi_{\theta^i},p_j)\}$ into an increasing sequence  $\{\hat{J}(\pi_{\theta^i},p_{\alpha_j})\}$\;\label{a1l9}
	\tcp{Metric Calculation.}
    Initialize metric $\hat{\mathcal{E}}(\pi_{\theta^i})\leftarrow0$ and  $M\leftarrow0$\;\label{a1l10}
    \ForEach{Block $\mathcal{P}_j$\label{a1l11}}
    {
    $w_j\leftarrow\int_{M}^{M+m_{\alpha_j}} W(x)\mathrm{d}x$\;\label{a1l12}
    $\hat{\mathcal{E}}(\pi_{\theta^i})\leftarrow\hat{\mathcal{E}}(\pi_{\theta^i})+w_j\cdot \hat{J}(\pi_{\theta^i},p_{\alpha_j})$\;\label{a1l13}
    $M\leftarrow M+m_{\alpha_j}$\;\label{a1l14}
    }
    \tcp{Policy Update.}
    $\pi_{\theta^{i+1}}\leftarrow
    $ Policy\_Update$(\pi_{\theta^i},\hat{\mathcal{E}}(\pi_{\theta^i}))$\;\label{a1l15}
	}
\end{algorithm}

Firstly, Algorithm \ref{Algorithm with model} randomly sets the initial policy $\pi_{\theta^0}$, and chooses the upper bound $\delta$ of the block diameter (line~\ref{a1l1}). The criteria for setting $\delta$ will be discussed in detail in Section \ref{section 4.1.2}. Then, by calling the Set\_Division algorithm, Algorithm \ref{Algorithm with model} divides $\mathcal{P}$ into $n$ blocks $\mathcal{P}_1,\mathcal{P}_2,\cdots,\mathcal{P}_n$, whose diameters are less than $\delta$ (line~\ref{a1l2}). That is,
\begin{equation*}
\forall~\mathcal{P}_j\in \{\mathcal{P}_1,\mathcal{P}_2,\cdots,\mathcal{P}_n\},\forall~p_x,p_y\in \mathcal{P}_j, ||p_x-p_y||_2\le \delta.
\end{equation*}
Note that the number of blocks $n$ is decided by how the Set\_Division algorithm divides $\mathcal{P}$ based on $\delta$. Because of space limit, we put our Set\_Division algorithm in Appendix \ref{set division function}. In fact, Algorithm \ref{Algorithm with model} works with any Set\_Division algorithm that could guarantee that the diameters of the divided blocks are upper bounded by $\delta$. Then, for each block $\mathcal{P}_j$, Algorithm \ref{Algorithm with model} arbitrarily chooses an element $p_j$ from the block to represent it (line ~\ref{a1l4}), and calculates the probability that an environment parameter falls into the block $\mathcal{P}_j$ (line ~\ref{a1l5}). 

Next, Algorithm \ref{Algorithm with model} trains the policy (line~\ref{a1l6} to~\ref{a1l15}). In each iteration $i$, it evaluates the performance $\hat{J}(\pi_{\theta^i},p_j)$ of the policy $\pi_{\theta^i}$ under each $p_j$ (line~\ref{a1l8}), and sorts the sequence $\{\hat{J}(\pi_{\theta^i},p_j)\}$ into an increasing one $\{\hat{J}(\pi_{\theta^i},p_{\alpha_j})\}$ (line \ref{a1l9}). Then, Algorithm \ref{Algorithm with model} calculates the metric $\hat{\mathcal{E}}(\pi_{\theta^i})$, which is an approximation of the UOR metric $\mathcal{E}(\pi_{\theta^i})$ (line \ref{a1l10} to~\ref{a1l14}). Specifically, it initializes the metric $\hat{\mathcal{E}}(\pi_{\theta^i})$ and the lower limit $M$ of the integral as zero (line \ref{a1l10}). For each block $\mathcal{P}_j$, it calculates the weight $w_j$ allocated to this block based on the ranking $\alpha_j$ of block $\mathcal{P}_j$ in the sorted sequence $\{J(\pi_{\theta^i},p_{\alpha_j})\}$ and preference function $W$ (line~\ref{a1l12}), and updates the metric $\hat{\mathcal{E}}(\pi_{\theta^i})$ (line~\ref{a1l13}) and the lower limit of the integral (line~\ref{a1l14}). Finally, based on the metric $\hat{\mathcal{E}}(\pi_{\theta^i})$, Algorithm \ref{Algorithm with model} updates the policy by applying a Policy\_Update algorithm (line~\ref{a1l15}). Note that Policy\_Update could be any policy gradient algorithm that updates the policy based on the metric $\hat{\mathcal{E}}(\pi_{\theta^i})$. 
%In detail, it approximate the ranking function $h$ from the probability in these blocks and the ranks of the performance of the represent parameter, and using the approximated ranking to count the weight (line~\ref{a1l12}). 

The above Algorithm~\ref{Algorithm with model} essentially uses integral discretization to calculate an approximate UOR metric $\hat{\mathcal{E}}$ which is used as the optimization objective of Policy\_Update. To discretize the integral for calculating the UOR metric, Algorithm~\ref{Algorithm with model} divides the environment parameter range into blocks. Furthermore, to get the ranking function for weight allocation, Algorithm~\ref{Algorithm with model} sorts the blocks according to the evaluated performance on them. 
%, and to get the ranking function for weight allocation, Algorithm~\ref{Algorithm with model} sorts the blocks according to the evaluated performance on them. 
%sums up the values of the environment parameters that represent these blocks to get the metric $\hat{\mathcal{E}}$
% There are two main difficulties in solving the UOR-PMDP, which are well solved by the Algorithm~\ref{Algorithm with model}. 
%One of difficulty is how to calculate the UOR metric $\mathcal{E}$,an integral form and it is well solved by integral discretization in the Algorithm~\ref{Algorithm with model}. 
%The algorithm divides the environment into blocks (line~\ref{a1l2}-\ref{a1l5}) and sums up the values in these blocks and get the metric $\hat{\mathcal{E}}$ to approximate the integration on the whole $\mathcal{P}$ (line~\ref{a1l11}-\ref{a1l14}). The other difficulty is to get the ranking function to allocate the weights, which is solved by  (line~\ref{a1l7}-\ref{a1l9}).

\subsubsection{Algorithm Analysis}
\label{section 4.1.2} 
% Considering that in the real world, environment and MDP do not change abruptly, it is natural to assume the continuity between the MDP and environment. Besides, we also assume the transition and reward function is also continuous. 
To analyze Algorithm~\ref{Algorithm with model}, we make the following three mild assumptions.
\begin{assumption}
The transition function $T$ and reward function $R$ are continuous to the environment parameter $p$. \label{Assumption p continue}
\end{assumption}
\begin{assumption}
The transition function $T$ and reward function $R$ are Lipschitz continuous to the state space $\mathcal{S}$ and action space $\mathcal{A}$ with constants $L_{T,\mathcal{S}}$, $L_{T,\mathcal{A}}$, $L_{R,\mathcal{S}}$, and $L_{R,\mathcal{A}}$, respectively. \label{Assumption R,T continue}
\end{assumption}
\begin{assumption}
The policy $\pi$ during the training process in Algorithm \ref{Algorithm with model} is Lipschitz continuous with constant $L_\pi$.\label{Assumption Pi continue}
\end{assumption}
% That $R,T$ are Lipschitz to $s,a$ is mild and has been used in previous works\cite{curi2020efficient,combining}. The continuity of $R,T$ to $p$ is also very common in the real-world cases, such as the change of the moisture and temperature in robot controlling.
Assumption \ref{Assumption p continue} is natural, because the $R$ and $T$ functions characterize the environment which will usually not change abruptly as the environment parameter $p$ changes. Furthermore, Assumptions \ref{Assumption R,T continue} and \ref{Assumption Pi continue} are commonly used in previous works \cite{curi2020efficient,curi2021combining}. 
%Considering that in the real world, as the environment parameter $p$ changes, the actual environment does not change abruptly. As the environment is characterized by $R$ and $T$ functions, 
Based on Assumptions~\ref{Assumption p continue}-\ref{Assumption Pi continue}, we prove the following Theorem \ref{Theorem near optimal}, which demonstrates the existence of the diameter upper bound $\delta$ under which Algorithm \ref{Algorithm with model} can converge to a near-optimal policy.

\begin{theorem}
\label{Theorem near optimal}
$\forall$ optimality requirement $\epsilon=2\epsilon_0>0$, $\exists~\delta_0$, such that as long as Policy\_Update can learn an $\epsilon_0$-suboptimal policy for metric $\hat{\mathcal{E}}$, by running Algorithm~\ref{Algorithm with model} with any diameter upper bound $\delta\le \delta_0$, we can guarantee that the output policy $\hat{\pi}$ of Algorithm~\ref{Algorithm with model} satisfies
\begin{equation}
\mathcal{E}(\hat{\pi})\ge \mathcal{E}(\pi^\ast)-\epsilon.
\end{equation}
\end{theorem}

Because of space limit, the proofs to all of the theorems and corollary in this paper are provided in the appendix. 

Theorem~\ref{Theorem near optimal} reveals that as long as the diameter upper bound $\delta$ is sufficiently small, the output policy $\hat{\pi}$ of Algorithm~\ref{Algorithm with model} will be close enough to the optimal policy $\pi^\ast$. However, as $\delta$ decreases, the number of the blocks output by Set\_Division on line \ref{a1l2} of Algorithm \ref{Algorithm with model} will increase, leading to an increased complexity of Algorithm \ref{Algorithm with model}. Thus, it is of great importance to have a quantitative relationship between $\epsilon$ and $\delta$, which could help us better choose the upper bound $\delta$ based on the optimality requirement $\epsilon$. To obtain the quantitative relationship between $\delta$ and $\epsilon$, we introduce the following Assumption \ref{assumption lipschitz of p}, which is stronger than Assumption \ref{Assumption p continue}. 
%However, although Theorem~\ref{Theorem near optimal} proves the existence of $\delta$, it does not provide the quantitative relationship between $\epsilon$ and $\delta$. 
%Due to the imperfection of it, we require to strength the theorem~\ref{Theorem near optimal} and provide the quantitative relationship between $\delta$ and $\epsilon$. 
%In order to strengthen the theorem and apply the algorithm~\ref{Algorithm with model} more conveniently, we strengthen assumption~\ref{Assumption Continuous of p} with the following assumption  
\begin{assumption}
The transition function $T$ and reward function $R$ are Lipschitz continuous to the environment parameter $p$ with constants $L_{T,p}$ and $L_{R,p}$. 
\label{assumption lipschitz of p}
\end{assumption}
%The assumption~\ref{assumption lipschitz of p} strengthen the continuity into Lipschitz one for environment parameter $p$, under which we can strength the theorem~\ref{Theorem near optimal} and provide a quantitative relation between the limitation $\delta$ and near-optimal requirement $\epsilon$. 
Based on Assumptions \ref{Assumption R,T continue}-\ref{assumption lipschitz of p}, we prove Theorem~\ref{Theorem near optimal strengthem}. 
\begin{theorem}
$\forall$ optimality requirement $\epsilon=2\epsilon_0>0$, $\exists~\delta_0= \upsilon\epsilon=O(\epsilon)$, 
% \begin{equation}
%     \upsilon=\frac{L_{r,s} L_{t,p}+L_{r,p}}{1-\gamma}+\mu\{L_{r,s}+L_{r,a} L_{\pi}+L_{r,s} (L_{t,s}+L_{t,a} L_\pi)\}
% \end{equation}
% with 
% \begin{equation}
%         \mu=\max\left\{\begin{aligned}&\frac{L_{t,p}}{((L_{t,s}+L_{t,a} L_\pi)-1)(1-\gamma (L_{t,s}+L_{t,a} L_\pi))},\\&\frac{L_{t,p}\gamma (L_{t,s}+L_{t,a} L_\pi)}{(1-\gamma)^2},\\&\frac{L_{t,p}}{(1-(L_{t,s}+L_{t,a} L_\pi)) (1-\gamma)}\end{aligned},\right.
% \end{equation}
such that as long as Policy\_Update can learn an $\epsilon_0$-suboptimal policy for metric $\hat{\mathcal{E}}$, by running Algorithm~\ref{Algorithm with model} with any diameter upper bound $\delta\le \delta_0$, we can guarantee that the output policy $\hat{\pi}$ of Algorithm~\ref{Algorithm with model} satisfies
\begin{equation}
\mathcal{E}(\hat{\pi})\ge \mathcal{E}(\pi^\ast)-\epsilon.
\end{equation}
Note that the constant $\upsilon$ depends on the Lipschitz constants in Assumptions \ref{Assumption R,T continue}-\ref{assumption lipschitz of p}, whose detailed form is presented in Equation (\ref{detailed equation}) in Appendix~\ref{Proof of Theorem 2}.

\label{Theorem near optimal strengthem}
\end{theorem}
Theorem~\ref{Theorem near optimal strengthem} indicates that when Algorithm~\ref{Algorithm with model} chooses $\delta=\upsilon\epsilon=O(\epsilon)$, the number of divided blocks is at most $O(\frac{1}{\epsilon^d})$. Therefore, with such choice of $\delta$, we can guarantee that the complexity of each iteration in Algorithm~\ref{Algorithm with model} is at most $O(\frac{1}{\epsilon^d})$. %The number of the iterations in the inner loop will be at most  $O(\frac{1}{\epsilon^d})$.

% Theorem~\ref{Theorem near optimal},~\ref{Theorem near optimal strengthem} and algorithm~\ref{Algorithm with model} providing a near optimal solution, though satisfying in experiments, the users are hardly get the accurate distribution function $\mathbb{D}$ but only an empirical distribution $\mathbb{D}^e$ in the real applications. In those common cases, the output policy of the algorithm with input distribution of $\mathbb{D}^e$ has the bias comparing to the optimal ones when executed in actual environment distribution $\mathbb{D}$.

% However, the distribution $\mathbb{D}^e$ are usually derived from the historical data or some oracles, albeit not precise, satisfying facts that the prediction of the real $\mathbb{D}$ would not be so absurd, the distance between real $\mathbb{D}$ and empirical $\mathbb{D}^e$ being small. Here we measure the distance of two distinct distribution using \textbf{1-norm} that
% \begin{equation*}
% \begin{aligned}
%     \forall \mathbb{D},\mathbb{E},~ \mathcal{D}_{TV}(\mathbb{D},\mathbb{E})=\int_\mathcal{P}\left|\mathbb{D}(p)-\mathbb{E}(p)\right|\mathrm{d}p
% \end{aligned}
% \end{equation*}

% Assume running algorithm ~\ref{Algorithm with model} with the empirical distribution $\mathbb{D}^e$ and obtain the policy $\hat{\pi}^e$, then we have the following theorem to bound the performance of the policy applying on the actual environment
In practice, the user may not know the accurate distribution $\mathbb{D}$, but only has access to a biased empirical distribution $\mathbb{D}^e$. In the following Theorem~\ref{Theorem Empirical}, we prove the theoretical guarantee of Algorithm~\ref{Algorithm with model} when it runs with $\mathbb{D}^e$.
%, the ser is only able to apply the Algorithm~\ref{Algorithm with model} with the empirical distribution. 
%In that case, we prove the following Theorem~\ref{Theorem Empirical} to guarantee that the Algorithm~\ref{Algorithm with model} still works.
\begin{theorem}
\label{Theorem Empirical}
Define the policy $\pi^e$ such that
\begin{equation}
    \pi^e=\argmax_{\pi}\mathcal{E}^e(\pi),
\end{equation}
where %$\mathcal{E}$ denotes UOR-PMDP metric under the accurate distribution $\mathbb{D}$ and 
$\mathcal{E}^e$ denotes the UOR metric under the UOR-PMDP with the empirical environment parameter distribution $\mathbb{D}^e$. 

% We measure the distance of two distinct distribution using total variation distance $D_{TV}$.
% \begin{equation*}
% \mathcal{D}_{TV}(\mathbb{D},\mathbb{D}^e)=\int_\mathcal{P}\left|\mathbb{D}(p)-\mathbb{D}^e(p)\right|\mathrm{d}p
% \end{equation*}
Then, $\forall$ given $\epsilon>0$, $\exists~\kappa=O(\epsilon^d)$, such that as long as $\mathbb{D}$ and $\mathbb{D}^e$ satisfies the total variation distance $D_{TV}(\mathbb{D},\mathbb{D}^e)\le\kappa$, then we can guarantee that
\begin{equation}
    \mathcal{E}(\pi^e)\ge \mathcal{E}(\pi^\ast)-\epsilon.
\end{equation}
\end{theorem}
Based on Theorems~\ref{Theorem near optimal strengthem} and~\ref{Theorem Empirical}, we have  Corollary~\ref{Corollary Empirical}.
\begin{corollary}
$\forall$ optimality requirement $\epsilon_1=3\epsilon>0$, $\exists~\delta_0=O(\epsilon)$ and $\kappa=O(\epsilon^d)$, such that as long as $D_{TV}(\mathbb{D},\mathbb{D}^e)\le\kappa$ and Policy\_Update can learn an $\epsilon$-suboptimal policy for metric $\hat{\mathcal{E}}$,
by running Algorithm~\ref{Algorithm with model} with diameter upper bound $\delta\le \delta_0$ and distribution $\mathbb{D}^e$, we can guarantee that the output policy $\hat{\pi}^e$ of Algorithm~\ref{Algorithm with model} satisfies
\begin{equation}
    \mathcal{E}(\hat{\pi}^e)\ge \mathcal{E}(\pi^\ast)-\epsilon_1.
\vspace{-0.8cm}    
\end{equation}
\label{Corollary Empirical}
\end{corollary}
% From the theorem~\ref{Theorem near optimal} or~\ref{Theorem near optimal strengthem}, we can guarantee that  $\mathcal{E}^e(\hat{\pi}^e)\ge \mathcal{E}^e(\pi^e)-\epsilon$. And from the theorem~\ref{Theorem Empirical} we can guarantee that $    \mathcal{E}(\pi^e)\ge \mathcal{E}(\pi^\ast)-\epsilon$. As a result, the output policy $\hat{\pi}^e$ has the property that
% \begin{equation}
% \begin{aligned}
%     \mathcal{E}(\hat{\pi}^e)\ge \mathcal{E}^e(\hat{\pi}^e)-\epsilon\ge \mathcal{E}^e(\pi^e)-2\epsilon\\\ge \mathcal{E}^e(\pi^\ast)-2\epsilon\ge \mathcal{E}(\pi^\ast)-3\epsilon
% \end{aligned}
% \end{equation}
\vspace{-0.1mm}
Corollary~\ref{Corollary Empirical} demonstrates that even running Algorithm~\ref{Algorithm with model} with the biased distribution $\mathbb{D}^e$, as long as $\mathbb{D}^e$ is close enough to $\mathbb{D}$, the output policy is still near-optimal.
% Albeit the inaccuracy of the empirical distribution function, the algorithm~\ref{Algorithm with model} still works, and the utility of that is also near optimal.
\subsection{Distribution-Free UOR-RL}
\subsubsection{Algorithm design}
In practice, it is likely that the distribution function $\mathbb{D}$ is unknown, making Algorithm~\ref{Algorithm with model} not applicable. Therefore, we propose the \textit{Distribution-Free UOR-RL (DF-UOR-RL)} training algorithm in Algorithm \ref{Algorithm distribution free} that trains a satisfactory policy even without any knowledge of the distribution function $\mathbb{D}$. 

\begin{algorithm}[t]%\vspace{10pt}
\caption{DF-UOR-RL Algorithm}
\label{Algorithm distribution free}
	\small
	\tcp{Initialization}
	Initialize empty trajectory clusters $\mathcal{C}_1,\mathcal{C}_2,\cdots,\mathcal{C}_{n_1}$, cluster size $n_2$, and policy $\pi_{\theta^0}$\;\label{a2l1}
	\tcp{Policy Training}
	\ForEach{Iteration $i=0$ to max-iterations\label{a2l2}}{
	\ForEach{$j=1$ to $n_1$\label{a2l3}}{
	\ForEach{$k=1$ to $n_2$\label{a2l4}}{
	    Observe environment parameter $p_{j,k}$\;\label{a2l5}
	    Execute $\pi_{\theta^i}$ under $p_{j,k}$,
	    get trajectory $\xi_{j,k}$, and $\mathcal{C}_j\leftarrow\mathcal{C}_j\cup\{\xi_{j,k}\}$\;\label{a2l6}
	    Evaluate discounted reward $\hat{J}(\xi_{j,k})$ of $\xi_{j,k}$\;\label{a2l7}
	}
	$\hat{J}_j\leftarrow\frac{1}{|\mathcal{C}_j|}\cdot \sum_{\xi_{j,k}\in \mathcal{C}_j} \hat{J}(\xi_{j,k})$\;\label{a2l8}
	}
	Sort the sequence $\{\hat{J}_j\}$ into an increasing sequence  $\{\hat{J}_{\alpha_j}\}$\;\label{a2l9}
	\tcp{Metric Calculation}
    Initialize metric $\tilde{\mathcal{E}}(\pi_{\theta^i})\leftarrow0$\;\label{a2l10}
    \ForEach{$j=1$ to $n$\label{a2l11}}
    {
    $w_j\leftarrow\int_{(j-1)/n}^{j/n} W(x)\mathrm{d}x$\;\label{a2l12}
    $\tilde{\mathcal{E}}(\pi_{\theta^i})\leftarrow\tilde{\mathcal{E}}(\pi_{\theta^i})+w_j\cdot \hat{J}_{\alpha_j}$\;\label{a2l13}
    }
    \tcp{Policy Update}
    $\pi_{\theta^{i+1}}\leftarrow
    $ Policy\_Update$(\pi_{\theta^i},\tilde{\mathcal{E}}(\pi_{\theta^i}))$\;\label{a2l14}
	}
\end{algorithm}

% In solving UOR-PMDP
%没有D的时候intuition是什么？
%算法解释
At the beginning, Algorithm~\ref{Algorithm distribution free} randomly sets the initial policy $\pi_{\theta^0}$ and $n_1$ empty clusters, and chooses the size of each cluster as $n_2$ (line~\ref{a2l1}). We will introduce in detail how to set the number of clusters $n_1$ and cluster size $n_2$ in Section~\ref{section Model Free Analysis}. Then, Algorithm~\ref{Algorithm distribution free} begins to train the policy (line \ref{a2l2}-\ref{a2l14}). In each iteration $i$, it samples $n_2$ trajectories for each cluster $\mathcal{C}_j$, by executing the current policy $\pi_{\theta^i}$ under the observed environment parameters (line~\ref{a2l5}-\ref{a2l6}), and evaluates the discounted reward of these trajectories (line~\ref{a2l7}).  After that, Algorithm~\ref{Algorithm distribution free} evaluates the performance $\hat{J}_j$ of each cluster $\mathcal{C}_j$ by averaging the discounted reward of the trajectories in the cluster  (line \ref{a2l8}), and sorts the sequence $\{\hat{J}_j\}$ into an increasing one $\{\hat{J}_{\alpha_j}\}$ (line~\ref{a2l9}). Then, Algorithm \ref{Algorithm distribution free} calculates the metric $\tilde{\mathcal{E}}(\pi_{\theta^i})$, which is an approximation of the UOR metric $\mathcal{E}(\pi_{\theta^i})$ (line \ref{a2l10}-\ref{a2l13}).  Initially, it sets the metric $\tilde{\mathcal{E}}(\pi_{\theta^i})$  as zero (line~\ref{a2l10}). Then, for each cluster $\mathcal{C}_j$, Algorithm~\ref{Algorithm distribution free} allocates the weight to cluster according to its ranking $\alpha_j$ and the preference function (line~\ref{a2l12}) and updates the $\tilde{\mathcal{E}}(\pi_{\theta^i})$ (line~\ref{a2l13}) based on the weight and performance of the cluster. Finally, Algorithm~\ref{Algorithm distribution free} obtains the $\tilde{\mathcal{E}}(\pi_{\theta^i})$ and uses it to update the policy (line~\ref{a2l14}).

Different from Algorithm \ref{Algorithm with model}, due to the lack of the knowledge of the distribution $\mathbb{D}$, Algorithm \ref{Algorithm distribution free} observes the environment parameter rather than directly sample it according to $\mathbb{D}$. Given that it is of large bias to evaluate $J(\pi,p)$ from only one trajectory, Algorithm~\ref{Algorithm distribution free} averages the discounted rewards of $n_2$ trajectories. The clusters in Algorithm~\ref{Algorithm distribution free} have the same functionality as the blocks in Algorithm~\ref{Algorithm with model}, and Algorithm~\ref{Algorithm distribution free} uses them to calculate an approximate UOR metric $\tilde{\mathcal{E}}(\pi_{\theta^i})$.
\subsubsection{Algorithm Analysis}
\label{section Model Free Analysis}
To analyze Algorithm~\ref{Algorithm distribution free}, we introduce an additional mild Assumption 5 on two properties of the environment parameters, including the difference between consecutively sampled environment parameters in line~\ref{a2l5} of Algorithm~\ref{Algorithm distribution free}, and the convergence rate of the posterior distribution of the environment parameter to the distribution $\mathbb{D}$. Because of space limit, we provide the detailed description of Assumption ~\ref{Assumption_5_complete} in Appendix \ref{Proof of Distribution Free}. 

Based on Assumptions~\ref{Assumption R,T continue}-\ref{Assumption_5_complete}, we have the theoretical guarantee of Algorithm~\ref{Algorithm distribution free} in the following Theorem~\ref{Theorem distribution function free}.
\begin{theorem}
\label{Theorem distribution function free}
$\forall$ optimality requirement $\epsilon=2\epsilon_0>0$ and confidence $\rho$,  $\exists~n_1=\Theta(\frac{-\ln\rho}{\epsilon^2}),n_2=\Theta(\frac{-\ln\rho}{\epsilon^{2d+2}})$, such that as long as Policy\_Update can learn an $\epsilon_0$-suboptimal policy for metric $\tilde{\mathcal{E}}$, by running Algorithm~\ref{Algorithm distribution free} with trajectory cluster number larger than $n_1$ and cluster size larger than $n_2$, we can guarantee that the output policy $\tilde{\pi}$ of Algorithm~\ref{Algorithm distribution free} satisfies
\begin{equation}
\mathcal{E}(\pi)\ge\mathcal{E}(\tilde{\pi})-\epsilon
\end{equation} 
with confidence more than $1-\rho$.
\end{theorem}

Theorem \ref{Theorem distribution function free} provides guidelines for setting the cluster size $n_1$ and cluster number $n_2$ in Algorithm~\ref{Algorithm distribution free}. In fact, as $n_1$ and $n_2$ increase, the performance evaluation of the cluster and the weight allocated to the cluster will be more accurate, both of which lead to a more accurate approximation of the UOR metric. However, the increase of either $n_1$ or $n_2$ leads to an increased complexity of each iteration of Algorithm~\ref{Algorithm distribution free}. To deal with such trade-off, we could set $n_1$ and $n_2$ based on the lower bounds in Theorem~\ref{Theorem distribution function free}, through which Algorithm~\ref{Algorithm distribution free} can guarantee both the optimality requirement $\epsilon$ and $O(\frac{\ln^2{\rho}}{\epsilon^{2d+4}})$ complexity of each iteration.
%定理分析：不能太小，不能太大。near-optimal

\section{Experiments}
% Considering that recent RL works are oftentimes dedicated to high-dimensional continuous decision-making tasks, we deploy UOR-RL experiments on such tasks, use deep neural network to parametrize the policy, and compare it to state-of-the-art deep RL baselines.

\subsection{Baseline Methods}
We compare UOR-RL with the following four baselines. \\
\textbf{Domain Randomization-Uniform (DR-U).} Domain Randomization (DR) \cite{tobin2017domain} is a method that randomly samples environment parameters in a domain, and optimizes the expected return over all collected trajectories. DR-U is an instance of DR, which samples environment parameters from a uniform distribution.
\textbf{Domain Randomization-Gaussian (DR-G).} DR-G is another instance of DR, which samples environment parameters from a Gaussian distribution.
\textbf{Ensemble Policy Optimization (EPOpt).} EPOpt \cite{rajeswaran2016epopt} is a method that aims to find a robust policy through optimizing the performance of the worst few collected trajectories. 
\textbf{Monotonic Robust Policy Optimization (MRPO).} MRPO \cite{jiang2021monotonic} is the state-of-the-art robust RL method, which is based on EPOpt and jointly optimizes the performance of the policy in both the average and worst cases.
% \begin{itemize}[leftmargin=*]
% \item \textbf{Domain Randomization-Uniform (DR-U).} Domain Randomization (DR) \cite{tobin2017domain} is a method that randomly samples environment parameters in a domain, and optimizes the expected return over all collected trajectories. DR-U is an instance of DR, which samples environment parameters from a uniform distribution.
% \item \textbf{Domain Randomization-Gaussian (DR-G).} DR-G is another instance of DR, which samples environment parameters from a Gaussian distribution.
% \item \textbf{Ensemble Policy Optimization (EPOpt).} EPOpt \cite{rajeswaran2016epopt} is a method that aims to find a robust policy through optimizing the performance of the worst few collected trajectories. 
% %$\tau$-percent collected trajectories, where $\tau$ is a tunable parameter that affects the robustness of the policy. 
% %decides the degree of robustness.
% \item \textbf{Monotonic Robust Policy Optimization (MRPO).} MRPO \cite{jiang2021monotonic} is the state-of-the-art robust RL method, which is based on EPOpt and jointly optimizes the \textcolor{blue}{policy performance} in both the average and worst cases.
% \end{itemize}

\subsection{MuJoCo Tasks and Settings}
We conduct experiments in six MuJoCo \cite{mujoco} tasks of version-0 based on Roboschool\footnote{\url{https://openai.com/blog/roboschool}}, including Walker 2d, Reacher, Hopper, HalfCheetah, Ant, and Humanoid. In each of the six tasks,  by setting different environment parameters, we get a series of environments with the same optimization goal but different dynamics.  Besides, we take 6 different random seeds for each task, and compare the performance of our algorithms to the baselines under these seeds. Because of space limit, we put the specific environment parameter settings and the random seed settings during the training process in Appendix \ref{para_training}. For testing, in each environment, we sample 100 environment parameters following the Gaussian distributions truncated over the range given in Table \ref{env_para}. 
%\begin{itemize}
%\item  \textbf{Walker2d-v0}: A two-dimensional bipedal robot is rewarded for walking at a high speed.
%\item \textbf{Reacher}. A 2-D robot is rewarded for reaching a random target.
%\item \textbf{Hopper}. A 2-D single-legged robot is rewarded for hopping at a high speed.
%\item \textbf{Half Cheetah}. A 2-D cheetah robot is rewarded for running at a high speed.
%\item \textbf{Ant}. A 2-D four-legged robot  is rewarded for walking at a high speed.
%\end{itemize}
% \renewcommand\arraystretch{1.2}{

\begin{table}[!ht]
    \renewcommand\arraystretch{1.2}
    % \vspace{-0.15cm}
    \caption{Environment Parameter Settings for Testing.}
    % \vspace{-0.1cm}
    \centering
    \resizebox{\columnwidth}{!}{
    \begin{tabular}{c|c|c|c}
    %\hline
    \Xhline{1.25pt}
    \textbf{Task} & \textbf {Parameters} & \textbf{Range $\mathcal{P}$} & \textbf{Distribution $\mathbb{D}$}\\
    %\Xhline{1.25pt}
    \hline
    % \multirow{2}{*}{Walker2d}  & Density $\in$ [750,1250] & $\mu = 1000$ ; $\sigma = 83.3$\\ 
    % & Friction $\in$ [0.5, 1.1] & $\mu = 0.8$ ; $\sigma = 0.1$\\ \hline
    \multirow{2}{*}{Reacher} & Body size & [0.008,0.05]& $\mathcal N(0.029,0.007^2)$\\  
    & Body length & [0.1,0.13] & $\mathcal N(0.015,0.005^2)$\\ \hline
    \multirow{2}{*}{Hopper} &  Density & [750,1250] & $\mathcal N(1000,83.3^2)$\\  
    & Friction & [0.5, 1.1] & $\mathcal N(0.8,0.1^2)$\\ \hline
    \multirow{2}{*}{Half Cheetah} &  Density & [750,1250] & $\mathcal N(1000,83.3^2)$\\  
    & Friction & [0.5, 1.1] & $\mathcal N(0.8,0.1^2)$\\ \hline
    \multirow{2}{*}{Humanoid} &  Density & [750,1250] & $\mathcal N(1000,83.3^2)$\\  
    & Friction & [0.5, 1.1] & $\mathcal N(0.8,0.1^2)$\\ \hline
    \multirow{2}{*}{Ant} &  Density & [750,1250] & $\mathcal N(1000,83.3^2)$\\  
    & Friction & [0.5, 1.1]& $\mathcal N(0.8,0.1^2)$\\ \hline
    \multirow{2}{*}{Walker 2d} &  Density & [750,1250] & $\mathcal N(1000,83.3^2)$\\ 
    & Friction & [0.5, 1.1]& $\mathcal N(0.8,0.1^2)$\\
    %\hline
    \Xhline{1.25pt}
    \end{tabular}}
    % \vspace{-0.2cm}
    \label{env_para}
\end{table}

\begin{table*}[!ht]
    \vspace{-0.46cm}
    \caption{Test results ($k=1$). Each value denotes the mean and std of $\mathcal{E}_1$ over 6 seeds.}
    \vspace{-0.1cm}
    \newcolumntype{"}{@{\hskip\tabcolsep\vrule width 1.25pt\hskip\tabcolsep}}
    \centering
    %\resizebox{\columnwidth}{!}{
    \small
    \begin{tabular}{c"cccccc}
    \Xhline{1.25pt}
    \textbf{Algorithm} & \textbf{Reacher} & \textbf{Hopper} &\textbf{Half Cheetah} &\textbf{Humanoid} & \textbf{Ant} & \textbf{Walker 2d} \\
    \Xhline{1.25pt}
    % \hline
    DR-U & $10.92 \pm 1.90 $ & $1465 \pm 447.78$ & $2375 \pm 70.36$ & $54.73 \pm 3.73$ & $2247 \pm 568.53$ & $1022 \pm 368.74$ \\
    DR-G & $10.86 \pm 2.34 $ & $1557 \pm 346.97$ & $2382 \pm 64.94$ & $56.43 \pm 12.77$ & $2258 \pm 538.72$ & $838 \pm 211.86$\\
    EPOpt  & $12.57\pm 0.59$ & $1360 \pm 631.32$ & $2450 \pm 84.62$ & $65.37 \pm7.53$ & $2307 \pm 192.94$ & $1058 \pm 432.98$\\
    MRPO  & $11.88\pm 2.27$ & $1578\pm333.85$ & $2665 \pm 147.89$ & $89.48 \pm 	5.72$ & $2303 \pm 560.06$ & $1368 \pm 519.04$\\
    DB-UOR-RL  & $\mathbf{14.38 \pm 0.86}$ & $\mathbf{1942 \pm 348.97}$ & $\mathbf{3067 \pm 169.31}$ & $\mathbf{95.57 \pm 17.64}$ & $3212 \pm 187.34$ & $\mathbf{1373 \pm 567.17}$\\
    DF-UOR-RL  & $13.63\pm 0.59$ & $1772 \pm 313.59$ & $2760 \pm 238.41$ & $93.97 \pm 20.36$ & $\mathbf{3213 \pm 399.13}$ & $1268 \pm 348.03$\\
    \Xhline{1.25pt}
    \end{tabular}
    \label{k=1}
\end{table*}
%\vspace{-1cm}
\begin{table*}[!ht]
    \vspace{-0.15cm}
    \caption{Test results ($k=0$). Each value denotes the mean and std of the average return of all trajectories over 6 seeds. }
    \vspace{-0.1cm}
    \newcolumntype{"}{@{\hskip\tabcolsep\vrule width 1.25pt\hskip\tabcolsep}}
    \centering
    %\resizebox{\columnwidth}{!}{
    \small
     \begin{tabular}{c"cccccc}
    \Xhline{1.25pt}
    \textbf{Algorithm} & \textbf{Reacher} & \textbf{Hopper} &\textbf{Half Cheetah} &\textbf{Humanoid} & \textbf{Ant} & \textbf{Walker 2d} \\
    \Xhline{1.25pt}
    % \hline
    DR-U & $16.47 \pm 1.91 $ & $1768 \pm 471.06$ & $2428 \pm 89.76$ & $71.32 \pm 5.10$ & $2272 \pm 572.20$ & $1159 \pm 369.08$ \\
    DR-G & $19.28 \pm 1.21 $ & $1993 \pm 159.17$ & $2420 \pm 61.21$ & $72.38 \pm 15.77$ & $2475 \pm 86.20$ & $1145 \pm 503.29$\\
    EPOpt  & $17.83\pm 0.49$ & $1455 \pm 675.24$ & $2555 \pm 142.51$ & $82.30 \pm 7.81$ & $2328 \pm 203.12$ & $1160 \pm 418.16$\\
    MRPO  & $17.55\pm 1.43$ & $1940 \pm193.08$ & $2652 \pm 125.13$ & $99.83 \pm 6.89$ & $2295 \pm 574.06$ & $1268 \pm 548.13$\\
    DB-UOR-RL  & $\mathbf{19.97 \pm 1.11 }$ & $\mathbf{2007 \pm 379.03}$ & $\mathbf{2910 \pm 135.06}$ & $\mathbf{104.53 \pm 18.68}$ & $3393 \pm 240.89$ & $\mathbf{1403 \pm 234.75}$\\
    DF-UOR-RL  & $19.53\pm 0.65$ & $1805 \pm 246.96$ & $2698 \pm 293.90$ & $98.28 \pm 11.51$ & $\mathbf{3410 \pm 106.21}$ & $1208 \pm 200.76$\\
    \Xhline{1.25pt}
    \end{tabular}
    \label{ave}
\end{table*}
% \vspace{-1cm}
\begin{table*}[ht!]
    \vspace{-0.15cm}
    \caption {Test results ($k=21$). Each value denotes the mean and std of the average return of worst 10\% trajectories over 6 seeds.}
    \vspace{-0.1cm}
    \newcolumntype{"}{@{\hskip\tabcolsep\vrule width 1.25pt\hskip\tabcolsep}}
    \centering
    %\resizebox{\columnwidth}{!}{
    \small
     \begin{tabular}{c"cccccc}
    \Xhline{1.25pt}
    \textbf{Algorithm} & \textbf{Reacher} & \textbf{Hopper} &\textbf{Half Cheetah} &\textbf{Humanoid} & \textbf{Ant} & \textbf{Walker 2d} \\
    \Xhline{1.25pt}
    % \hline
    DR-U & $-0.81 \pm 1.86 $ & $442 \pm	136.28$ & $1895 \pm 112.92$ & $8.00\pm 7.19$ & $2163 \pm 564.79$ & $461 \pm 369.96$ \\
    DR-G & $-1.90 \pm 3.73 $ & $593 \pm 367.47$ & $1957 \pm 225.45$ & $8.55 \pm 7.37$ & $2133 \pm 531.10$ & $380 \pm 388.40$\\
    EPOpt  & $-0.14\pm 1.79$ & $664 \pm 495.87$ & $2132 \pm 418.40$ & $10.47 \pm 4.20$ & $2192 \pm 153.55$ & $360 \pm 222.00$\\
    MRPO  & $-3.66\pm 7.54$ & $678 \pm640.70$ & $\mathbf{2523 \pm 195.72}$ & $34.92\pm 	5.70$ & $2187 \pm555.22$ & $515 \pm 419.33$\\
    DB-UOR-RL  & $\mathbf{3.82 \pm 	0.85 }$ & $\mathbf{	824 \pm	583.20}$ & $2497 \pm 404.51$ & $\mathbf{	35.45\pm 	7.97}$ & $3132 \pm 119.07$ & $\mathbf{624 \pm 283.99}$\\
    DF-UOR-RL  & $1.75\pm 	1.66$ & $452	 \pm 641.92$ & $2383 \pm 581.09$ & $	33.18 \pm 44.76$ & $\mathbf{3265 \pm 	196.34}$ & $497 \pm 278.93	$\\
    \Xhline{1.25pt}
    \end{tabular}
    \label{worst10}
\end{table*}
\begin{figure*}[ht!]
% \vspace{-0.5cm}
    \centering
    \includegraphics[width=\textwidth]{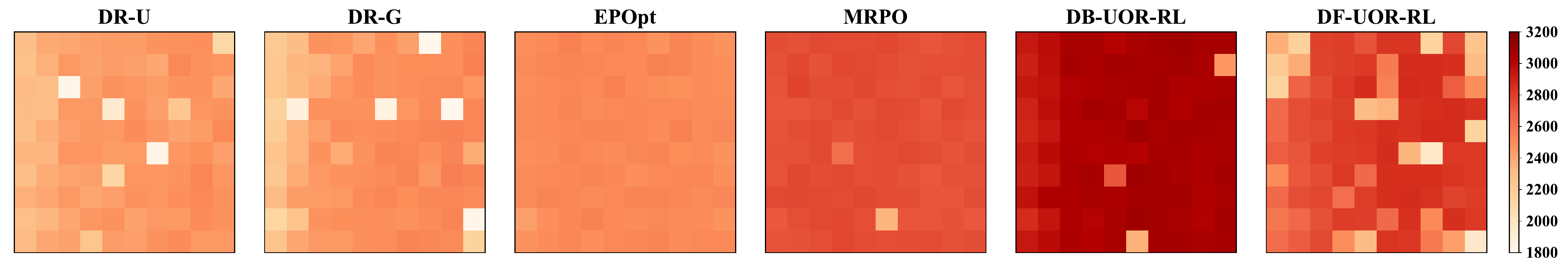}
    \vspace{-0.6cm}
    \caption{Heat map of $\mathcal{E}_1$ in sub-ranges (Half Cheetah). The x-axis and y-axis denote friction and density, respectively. The ranges of these two parameters are chosen as in Table \ref{env_para}, and are evenly divided into 10 sub-ranges. }
    \label{Halfcheetah-v0}
\end{figure*}

\begin{figure*}[ht!]
\vspace{-0.2cm}
    \centering
    \includegraphics[width=\textwidth]{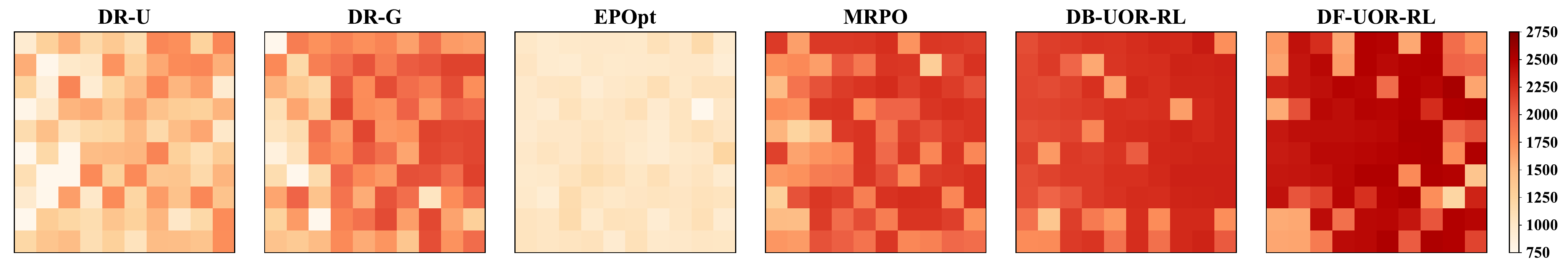}
    \vspace{-0.6cm}
    \caption{Heat map of $\mathcal{E}_1$ in sub-ranges (Hopper).  The x-axis and y-axis denote friction and density, respectively. The ranges of these two parameters are chosen as in Table \ref{env_para}, and are evenly divided into 10 sub-ranges.}
    \vspace{-0.45cm}
    \label{Hopper-v0}
\end{figure*}
% In the experiments, we let the preference function $W$ take the form as given by Equation (\ref{fomula robust weight function}), which uses a robustness degree $k$ to represent user preference.  
% Thus, different $k$'s correspond to different UOR metrics, and we use $\mathcal{E}_{k_0}$ to denote the UOR metric when $k=k_0$. 

In the experiments, we let the preference function $W$ take the form as given by Equation (\ref{fomula robust weight function}), which uses a robustness degree $k$ to represent user preference.  
Thus, we conduct experiments on various UOR metrics, including the average and max-min robustness, by choosing different $k$ in Equation (\ref{fomula robust weight function}), and we use $\mathcal{E}_{k_0}$ to denote the UOR metric with $k=k_0$.

Considering that the state and action spaces of MuJoCo are high-dimensional and continuous, we choose to use deep neural networks to represent the policies of UOR-RL and the baseline methods, and use PPO \cite{schulman2017proximal} to implement the policy updating process. 
%In each metric, we test in three different seeds, and calculate the mean and standard deviation of each model's performance in each environment. 
%In training, for DR-U, EPOpt and MRPO, we sample uniformly over the range of environment parameters given in table \ref{env_para}. For DR-G and our two algorithms, we sample parameters according to Gaussian distribution as shown in table \ref{env_para}. 

\subsection{Experimental Results and Discussions}

We compare UOR-RL with the baseline methods when $k\in\{0,1,21\}$. Specifically, when $k=0$, $W(x)\equiv1$, and thus $\mathcal{E}_{0}$ is equivalent to the expected return $\mathbb{E}_{p\sim \mathbb{D}}\left[J(\pi,p)\right]$ over the distribution $\mathbb{D}$; when $k=21$, $\mathcal{E}_{21}$ approximates the expected return over the worst 10\% trajectories, because more than 90\% weight is allocated to them according to the preference function $W$; when $k=1$, $\mathcal{E}_{1}$ represents the UOR metric between $\mathcal{E}_{0}$ and $\mathcal{E}_{21}$. 
%robustness objective between the above two cases. 
%three different $k$'s, including $k=0$, $k=1$, and $k=21$. 

Table \ref{k=1} shows the test results under the UOR metric $\mathcal{E}_{1}$. Among all algorithms, DB-UOR-RL performs the best, and it outperforms the four baselines in each environment. Such results indicate that DB-UOR-RL is effective under metric $\mathcal{E}_{1}$. At the same time, although the performance of DF-UOR-RL is not as good as that of DB-UOR-RL, it is better than those of the baselines. This shows that DF-UOR-RL could output competitive policies, even when the distribution of environment parameters is unknown. 
%Note that both DB-UOR-RL and DF-UOR-RL are trained with $k=0$, which represents the assignment of the same weight to each trajectory. 

Table \ref{ave} shows the test result under the average return of all trajectories. In most environments, DR-G achieves the best performance among the baselines under such average-case metric, because it directly takes this average-case metric as its optimization objective. From Table \ref{ave}, we could observe that the performance of DB-UOR-RL and DF-UOR-RL is close to or better than that of DR-G in most environments. Such observation indicates that UOR-RL can also yield acceptable results, when robustness is not considered.

Table \ref{worst10} shows the test result under the average return of the worst 10\% trajectories. From the table, both DB-UOR-RL and DF-UOR-RL perform no worse than the best baselines in most environments, which shows that UOR-RL also yields sufficiently good performance in terms of the traditional robustness evaluation criteria.

Apart from Tables \ref{k=1}-\ref{worst10}, we also visualize the performance of UOR-RL and the baselines in the ranges of environment parameters given by Table \ref{env_para} by plotting heat maps. Because of space limit, we only show the heat maps of the Half Cheetah and Hopper tasks with $k=1$ in Figures \ref{Halfcheetah-v0} and \ref{Hopper-v0}, respectively. In these two figures, a darker color means a better performance. We could observe that both DB-UOR-RL and DF-UOR-RL are darker in color than baselines in most sub-ranges, which supports the superiority of UOR-RL for most environment parameters.
Moreover, we place the heat maps of the other four tasks in Appendix~\ref{other_heatmaps}.

To show the effect of the robustness degree parameter $k$ on the performance of UOR-RL, we carry out experiments with four robustness degree parameters $k\in\{0,1,5,21\}$ in Half Cheetah under the same environment parameters. The results are shown in Figures \ref{k_curve_alg1} and \ref{k_curve_alg3}. To plot these two figures, we sort the collected trajectories by return into an increasing order, divide the trajectories under such order into 10 equal-size groups, calculate the \textit{average return of the trajectories (ART)} in each group, and compute the normalized differences between the ARTs that correspond to each consecutive pair of $k$'s in $\{0,1,5,21\}$. We could observe that every curve in these two figures shows a decreasing trend as the group index increases. Such observation indicates that, as $k$ increases, both DB-UOR-RL and DF-UOR-RL pay more attention to the trajectories that perform poorer, and thus the trained policies become more robust.
%model pays more attention to these poorly performing trajectories, i.e., the model becomes more robust.
%the normalized difference for each 
%carry out the following experiments. First, for both DB-UOR-RL and DF-UOR-RL, we choose four robustness degree parameters $k\in\{0,1,5,21\}$ in Half Cheetah, respectively. 
%Second, we select the same environment parameters following Table \ref{env_para} for each model to test.
%the collected trajectories by return and divide them into ten 10\% intervals and count the average return of the trajectories(AT) in each interval on this basis. Last, we subtract the AT in each interval obtained at different k, normalize the difference, and plot the line graph shown in Figure \ref{k_curve_alg1} and \ref{k_curve_alg3}. 
%These two graphs show the performance improvement in each 10\% interval in UOR-RL as $k$ increases.

Additionally, we plot the training curves of the baselines and UOR-RL, and place them in Appendix \ref{train_cruve}.

% \begin{figure}[h]
% \centering
% \begin{minipage}{...}
% \centering
% \includegraphics[width=...]{...}
% \caption{...}
% \end{minipage}
% \qquad
% \begin{minipage}{...}
% \centering
% \includegraphics[width=...]{...}
% \caption{...}
% \end{minipage}
% \end{figure}

%\subsubsection{Main Results}
%The results in Figure \ref{k_curve_alg1} and \ref{k_curve_alg3} show that in both DB-UOR-RL and DF-UOR-RL, as the AT of the selected interval increasing, the relative improvement in performance by increasing $k$ decreases. 

\section{Related Work}
%\paragraph{Robust RL.} 
Robust RL \cite{Iyengar04robustdynamic,nilim2005robust,RobustMDP} aims to optimize policies under the worst-case environment, traditionally by the zero-sum game formulation \cite{Littman94markovgames,Littman1996grl}. Several recent works focus on finite or linear MDPs, and propose robust RL algorithms with theoretical guarantees \cite{derman2021twice,wang2021online,badrinath2021robust,zhang2021robust,grand2020first,kallus2020double}. However, real-world applications are usually with continuous state and action spaces, as well as complex non-linear dynamics \cite{kumar2020one,zhang2020stability}. 

A line of deep robust RL works robustify policies against different factors that generate the worst-case environment, such as agents' observations \cite{state_robust_1, state_robust_2,state_robust_3}, agents' actions \cite{action_robust_1,action_robust_2}, transition function \cite{mankowitz2020robust,viano2021robust,chen2021improved}, and reward function \cite{wang2020reinforcement}. 
Another line of recent works \cite{kumar2020one, tobin2017domain, jiang2021monotonic, igl2019generalization, cobbe2019quantifying} aim to improve the average performance over all possible environments.
Considering only the worst or average case may cause the policy to be overly conservative or aggressive, and limit \cite{state_robust_1, state_robust_2,state_robust_3,action_robust_1,action_robust_2,mankowitz2020robust,viano2021robust,chen2021improved,wang2020reinforcement,kumar2020one, tobin2017domain, jiang2021monotonic, igl2019generalization, cobbe2019quantifying} for boarder applications. 
Hence, some researches study to use other cases to characterize the robustness of the policy.
\cite{risk-sensitive2015} optimizes the policy performance on $\alpha$-percentile worst-case environments; \cite{soft-robust} considers the robustness with a given environment distribution; \cite{distributionally2010,distributionally2015}       
aim to improve the policy performance on the worst distribution in an environment distribution set. 
\begin{figure}[!ht]
    \centering
    % \vspace{-0.3cm}
    \includegraphics[width=\columnwidth]{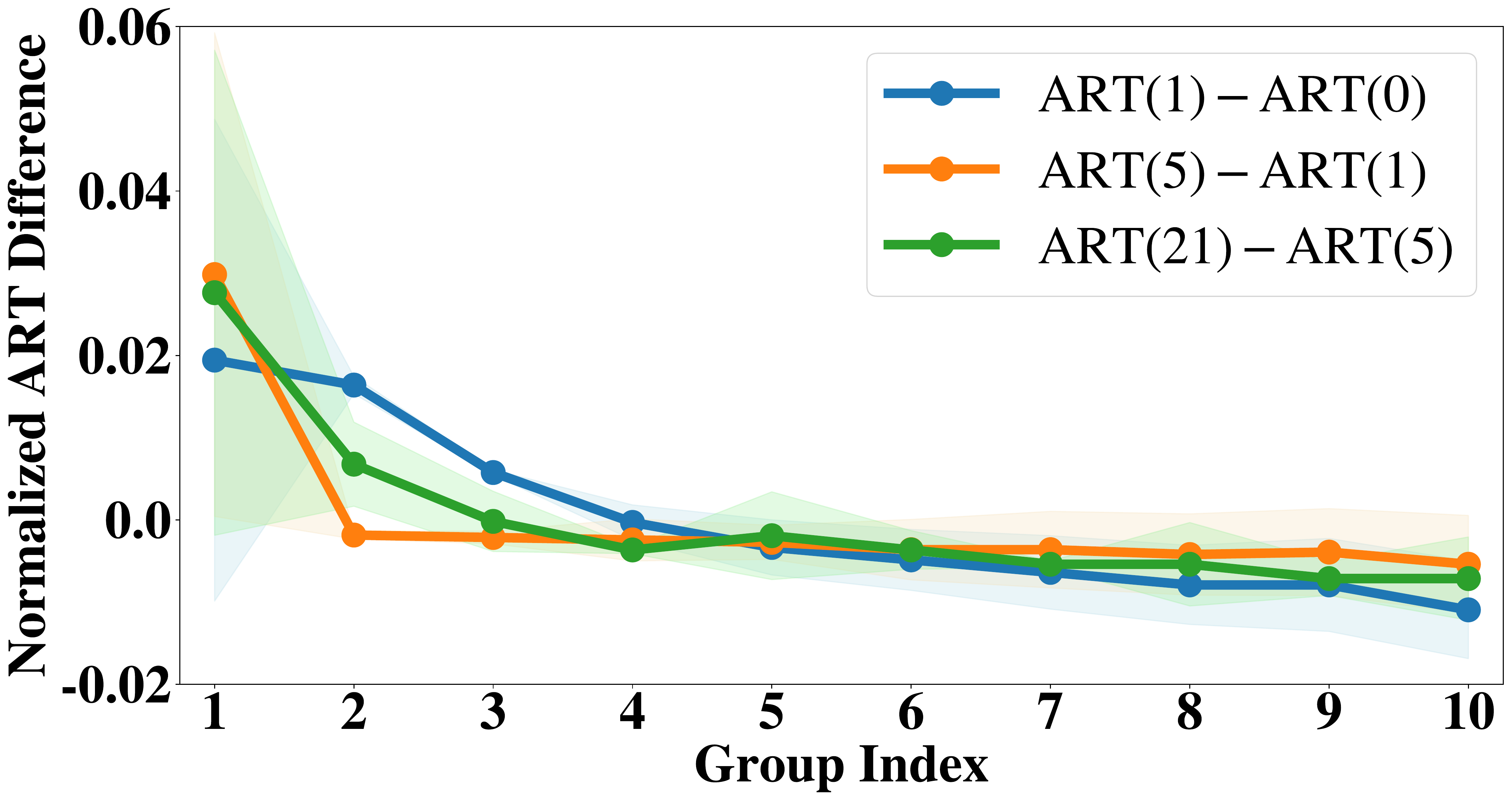}
    \vspace{-0.4cm}
    \caption{Normalized ART difference for each sorted group (DB-UOR-RL). ART($k_0$) denotes the ART with $k=k_0$.}
    \label{k_curve_alg1}
    \vspace{-0.2cm}
\end{figure}

\begin{figure}[!ht]
    \centering
    \includegraphics[width=\columnwidth]{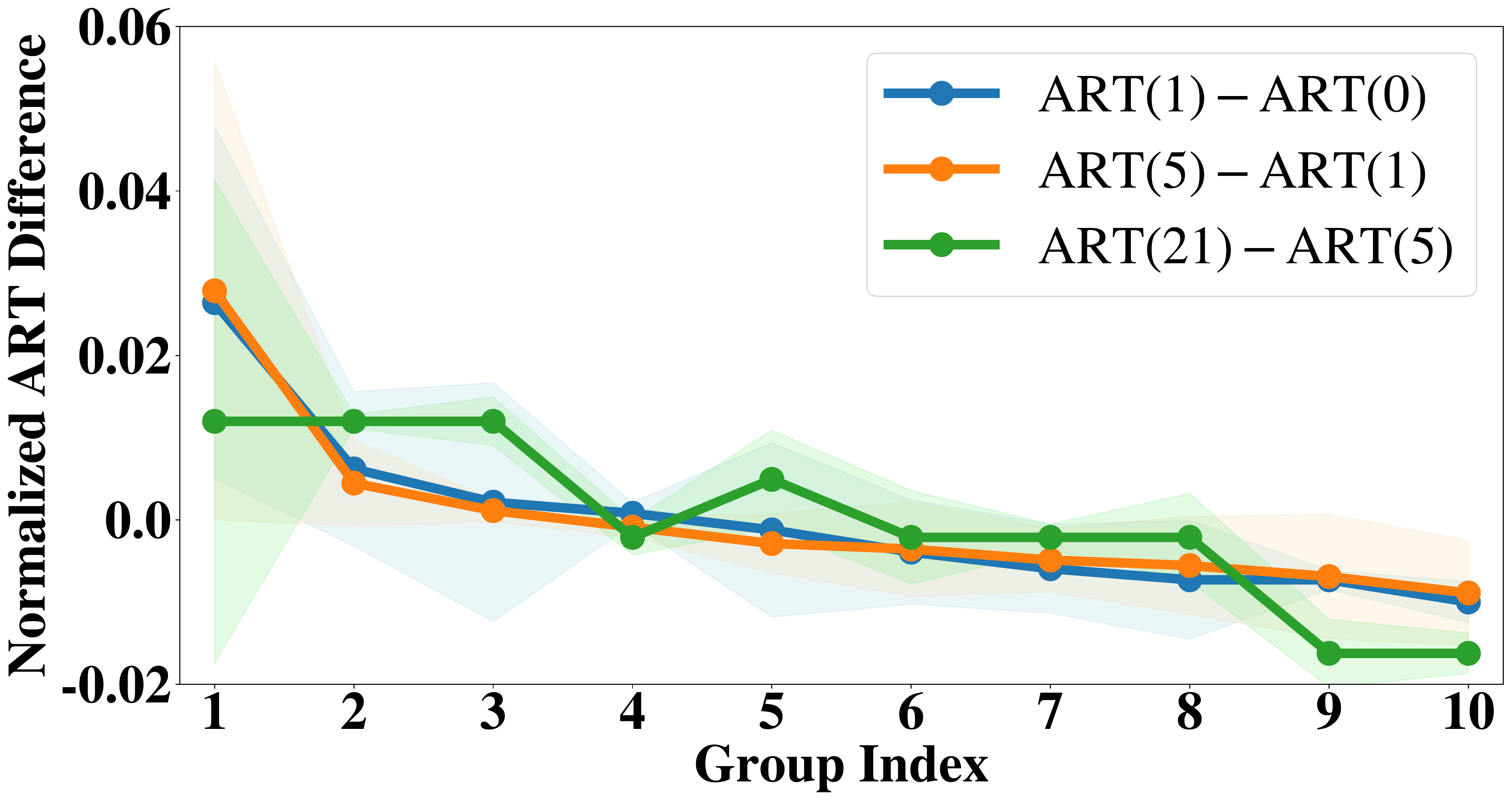}
    \vspace{-0.4cm}
    \caption{Normalized ART difference for each sorted group (DF-UOR-RL). ART($k_0$) denotes the ART with $k=k_0$.}
    \label{k_curve_alg3}
    \vspace{-0.2cm}
\end{figure}

Each of \cite{state_robust_1, state_robust_2,state_robust_3,action_robust_1,action_robust_2,mankowitz2020robust,viano2021robust,chen2021improved,wang2020reinforcement,kumar2020one, tobin2017domain, jiang2021monotonic, igl2019generalization, cobbe2019quantifying,risk-sensitive2015,soft-robust,distributionally2010,distributionally2015} optimizes a specific type of robustness, and is only suitable to a specific preference to robustness (e.g. methods focusing on worst case suit the most conservative preference to robustness). However, user preference varies in different scenarios, and an RL method that optimizes a specific type of robustness will be no more suitable when user preference changes. In real applications, it is significant to take user preference into consideration and design a general framework suitable to various types of robustness. Therefore, we design UOR-RL as a general framework, which can be applied to satisfy a variety of preference to robustness.
As far as we know, UOR is the first RL framework to take user preference into consideration.% \textcolor{blue}{ and  provides sufficient both theoretical and experimental analysis.} 

% , especially when users get involved and take diverse preferences. 

%We fill the gap by proposing the novel UOR which considers all environments with weights specified by user preference, and takes the worst-case and average-case as special cases in particular.
%these two singleton cases 
%In contrast, our UOR-RL framework tackles the discrepancy in terms of transition dynamics and reward functions between training and execution environments, which also attracted popular attentions . 
%Moreover, instead of only focusing on the worst-case environment, our UOR-RL framework considers all possible environments, and weighs them differently according to user preference. As far as we know, UOR-RL is the first framework to consider user preference in robustness with both theoretical and empirical guarantees.
%\paragraph{Generalization in RL.} 
%The generalization ability of RL policies is required when the testing environment misaligns with the training environment \cite{cobbe2019quantifying, igl2019generalization, kumar2020one}. 
%Recent works have designed various techniques, including ensembles of models \cite{epopt, kahn2017uncertainty}, domain randomization \cite{tobin2017domain, jiang2021monotonic} and distributional critics \cite{bodnar2019quantile, tang2019worst} to quantify and optimize over environment uncertainties. 

\section{Conclusion}
In this paper, we propose the UOR metric, which integrates user preference into the measurement of robustness. Aiming at optimizing such metric, we design two UOR-RL training algorithms, which work in the scenarios with or without a priori known environment distribution, respectively. Theoretically, we prove that the output policies of the UOR-RL training algorithms, in the scenarios with accurate, inaccurate or even completely no knowledge of the environment distribution, are all $\epsilon$-suboptimal to the optimal policy. Also, we conduct extensive experiments in 6 MuJoCo tasks, and the results validate that UOR-RL is comparable to the state-of-the-art baselines under traditional metrics and establishes new state-of-the-art performance under the UOR metric.

% The main contributions of this paper are summarized as follows.
% \begin{itemize}
%     \item We propose a user-oriented metric for robustness measurement, namely UOR, which allocates different weights to different environments according to user preference. To the best of our knowledge, UOR is the first metric that integrates user preference into the measurement of robustness.
%     \item We design two UOR-RL training algorithms for the scenarios with or without a priori known environment distribution, respectively. Both algorithms take the UOR metric as the optimization objective so as to obtain policies better aligned with user preference. 
%     \item We prove through rigorous theoretical analysis a series of results showing that our UOR-RL training algorithms converge to near-optimal policies even with inaccurate or completely no knowledge about the environment distribution. 
%     \item We conduct extensive experiments in 4 MoJoCo tasks. The experimental results demonstrate that UOR-RL is comparable to the state-of-the-art baselines under the average and worst-case performance metric, and more importantly establishes new state-of-the-art performance under the UOR metric.
%     \end{itemize}

\section{Acknowledgements}
This work was supported by NSF China (No. U21A20519, U20A20181, 61902244).
\printbibliography
\bibstyle{aaai22}
% \bibliography{reference}

\newpage
\appendix
\onecolumn
\section{Proof of Theorems and Corollary}
\subsection{Theorem ~\ref{Theorem near optimal}}
%%%%%%%%%%%%%%%%%%%%%%%%%%%%%%%%%%%%%%%%%%%%%%%%%%%%%%%%%%%%%%%%%%%%%%%%%%%%%%%
%%%%%%%%%%%%%%%%%%%%%%%%%%%%%%%%%%%%%%%%%%%%%%%%%%%%%%%%%%%%%%%%%%%%%%%%%%%%%%%

\begin{proof}~\par
We divide the proof into 3 parts.
\subsubsection{Continuity of $J(\pi,p)$}
Firstly, we want to prove $J(\pi,p)$ is continuous to the environment $p$. In the most of MDPs, the state space and action space $\mathcal{S,A}$ are bounded. Even if $S,A$ are unbounded, we can have function which maps unbounded set into a bounded one. For example, let $\mathcal{S}\in \mathbb{R}^k$, we define the function $M:\mathbb{R}^k\rightarrow [0,1]^k$ that
\begin{equation}
    M((x_1,x_2,\dots,x_k)^T)=(\frac{1}{1+\exp(-x_1)},\frac{1}{1+\exp(-x_2)},\dots,\frac{1}{1+\exp(-x_k)})^T.
\end{equation}
And if $\mathcal{S,A}$ is not closed, we can also add the boundary into them, that is
\begin{equation}
    \mathcal{S}=\mathcal{S}\cup \partial(\mathcal{S}), \mathcal{A}=\mathcal{A}\cup \partial(\mathcal{A}).
\end{equation}
As a result, w.l.o.g, we assume $\mathcal{S,A}$ are compact sets.

Furthermore, we can also assume the environment range $\mathcal{P}$ is a compact set. If $\mathcal{P}\in \mathbb{R}^d$ is unbounded, we define $B_r(0)$ be the closed ball in space $R^d$ with center at origin and radius $r$. 
From the property of distribution $\mathbb{D}$ that
\begin{equation}
    1=\int_\mathcal{P}\mathbb{D}(p)\mathrm{d}p=\int_{\mathcal{P}\cup B_\infty(0)}\mathbb{D}(p)\mathrm{d}p=\lim_{r\rightarrow\infty}\int_{\mathcal{P}\cup B_r(0)}\mathbb{D}(p)\mathrm{d}p
\end{equation}

Thus for $\forall~\varphi,\exists r_0$ that 
\begin{equation}
    \int_{\mathcal{P}\cup B_{r_0}(0)}\mathbb{D}(p)\mathrm{d}p>1-\varphi
\end{equation}
So we can thake $\mathcal{P}'=\mathcal{P}\cup B_{r_0}(0)$ to be bounded and use the Theorem~\ref{Theorem Empirical}. As a result, w.l.o.g, we can assume $\mathcal{P}$ is also bounded and compact.

Therefore, the domains of $R$ is $\mathcal{S}^2\times \mathcal{A}\times\mathcal{P}\rightarrow R$ is a compact set, so is $T$. From \textbf{Cantor}'s Theorem, we can get $R,T$ is uniformly continue.
%  Then for any given $\epsilon$, we can prove in the algorithm.
% \begin{equation}
%     \begin{aligned}
%     \pi^*=\argmax_{\pi}\mathcal{E}(\pi)\\
%     \pi'=\argmax_{\pi}\hat{\mathcal{E}}(\pi)\\
%     E(\pi^*,f,\mathbb{D})\le E(\pi',f,\mathbb{D})+\epsilon
%     \end{aligned}
% \end{equation}
So we get
\begin{equation}
\begin{aligned}
    \forall~\epsilon, \exists~\delta>0, \forall s\in \mathcal{S}, a\in \mathcal{A}, p_1,p_2\in \mathcal{P}, |p_1-p_2|\footnote{In Theorem Proof, we use the 2-norm to measure the distance. For ease of illustration, we user |x-y| to denotes 2-norm distance $||x-y||_2$}<\delta \Rightarrow |R_{p_1}(s,a)-R_{p_2}(s,a)|<\epsilon\\
        \forall~\epsilon, \exists~\delta>0, \forall s_1,s_2\in \mathcal{S}, a\in \mathcal{A}, p_1,p_2\in \mathcal{P}, |p_1-p_2|<\delta \Rightarrow |T_{p_1}(s_1,a,s_2)-T_{p_2}(s_1,a,s_2)|<\epsilon\\
\end{aligned}
\end{equation}
where for ease of illustration, we use $R_p(s,a,s'), T_p(s,a)$ respectively denote $R_p(s,a,s',p)$ and $T_p(s,a,p)$.

For $\forall$ given $\epsilon$, let set $\Delta_\epsilon\subseteq\mathbb{R^+}$ satisfies
\begin{equation}
\begin{aligned}
    \Delta_\epsilon^R:&=\{\delta~|~\forall s\in \mathcal{S}, a\in \mathcal{A}, p_1,p_2\in \mathcal{P}, |p_1-p_2|<\delta \Rightarrow |R_{p_1}(s,a)-R_{p_2}(s,a)|<\epsilon\}\\
    \Delta_\epsilon^T:&=\{\delta~|~\forall s_1,s_2\in \mathcal{S}, a\in \mathcal{A}, p_1,p_2\in \mathcal{P}, |p_1-p_2|<\delta \Rightarrow |T_{p_1}(s_1,a,s_2)-T_{p_2}(s_1,a,s_2)|<\epsilon\}
\end{aligned}
\end{equation}
$\Delta_\epsilon^R$ has upper bound $d(\mathcal{P})$ (diameter of set $\mathcal{P}$), so it has supremum $\delta_\epsilon^R$. Here we define a function 
\begin{equation}
    \delta_R(\epsilon)=\delta_\epsilon^R
\end{equation}
Similarly we define the funtion
\begin{equation}
    \delta_T(\epsilon)=\delta_\epsilon^T
\end{equation}
with $\delta_\epsilon^T$ as supremum of $\Delta_\epsilon^T$.

As for $\delta_{R}$, there are several properties
\begin{enumerate}
    \item Domain of $\delta_{R}$ is $(0,\max R-\min R]$
    \item Range of $\delta_R\subseteq (0,d(P)]$
    \item $\delta_R$ is strictly increasing.
\end{enumerate}
 Similarly, these properties are the same for $ \delta_{T}$. 
%  And since $\mathcal{S,A}$ is compact set, we can also $\forall pi$, we can have function $\delta_\pi$.
% \begin{assumption} At any episode t, agents policies $\pi_t$ is Lipschitz continuous to $s$ and $p$
% with constant $L_{\pi,s}, L_{\pi,p}$. 
% \label{assump_pi_L_continue}
% \end{assumption}

 After that, we discuss $J(\pi,p)$. We define it as 
\begin{equation}
    J(\pi,p)=\mathbb{E}[\{\sum_{i=0}^\infty \gamma^i\cdot R_p(s_i,\pi(s_i),s_{i+1})\}]
\end{equation}
 Then we compare the performance of policy $\pi$ under two environments $p_1$ and $p_2$. We have $s_0^1=s_0^2=s_0$ and $|p_1-p_2|<\delta$.
\begin{equation}
\begin{aligned}
    &J(\pi,p_1)-J(\pi,p_2)\\
    =&\mathbb{E}\left[\sum_{i=0}^\infty \gamma^i\cdot R_{p_1}(s^1_i,\pi(s^1_i,p_1),s^1_{i+1})\right]-\mathbb{E}\left[\sum_{i=0}^\infty \gamma^i\cdot R_{p_2}(s^2_i,\pi(s^2_i,p_2),s^2_{i+1})\right]\\
    =&\sum_{i=0}^\infty \gamma^i\cdot\mathbb{E}\left[ R_{p_1}(s^1_i,\pi(s^1_i),s^1_{i+1})\right]-\sum_{i=0}^\infty\gamma^i\cdot \mathbb{E}\left[R_{p_2}(s^2_i,\pi(s^2_i),s^2_{i+1})\right]\\
    =&\sum_{i=0}^\infty \gamma^i\cdot\left[\mathbb{E}\left[ R_{p_1}(s^1_i,\pi(s^1_i),s^1_{i+1})\right]-\mathbb{E}\left[R_{p_2}(s^2_i,\pi(s^2_i),s^2_{i+1})\}\right]\right]\\
\end{aligned}
\end{equation}
We know the distribution of initial state $s_0$ is $\mathbb{S}_0$. We define that under the environment parameter $p$ and executing policy $\pi$, the state distribution in the $t^{th}$ step is $\mathbb{S}_t^{\pi,p}$. And the actions distribution under the state $s$, environment parameter $p$ and executing policy $\pi$ is $\pi^p_s$. And the transition distribution from state $s$ and action $a$ under parameter $p$ is $T_{p}^{s,a}$.

Then we have 
\begin{equation}
\begin{aligned}
    \mathbb{E}\left[R_{p_1}(s^1_i,\pi(s^1_i),s^1_{i+1})\right]=&\int_\mathcal{S}\mathbb{S}_t^{\pi,p_1}(s_1)\int_\mathcal{A}\pi^{p_1}_s(a)\int_\mathcal{S}T_{p_1}^{s_1,a}(s_2)R_{p_1}(s_1,a,s_2) \mathrm{d}s_2\mathrm{d}a\mathrm{d}s_1\\
    =&R_{p_1}(\mathbb{S}_t^{\pi,p_1},\pi(\mathbb{S}_t^{\pi,p_1}),T_{p_1}(\mathbb{S}_t^{\pi,p_1},\pi(\mathbb{S}_t^{\pi,p_1}))).
\end{aligned}
\end{equation}

  Firstly, we consider the distance of distribution $\mathbb{S}_t^{\pi,p_1}$ and $\mathbb{S}_t^{\pi,p_2}$
\begin{equation}
    \begin{aligned}
    &D_{TV}(\mathbb{S}_{t+1}^{\pi,p_1},\mathbb{S}_{t+1}^{\pi,p_2})\\
    =&|T_{p_1}(\mathbb{S}_{t}^{\pi,p_1},\pi(\mathbb{S}_{t}^{\pi,p_1}))-T_{p_2}(\mathbb{S}_{t}^{\pi,p_2},\pi(\mathbb{S}_{t}^{\pi,p_2}))|\\
    =&|T_{p_1}(\mathbb{S}_{t}^{\pi,p_1},\pi(\mathbb{S}_{t}^{\pi,p_1}))-T_{p_1}(\mathbb{S}_{t}^{\pi,p_2},\pi(\mathbb{S}_{t}^{\pi,p_1}))+T_{p_1}(\mathbb{S}_{t}^{\pi,p_2},\pi(\mathbb{S}_{t}^{\pi,p_1}))-T_{p_1}(\mathbb{S}_{t}^{\pi,p_2},\pi(\mathbb{S}_{t}^{\pi,p_2}))\\
    &~+T_{p_1}(\mathbb{S}_{t}^{\pi,p_2},\pi(\mathbb{S}_{t}^{\pi,p_2}))-T_{p_2}(\mathbb{S}_{t}^{\pi,p_2},\pi(\mathbb{S}_{t}^{\pi,p_2}))|\\
    \le&L_{T,S}\cdot[ D_{TV}(\mathbb{S}_{t}^{\pi,p_1},\mathbb{S}_{t}^{\pi,p_2}) +L_{T,A} L_\pi\cdot D_{TV}(\mathbb{S}_{t}^{\pi,p_1},\mathbb{S}_{t}^{\pi,p_2})]+ \delta_{T}^{-1}(\delta)\\
    =&(L_{T,S}+L_{T,A}L_\pi)\cdot D_{TV}(\mathbb{S}_{t}^{\pi,p_1},\mathbb{S}_{t}^{\pi,p_2})+\delta_{T}^{-1}(\delta)
    \end{aligned}
\end{equation}
 Then we signal $\Delta^i$ and define a new function $dis_\delta(x)$.
\begin{equation}
    \Delta^i:=D_{TV}(\mathbb{S}_{t}^{\pi,p_1},\mathbb{S}_{t}^{\pi,p_2})~~~~~~~~~dis_\delta(x):=(L_{T,S}+L_{T,A}L_\pi)\cdot x+\delta_{T}^{-1}(\delta).
\end{equation}
 Then we have $\Delta^{i+1}\le dis_\delta(\Delta^i)$
 Let $dis_\delta^{(2)}=dis_\delta\circ dis_\delta$ and $dis_\delta^{(k)}=dis_\delta^{(k+1)}\circ dis_\delta$\\
So we can have
\begin{equation}
    \Delta^{i}\le dis_\delta(\Delta^{i-1})\le dis_\delta^{(2)}(\Delta^{i-2})\dots \le dis_\delta^{(i)}(\Delta^0)=dis_\delta^{(i)}(|s_0^1-s_0^2|)=dis_\delta^{(i)}(0).
\end{equation}
 Let $\alpha=L_{T,S}+L_{T,A}L_\pi,~ \beta(\delta)=\delta_{T}^{-1}(\delta)$, then we can simplify the formula that
\begin{equation}
    \begin{aligned}
        dis_\delta^{(i)}(0)&=\alpha\cdot dis_\delta^{(i-1)}(0)+\beta(\delta)\\
        &=\alpha^2\cdot dis_\delta^{(i-2)}(0)+\alpha\cdot \beta(\delta)+\beta(\delta)\\
        &\dots\\
        &=\alpha^{i-1}\cdot dis_\delta(0)+\beta(\delta)\cdot \sum_{j=0}^{i-2}\alpha^j\\
        &=\left\{\begin{aligned}
        &\frac{(\alpha^i-1)\cdot \beta(\delta)}{\alpha-1}~~~~~~~~~~~~&\alpha>1\\
        &i\cdot \alpha\cdot \beta(\delta)~~~~~~~~&\alpha=1\\
        &\frac{(1-\alpha^i)\cdot \beta(\delta)}{1-\alpha}~~~~~~~~~~~~&\alpha<1\\
        \end{aligned}\right.
    \end{aligned}
\end{equation}
Then consider the distance of reward function $R$
\begin{equation}
\begin{aligned}
    &\left|R_{p_1}(\mathbb{S}_t^{\pi,p_1},\pi(\mathbb{S}_t^{\pi,p_1}),T_{p_1}(\mathbb{S}_t^{\pi,p_1},\pi(\mathbb{S}_t^{\pi,p_1})))-R_{p_2}(\mathbb{S}_t^{\pi,p_2},\pi(\mathbb{S}_t^{\pi,p_2}),T_{p_2}(\mathbb{S}_t^{\pi,p_2},\pi(\mathbb{S}_t^{\pi,p_2})))\right|\\
    % =&|R_{p_1}(s^1_i,\pi(s^1_i),s^1_{i+1})-R_{p_1}(s^2_i,\pi(s^1_i),s^1_{i+1})+R_{p_1}(s^2_i,\pi(s^1_i),s^1_{i+1})-R_{p_1}(s^2_i,\pi(s^2_i),s^1_{i+1})\\+&R_{p_1}(s^2_i,\pi(s^2_i),s^1_{i+1})-R_{p_1}(s^2_i,\pi(s^2_i),s^2_{i+1})+R_{p_1}(s^2_i,\pi(s^2_i),s^2_{i+1})-R_{p_2}(s^2_i,\pi(s^2_i),s^2_{i+1})|\\
    % \le& L_{R,S}\cdot |s_i^1-s_i^2|+L_{R,A}\cdot L_{\pi}\cdot |s_i^1-s_i^2|+L_{R,S}\cdot |s_{i+1}^1-s_{i+1}^2|+ \delta_{R}^{-1}(\delta)\\
    % \le& L_{R,S}\cdot |s_i^1-s_i^2|+L_{R,A}\cdot L_{\pi}\cdot |s_i^1-s_i^2|+L_{R,S}\cdot dis_\delta(|s_i^1-s_i^2|)+ 
    \le&|R_{p_1}(\mathbb{S}_t^{\pi,p_1},\pi(\mathbb{S}_t^{\pi,p_1}),T_{p_1}(\mathbb{S}_t^{\pi,p_1},\pi(\mathbb{S}_t^{\pi,p_1})))-R_{p_1}(\mathbb{S}_t^{\pi,p_2},\pi(\mathbb{S}_t^{\pi,p_1}),T_{p_1}(\mathbb{S}_t^{\pi,p_1},\pi(\mathbb{S}_t^{\pi,p_1})))|\\
    +&|R_{p_1}(\mathbb{S}_t^{\pi,p_2},\pi(\mathbb{S}_t^{\pi,p_1}),T_{p_1}(\mathbb{S}_t^{\pi,p_1},\pi(\mathbb{S}_t^{\pi,p_1})))-R_{p_1}(\mathbb{S}_t^{\pi,p_2},\pi(\mathbb{S}_t^{\pi,p_2}),T_{p_1}(\mathbb{S}_t^{\pi,p_1},\pi(\mathbb{S}_t^{\pi,p_1})))|\\
    % +&|R_{p_1}(\mathbb{S}_t^{\pi,p_2},\pi(\mathbb{S}_t^{\pi,p_2},p_1),T_{p_1}(\mathbb{S}_t^{\pi,p_1},\pi(\mathbb{S}_t^{\pi,p_1},p_1)))-R_{p_1}(\mathbb{S}_t^{\pi,p_2},\pi(\mathbb{S}_t^{\pi,p_2},p_2),T_{p_1}(\mathbb{S}_t^{\pi,p_1},\pi(\mathbb{S}_t^{\pi,p_1},p_1)))|\\
    +&|R_{p_1}(\mathbb{S}_t^{\pi,p_2},\pi(\mathbb{S}_t^{\pi,p_2}),T_{p_1}(\mathbb{S}_t^{\pi,p_1},\pi(\mathbb{S}_t^{\pi,p_1})))-R_{p_1}(\mathbb{S}_t^{\pi,p_2},\pi(\mathbb{S}_t^{\pi,p_2}),T_{p_2}(\mathbb{S}_t^{\pi,p_2},\pi(\mathbb{S}_t^{\pi,p_2})))|\\
    +&|R_{p_1}(\mathbb{S}_t^{\pi,p_2},\pi(\mathbb{S}_t^{\pi,p_2}),T_{p_2}(\mathbb{S}_t^{\pi,p_2},\pi(\mathbb{S}_t^{\pi,p_2})))-R_{p_2}(\mathbb{S}_t^{\pi,p_2},\pi(\mathbb{S}_t^{\pi,p_2}),T_{p_2}(\mathbb{S}_t^{\pi,p_2},\pi(\mathbb{S}_t^{\pi,p_2})))|\\
    \le& L_{R,S}\Delta^i+L_{R,A}L_\pi(\Delta^i)+L_{R,S}(\alpha\Delta^i+\beta(\delta))+\delta_R^{-1}(\delta)
\end{aligned}
\end{equation}
 Let $\zeta=L_{R,S}+L_{R,A}\cdot L_{\pi}+L_{R,S}\cdot (L_{T,S}+L_{T,A}\cdot L_\pi)$ and $\eta(\delta)=L_{R,S}\cdot  \delta_{T}^{-1}(\delta)+ \delta_{R}^{-1}(\delta)$, then we can find a upper bound of the distance of $J(\pi,p_1)$ and $J(\pi,p_2)$. 
\begin{equation}
    \begin{aligned}
    &J(\pi,p_1)-J(\pi,p_2)\\
     =&\sum_{i=0}^\infty \gamma^i\cdot\left[\mathbb{E}\left[ R_{p_1}(s^1_i,\pi(s^1_i),s^1_{i+1})\right]-\mathbb{E}\left[R_{p_2}(s^2_i,\pi(s^2_i),s^2_{i+1})\}\right]\right]\\
     \le&\sum_{i=0}^\infty \gamma^i(\zeta\cdot \Delta^i+\eta(\delta))\\
     =&\frac{\eta(\delta)}{1-\gamma}+\zeta\cdot \sum_{i=0}^\infty \gamma^i\Delta^i
    \end{aligned}
\end{equation}
 Here we need discuss respectively with the relationship with 1 and $\alpha$
\begin{equation}
    \begin{aligned}
\sum_{i=0}^\infty \gamma^i\Delta^i=\left\{\begin{aligned}
        &\sum_{i=0}^\infty \gamma^i\cdot \frac{(\alpha^i-1)\cdot \beta(\delta)}{\alpha-1}\le \frac{\beta(\delta)}{\alpha-1}\sum_{i=0}^\infty (\gamma\cdot \alpha)^i\le  \frac{\beta(\delta)}{(\alpha-1)(1-\gamma\cdot \alpha)}~~~~&\alpha>1\\
        &\sum_{i=0}^\infty \gamma^i\cdot i\cdot \alpha\cdot \beta(\delta)=\frac{\gamma\cdot \alpha\cdot \beta(\delta)}{(1-\gamma)^2} &\alpha=1\\
        &\sum_{i=0}^\infty \gamma^i\cdot \frac{(1-\alpha^i)\cdot \beta(\delta)}{1-\alpha}\le\sum_{i=0}^\infty \gamma^i\cdot \frac{\beta(\delta)}{1-\alpha}=\frac{\beta(\delta)}{(1-\alpha)\cdot (1-\gamma)} &\alpha<1\\
\end{aligned}\right.
    \end{aligned}
\end{equation}
  Since $\alpha,\zeta,\gamma$ are all fixed constants. We can set
\begin{equation}
\begin{aligned}
    \mu_0&=\max\{\frac{1}{(\alpha-1)(1-\gamma\cdot \alpha)},\frac{\gamma\cdot \alpha}{(1-\gamma)^2},\frac{1}{(1-\alpha)\cdot (1-\gamma)}\}\\
    &\lambda=\frac{1}{1-\gamma}~~~~~~\mu=\mu_0\cdot \zeta
\end{aligned}
\end{equation}
So we can conclude
\begin{equation}
    |J(\pi,p_1)-J(\pi,p_2)|\le\lambda\cdot \eta(\delta)+\mu\cdot \beta(\delta)    =:j(\delta)
\end{equation}
\subsubsection{$\mathcal{E}(\pi)~~\mathrm{and}~~\hat{\mathcal{E}}(\pi)$}
In this section, we focus on a fixed policy $\pi$, and discuss the value between $\mathcal{E}$ and $\hat{\mathcal{E}}$ executing $\pi$.

As for preference function $W$, if we take normalized one $W'$ as the new preference function that
\begin{equation}
    W'(x)=W(x)/\mathcal{W}~~with~~~\mathcal{W}=\int_0^1W(x)\mathrm{d}x
\end{equation}
The optimal policy keeps the same. That is, we can normalize the preference function will the actually influence the UOR-PMDP. As a result, w.l.o.g, we can assume
\begin{equation}
    \int_0^1W(x)\mathrm{d}x=1.
\end{equation}
  From the calculation of $\hat{\mathcal{E}}(\pi)$, we can define a new function $H'_{\pi}$ as a step function, assume the order is $\alpha_1$ to $\alpha_n$
\begin{equation}
    \begin{aligned}
    M_i=&\sum_{k=1}^{i} m_{\alpha_i}~~~~(M_0=0)\\
    H'_{\pi}(x)=J(\pi,p_{\alpha_i})~~~~~&with~~~M_{i-1}\le x< M_{i}
    \end{aligned}
\end{equation}
From the definition of $\mathcal{E}$, we know
\begin{equation}
        \mathcal{E}(\pi)=\int_\mathcal{P}\mathbb{D}(p)\cdot J(\pi,p)\cdot W(h(\pi,p))\mathrm{d}p.
\end{equation}
We define another function $H$ that
\begin{equation}
    H(x):=\int_\mathcal{P}\mathbb{D}(p)\cdot\mathds{1}\left[J(\pi,p)\le x\right]\mathrm{d}p
\end{equation}
From the definition we get easily get
\begin{equation}
    H(J(\pi,p))=h(\pi,p).
\end{equation}
For the function $H$, we find $H$ is monotonic increasing. 
We signal $H^{-1}$ as the inverse function of $H$. It's important to mention that since $H$ may be not strictly increasing, the true inverse function may not exists, but we can let the lower bound one as the value of the $H^{-1}$, so we assume the inverse function existing.

After that we can have
\begin{equation}
\begin{aligned}
            \mathcal{E}(\pi)&=\int_\mathcal{P}\mathbb{D}(p)\cdot J(\pi,p)\cdot W(h(\pi,p))\mathrm{d}p\\
                &=\int_{\mathcal{P}} (H_\pi^{-1}(h(\pi,p))\cdot W(h(\pi,p))(\mathbb{D}(p)\mathrm{d}p)\\
    &\overset{x=h(\pi,p)}{=}\int_0^1 H_\pi^{-1}(x)\cdot W(x)\mathrm{d}x
\end{aligned}
\end{equation}
Also we have
\begin{equation}
    \hat{\mathcal{E}}(\pi)=\int_0^1 H'_\pi(x)\cdot W(x)\mathrm{d}x.
\end{equation}
As a result
\begin{equation}
\begin{aligned}
    &|\mathcal{E}(\pi)-\hat{\mathcal{E}}(\pi)|\\=&|\int_0^1 (H_\pi^{-1}(x)-H'_\pi(x))\cdot W(x)|\\
    =&|\sum_{i=0}^n \int_{M_i}^{M_{i+1}}(H_\pi^{-1}(x)-H'_\pi(x))\cdot W(x)|\\
    \le&\sum_{i=0}^n \int_{M_i}^{M_{i+1}}|H_\pi^{-1}(x)-H'_\pi(x)|\cdot W(x)
\end{aligned}
\end{equation}

We from two direction to bound $H_\pi^{-1}-H'_\pi$
\begin{enumerate}
    \item $H'_\pi(x)\ge H_\pi^{-1}(x)-j(\delta):$
    \begin{equation}
        \begin{aligned}
        &Let~~M_i\le x< M_{i+1} \Rightarrow\forall j<i, J(\pi,p_{l^j})\le H'_\pi(x)\\
        &From~~\forall k, d(\mathcal{P}_k)\le\delta\\
        \Rightarrow& \forall y\in \mathcal{P}_{l^j},|J(\pi,p_{l^j})-J(\pi,y)|<j(\delta)\\
        \Rightarrow& J(\pi,y)\le J(\pi,p_{l^j})+j(\delta)\le H'_\pi(x)+j(\delta)\\
        &Let~~\mathbb{U}_i=\bigcup_{j=0}^{i} \mathcal{P}_{\alpha_j}\\
        \Rightarrow& \forall y\in \mathbb{U}_i,J(\pi,y)<H'_\pi(x)+j(\delta)\\
        &We~know,~\int_{\mathbb{U}_i}\mathbb{D}(p)\mathrm{d}p=M_{i+1}\\
        \Rightarrow&\int_\mathcal{P} \mathds{1}[J(\pi,p)<H'_\pi(x)+j(\delta)]\cdot \mathbb{D}(p)\mathrm{d}p\ge M_{i+1}~~~~~~~~~~~~~~~~~~~~~~~~~~~~~~~~~\\
        \Rightarrow&H_\pi^{-1}(x)\le H_\pi^{-1}(M_{i+1})\le H'_\pi(x)+j(\delta)\\
        \Rightarrow&H'_\pi(x)\ge H_\pi^{-1}(x)-j(\delta)
        \end{aligned}
    \end{equation}
    \item $H'_\pi(x)\le H_\pi^{-1}(x)+j(\delta)$:
    \begin{equation}
        \begin{aligned}
        &Let~~M_i\le x< M_{i+1}\Rightarrow\forall j>i, J(\pi,p_{l^j})\ge H'_\pi(x)\\
        &From~~\forall k, d(\mathcal{P}_k)\le\delta\\
        \Rightarrow& \forall y\in \mathcal{P}_{l^j},|J(\pi,p_{l^j})-J(\pi,y)|<j(\delta)\\
        \Rightarrow& J(\pi,y)\ge J(\pi,p_{l^j})-j(\delta)\ge H'_\pi(x)-j(\delta)\\
        &Let~~\mathbb{V}_i=\bigcup_{j=i}^{n} \mathcal{P}_{\alpha^j}\\
        \Rightarrow& \forall y\in \mathbb{V}_i,J(\pi,y)\ge H'_\pi(x)-j(\delta)\\
        &We~know,~\int_{\mathbb{V}_i}\mathbb{D}(p)\mathrm{d}p=1-M_{i}\\
        \Rightarrow&\int_\mathcal{P} \mathds{1}[J(\pi,p)\ge H'_\pi(x)-j(\delta)]\cdot \mathbb{D}(p)\mathrm{d}p\ge 1-M_{i}~~~~~~~~~~~~~~~~~~~~~~~~~~~~~~~~~\\
        A&nd \int_\mathcal{P} \mathds{1}[J(\pi,p)\ge H'_\pi(x)-j(\delta)]\cdot \mathbb{D}(p)\mathrm{d}p+\int_\mathcal{P} \mathds{1}[J(\pi,p)\le H'_\pi(x)-j(\delta)]\cdot \mathbb{D}(p)\mathrm{d}p\\
        =& \int_\mathcal{P} \mathds{1}\cdot \mathbb{D}(p)\mathrm{d}p-\int_\mathcal{P} \mathds{1}[J(\pi,p)=H'_\pi(x)-j(\delta)]\cdot \mathbb{D}(p)\mathrm{d}p=1-0=1\\
        \Rightarrow&\int_\mathcal{P} \mathds{1}[J(\pi,p)\le H'_\pi(x)-j(\delta)]\cdot \mathbb{D}(p)\mathrm{d}p\le M_{i}\\
        \Rightarrow&H_\pi^{-1}(x)\ge H_\pi^{-1}(M_i)\ge H'_\pi(x)+j(\delta)\\
        \Rightarrow&H'_\pi(x)\le H_\pi^{-1}(x)+j(\delta)
        \end{aligned}
    \end{equation}
\end{enumerate}
  As a result, we can conclude $|H'_\pi(x)-H_\pi^{-1}(x)|\le j(\delta)$, Then
\begin{equation}
    \begin{aligned}
    &|\mathcal{E}(\pi)-\hat{\mathcal{E}}(\pi)|\\
    \le&\sum_{i=0}^n \int_{M_i}^{M_{i+1}}|H_\pi^{-1}(x)-H'_\pi(x)|\cdot W(x)\mathrm{d}x\\
    \le&\sum_{i=0}^n \int_{M_i}^{M_{i+1}}j(\delta)\cdot W(x)\mathrm{d}x\\
    =&j(\delta)\cdot \sum_{i=0}^n \int_{M_i}^{M_{i+1}}W(x)\mathrm{d}x\\
    =&j(\delta)\cdot \int_{0}^{1}W(x)\mathrm{d}x\\
    =&j(\delta)
    \end{aligned}
\end{equation}
with
\subsubsection{$\hat{\pi}$, $\pi^*$}

 We define these two optimal policies that
\begin{equation}
    \begin{aligned}
    \pi^*=\argmax_{\pi}\mathcal{E}(\pi)\\
    \hat{\pi}^\ast=\argmax_{\pi}\hat{\mathcal{E}}(\pi)
    \end{aligned}
\end{equation}
%  $\pi^*$ is the optimal policy we require during the training, but the algorithm can only output $\pi'$.

Consider the function $j(\delta)=\eta(\delta)+\mu\cdot \beta(\delta)$ that
\begin{equation}
\begin{aligned}
    \beta(\delta)&=L_{T,A}L_\pi\delta+\delta_{T}^{-1}(\delta)\\
    \eta(\delta)&=L_{R,S}\cdot  \delta_{T}^{-1}(\delta)+ \delta_{R}^{-1}(\delta).
\end{aligned}
\end{equation}
$ \delta_{T}$ and $ \delta_{R}$ is a strictly increasing function, and $ \delta_{T}(0)= \delta_{R}(0)=0 $. So we know $\beta(\delta)$ ,$\eta(\delta)$ and $j(\delta)$ are all strictly increasing and $\beta(0)=\eta(0)=j(0)=0$.
As a result, $j^{-1}$ exists and is strictly increasing with $j^{-1}(0)=0$.

% And that $\lambda\cdot L_{R,S}+\mu>0, \lambda>0$. As a result, $j$ is also a strictly monotonic increasing function. So $j^{-1}$ exists. 

Then for $\forall \epsilon=2\epsilon_0$, let $\delta_0=j^{-1}(\epsilon_0/2)$. As long as we have block diameter upper bound is $\delta_0$ then we can get
\begin{equation}
\forall policy~~\pi, |\mathcal{E}(\pi)-\hat{\mathcal{E}}(\pi)|\le j(\delta_0)=j(j^{-1}(\epsilon/2))=\epsilon/2.
\end{equation}
And we know the Policy\_Update function can output $\epsilon_0-suboptimal$ policy $\hat{\pi}$, that is
\begin{equation}
    \hat{\mathcal{E}}(\hat{\pi})\ge \hat{\mathcal{E}}(\hat{\pi}^\ast)-\epsilon_0
\end{equation}
Then we can conclude the proof above that
\begin{equation}
    \begin{aligned}
        \mathcal{E}(\hat{\pi})\ge\hat{\mathcal{E}}(\hat{\pi})-\epsilon_0/2\ge\hat{\mathcal{E}}(\hat{\pi}^\ast)-3\epsilon_0/2\ge\hat{\mathcal{E}}(\pi^\ast)-3\epsilon_0/2\ge\mathcal{E}(\pi^\ast)-2\epsilon_0=\mathcal{E}(\pi^\ast)-\epsilon
    \end{aligned}
\end{equation}
Then we prove this theorem.
\end{proof}
\subsection{Theorem ~\ref{Theorem near optimal strengthem}}
\label{Proof of Theorem 2}
\begin{proof}
Since the we strengthen the Assumption~\ref{Assumption p continue} and get the Lipschitz continuity one, Assumption~\ref{assumption lipschitz of p}, we can modify the proof of Theorem~\ref{Theorem near optimal} that
\begin{equation}
    \begin{aligned}
    \delta_T(\delta)=&L_{T,p}\cdot \delta\\
    \delta_R(\delta)=&L_{R,p}\cdot \delta
    \end{aligned}
\end{equation}
As a result, we can get accurate value of $\beta$ and $\eta$ that
\begin{equation}
    \begin{aligned}
    \beta(\delta)&=\frac{\delta}{L_{T,p}}\\
    \eta(\delta)&=\frac{L_{R,S}\cdot\delta}{L_{T,p}}+ \frac{\delta}{L_{R,p}}.
    \end{aligned}
\end{equation}
So we can have the $j(\delta)$ that
\begin{equation}
\begin{aligned}
j(\delta)&=\lambda\eta(\delta)+\mu\beta(\delta)\\
&=\frac{\frac{L_{R,S}\cdot\delta}{L_{T,p}}+ \frac{\delta}{L_{R,p}}}{1-\gamma}+[L_{R,S}+L_{R,A}\cdot L_{\pi}+L_{R,S}\cdot (L_{T,S}+L_{T,A}\cdot L_\pi)]\mu_0\cdot\frac{\delta}{L_{T,p}}\\
&=[\frac{L_{R,p}L_{R,S}+L_{T,p}}{L_{T,p}L_{R,p}(1-\gamma)}+\frac{(L_{R,S}+L_{R,A} L_{\pi}+L_{R,S}(L_{T,S}+L_{T,A}L_\pi))\mu_0}{L_{T,p}}]\delta\\
&=:\frac{\delta}{\upsilon}
\label{detailed equation}
\end{aligned}
\end{equation}
with
\begin{equation}
    \mu_0=\max\{\frac{1}{(\alpha-1)(1-\gamma\cdot \alpha)},\frac{\gamma\cdot \alpha}{(1-\gamma)^2},\frac{1}{(1-\alpha)\cdot (1-\gamma)}\}
\end{equation}
Then we can have $j^{-1}(\epsilon)=\upsilon\epsilon=O(\epsilon)$.
\end{proof}
\subsection{Theorem ~\ref{Theorem Empirical}}
\begin{proof}
~\par First, we define two function
\begin{equation}
    \begin{aligned}
    H_{\pi}(x)=&\int_\mathcal{P} \mathds{1}[J(\pi,p)\le x]\cdot \mathbb{D}(p)\mathrm{d}p\\
    H_{\pi,e}(x)=&\int_\mathcal{P} \mathds{1}[J(\pi,p)\le x]\cdot \mathbb{D}^e(p)\mathrm{d}p
    \end{aligned}
\end{equation}
  So from that, we know
\begin{equation}
\begin{aligned}
 \mathcal{E}(\pi)=&\int_0^1 H^{-1}_{\pi}(x)\cdot W(x)\mathrm{d}x\\
 \mathcal{E}^e(\pi)=&\int_0^1 H^{-1}_{\pi,\mathbb{D}^e}(x)\cdot W(x)\mathrm{d}x
\end{aligned}
\end{equation}
 Before that, there is a property
\begin{lemma}
\begin{equation}
    \exists J_0,\forall~policy~\pi,\forall~0\le x\le1,~\forall\mathbb{D},~ |H^{-1}_{\pi}(x)|\le J_0
\end{equation}
\label{lemma1}
\end{lemma}
\begin{proof}
~\par For $\forall~policy~\pi,\forall ~0\le x\le1,~\forall\mathbb{D}$, there exists $p_1,p_2\in \mathcal{P}$,
\begin{equation}
    J(\pi,p_1)\le H^{-1}_{\pi}(x)\le J(\pi,p_2).
\end{equation}

Exactly, we can get the $p_1=\argmin_{p\in\mathcal{P}}J(\pi,p),p_2=\argmax_{p\in\mathcal{P}}J(\pi,p)$

After that, from assumption 4,5, $R$ function is continue to $\mathcal{P}$, Lipschitz to $S$ and $A$, so $R$ is continue in its domain $\mathbb{S}^2\times \mathbb{A}\times \mathcal{P}$.

Besides, from assumption 1,2,6, we get the domain of $R$ is a compact set, with the continuity of $R$, we know range of $R$ is a compact set, so its bounded. We can assume
\begin{equation}
    |R_p(s,a,s')|\le R_0
\end{equation}

Then we can get
\begin{equation}
\begin{aligned}
    \forall p,\pi, |J(\pi,p)|&=|\mathbb{E}[\{\sum_{i=0}^\infty \gamma^i\cdot R_{p_1}(s_i,\pi(s_i),s_{i+1})\}]|\\
    &\le 
    \mathbb{E}[\sum_{i=0}^\infty\gamma^i*R_0]\\
    &=\frac{R_0}{1-\gamma}
\end{aligned}
\end{equation}

Let $J_0=\frac{R_0}{1-\gamma}$, we have
\begin{equation}
\begin{aligned}
    -\frac{R_0}{1-\gamma}\le J(\pi,p_1)&\le H^{-1}_{\pi}(x)\le J(\pi,p_2)\le \frac{R_0}{1-\gamma}\\
    \Longrightarrow& |H^{-1}_{\pi}(x)|\le \frac{R_0}{1-\gamma}
\end{aligned}
\end{equation}
\end{proof} 
\begin{lemma}
$\forall$ given $\epsilon>0, policy~\pi$, $\exists~\kappa=O(\epsilon)$, as long as $D_{TV}(\mathbb{D},\mathbb{D}^e)\le\kappa$,  then we can promise
\begin{equation}
    |\mathcal{E}(\pi)-\mathcal{E}^e(\pi)|\le\epsilon
\end{equation}
\label{lemma2}
\end{lemma}
\begin{proof}

~\par We divide $\mathcal{E}(\pi)$ into 2 parts.
\begin{equation}
\begin{aligned}
    \mathcal{E}(\pi)&=\int_0^1 H_\pi^{-1}(x)\cdot W(x)\mathrm{d}x\\
    &=\int_0^{0^+} H_\pi^{-1}(x)\cdot W(x)\mathrm{d}x+\int_{0^+}^1 H_\pi^{-1}(x)\cdot W(x)\mathrm{d}x\\
    :&=\mathcal{E}_1(\pi)+\mathcal{E}_2(\pi)
\end{aligned}
\end{equation}

Then we can easily get
\begin{equation}
    |\mathcal{E}(\pi)-\mathcal{E}^e(\pi)|\le |\mathcal{E}_1(\pi)-\mathcal{E}^e_1(\pi)|+|\mathcal{E}_2(\pi)-\mathcal{E}^e_2(\pi)|
\end{equation}

\paragraph{Part 1:$\mathcal{E}_1(\pi)$}
~\par From the definition, we have 
\begin{equation}
    \mathcal{E}_1(\pi)=\int_0^{0^+}W(x)\mathrm{d}x\cdot\lim_{x \to 0^+}H_\pi^{-1}(x)
\end{equation}

Let $A=\int_0^{0^+}W(x)\mathrm{d}x$, then we get 
\begin{equation}
    |\mathcal{E}_1(\pi)-\mathcal{E}^e_1(\pi)|= A\cdot\lim_{x \to 0^+}|H_\pi^{-1}(x)-H_{\pi,e}^{-1}(x)|
\end{equation}
Since the environment distribution will not change abruptly, we can mildly assume the distribution function is continuous, that is, $\mathbb{D}(p)$ is continuous to $p$. Then as $\mathcal{P}$ is a compact set, so $\mathbb{D}$ is uniformly continuous. 
Let $p_0$ be the worst-case environment parameter for distribution $\mathbb{D}$ that $h(\pi,p_0)=0$. Then, w.l.o.g, let $\mathbb{D}(p_0)=D>0$.
Then since the uniform continuity of $\mathbb{D}$, for $\frac{D}{2}$, $\exists~\delta_1$, for $\forall~p_1\in\mathcal{P}$ and $|p_0-p_1|<\delta_1$, then $|\mathbb{D}(p_0)-\mathbb{D}(p_1)|<\frac{D}{2}$.
As a result, we have
\begin{equation}
\begin{aligned}
    &\forall~p\in B_{\delta_1}(p_0), 
    \frac{D}{2}<\mathbb{D}(p)<\frac{3D}{2}\\
    \Rightarrow &\int_{B_{\delta_1}(p_0)}\mathbb{D}(p)\mathrm{d}p>c\cdot \delta_1^d \cdot\frac{D}{2},
\end{aligned}
\end{equation}
where $c\delta_1^d $ is the volume of the ball $B_{\delta_1}(p_0)$ and $c$ depends on the dimension of environment parameter space.

According to Theorem~\ref{Theorem near optimal strengthem}, for $\frac{\epsilon}{3A}, \exists~\delta_0=\frac{\upsilon\epsilon}{3A}$. Since the definition of $O()$ and that as $\epsilon\rightarrow0, \delta_0\rightarrow0$, we can let $\delta_0<\delta_1$.
Then we choose $\kappa_1=c\cdot \delta_0^d\cdot\frac{D}{2}=O(\delta_0^d)=O(\epsilon^d)$, if $D_{TV}(\mathbb{D},\mathbb{D}^e)<\kappa_0$, then 
\begin{equation}
    \int_{B_{\delta_0}(p_0)}\mathbb{D}^e(p)\mathrm{d}p>0.
\end{equation}
Therefore, $\exists~p'\in B_{\delta_0}$ that $\mathbb{D}^e(p')>0$,
then 
\begin{equation}
\lim_{x\rightarrow0^+}H^{-1}_{\pi,e}(x)\ge\lim_{x\rightarrow0^+}H^{-1}_{\pi}(x)-\epsilon.
\end{equation}
Similarly, when consider the worst case in $\mathbb{D}^e$, we can get 
\begin{equation}
\lim_{x\rightarrow0^+}H^{-1}_{\pi}(x)\ge\lim_{x\rightarrow0^+}H^{-1}_{\pi,e}(x)-\epsilon.
\end{equation}
As a result, as long as $D_{TV}(\mathbb{D},\mathbb{D}^e)\le\kappa_1$, we can guarantee

% So for the given $\epsilon,~\exists x$, as long as $d(p_1,p_2)<x$, then $|J(\pi, p_1)-J(\pi,p_2)|< \epsilon/(3A)$, and from assumption 3, we know that for the given $x$, there is $\kappa_1$ that $\int_{B_r(p)} \mathbb{D}(x)\mathrm{d}x>\kappa_1$. So we can conclude as long as $D_{TV}(\mathbb{D},\mathbb{D}^e)<\kappa_1$, then we have 
\begin{equation}
    |\mathcal{E}_1(\pi)-\mathcal{E}^e_1(\pi)|= A\cdot\lim_{x \to 0^+}|H_\pi^{-1}(x)-H_{\pi,e}^{-1}(x)|< A\cdot \frac{\epsilon}{3A}=\frac{\epsilon}{3}
\end{equation}
\paragraph{Part 2:$\mathcal{E}_2(\pi)$} 
~\par From
\begin{equation}
    0=\int_{0^+}^{0^+}W(x)\mathrm{d}x=\lim_{x\rightarrow 0^+}\int_{0^+}^{x}W(x)\mathrm{d}x
\end{equation}

we know for given $\epsilon$
\begin{equation}
\exists \sigma>0, \mathrm{s.t.} \int_{0^+}^{\sigma}W(x)\mathrm{d}x\le\frac{\epsilon}{6J_0}
\end{equation}

From the limitation of the $W$, we know $W$ is monotonic decreasing, so 
\begin{equation}
    \forall \sigma\le x\le 1, W(x)\le W(\sigma)
\end{equation}

Then we extend the $W$ function to $W_1$ that
\begin{equation}
    W_1(x)=\left\{\begin{aligned}
    W(x)~~~~~x\ge\sigma\\
    W(\sigma)~~~~~x\le\sigma
    \end{aligned}\right.
\end{equation}

So we can change the $W$ into $W_1$ that for $\forall \mathbb{D},\pi$, we have
\begin{equation}
    \begin{aligned}
    &\mathcal{E}_2(\pi)-\int_0^1 H_\pi^{-1}(x)\cdot W_1(x)\mathrm{d}x\\
    =&\mathcal{E}_2(\pi)-\int_{0^+}^1 H_\pi^{-1}(x)\cdot W_1(x)\mathrm{d}x\\
    =&\int_{0^+}^1 H_\pi^{-1}(x)\cdot [W(x)-W_1(x)]\mathrm{d}x\\
    =&\int_{0^+}^\sigma H_\pi^{-1}(x)\cdot [W(x)-W_1(x)]\mathrm{d}x\\
    \le& J_0\cdot \int_{0^+}^\sigma  [W(x)-W_1(x)]\mathrm{d}x\\
    \le& J_0\cdot \int_{0^+}^\sigma  W(x)\mathrm{d}x=\frac{\epsilon}{6}
    \end{aligned}
\end{equation}

Then we can change the $E_2$ into another form
\begin{equation}
    \begin{aligned}
    &|\mathcal{E}_2(\pi)-\mathcal{E}^e_2(\pi)|\\
    \le&|\mathcal{E}_2(\pi)-\int_0^1 H_\pi^{-1}(x)\cdot W_1(x)\mathrm{d}x-\mathcal{E}^e_2(\pi)+\int_0^1 H_{\pi,e}^{-1}(x)\cdot W_1(x)\mathrm{d}x|\\
    +&|\int_0^1 H_\pi^{-1}(x)\cdot W_1(x)\mathrm{d}x-\int_0^1 H_{\pi,e}^{-1}(x)\cdot W_1(x)\mathrm{d}x|\\
    \le&|\mathcal{E}_2(\pi)-\int_0^1 H_\pi^{-1}(x)\cdot W_1(x)\mathrm{d}x|+|\mathcal{E}^e_2(\pi)-\int_0^1 H_{\pi,e}^{-1}(x)\cdot W_1(x)\mathrm{d}x|\\
    +&|\int_0^1 H_\pi^{-1}(x)\cdot W_1(x)\mathrm{d}x-\int_0^1 H_{\pi,e}^{-1}(x)\cdot W_1(x)\mathrm{d}x|\\
    \le& \frac{\epsilon}{6}+\frac{\epsilon}{6}+|\int_0^1 H_\pi^{-1}(x)\cdot W_1(x)\mathrm{d}x-\int_0^1 H_{\pi,e}^{-1}(x)\cdot W_1(x)\mathrm{d}x|\\
    =& \frac{\epsilon}{3}+|\int_0^1 H_\pi^{-1}(x)\cdot W_1(x)\mathrm{d}x-\int_0^1 H_{\pi,e}^{-1}(x)\cdot W_1(x)\mathrm{d}x|
    \end{aligned}
\end{equation}
% \begin{equation}
% \begin{aligned}
% \Rightarrow& |\int_{0^+}^{\sigma} H_\pi^{-1}(x)\cdot W(x)\mathrm{d}x-\int_{0^+}^{\sigma} H_{\pi,e}^{-1}(x)\cdot W(x)\mathrm{d}x|\\
% =&|\int_{0^+}^{\sigma} [H_\pi^{-1}(x)^{-1}(x)-H_{\pi,e}^{-1}(x)^{-1}(x)]\cdot W(x)\mathrm{d}x|\\
% \le&\int_{0^+}^{\sigma} |H_\pi^{-1}(x)^{-1}(x)-H_{\pi,e}^{-1}(x)^{-1}(x)|\cdot W(x)\mathrm{d}x\\
% \le&2J_0\cdot\int_{0^+}^{\sigma}W(x)\mathrm{d}x\\
% \le& 2J_0\cdot\frac{\epsilon}{6J_0}=\frac{\epsilon}{3}
% \end{aligned}
% \end{equation}

% And for the rest of the part, from the limitation of the $f$, we know $f$ is monotonic decreasing, so 
% \begin{equation}
%     \forall \sigma\le x\le 1, W(x)\le W(\sigma)
% \end{equation}

% Then we extend the $f$ function to $W_1$ that
% \begin{equation}
%     W_1(x)=\left\{\begin{aligned}
%     W(x)~~~~~x\ge\sigma\\
%     W(\sigma)~~~~~x\le\sigma
%     \end{aligned}\right.
% \end{equation}

% So we can know $\forall 0\le x\le1,~W_1(x)\le W(\sigma)$, and $\int_0^1W_1(x)\le 1$

And we also have 
\begin{equation}
\begin{aligned}
    =&|\int_{\sigma}^1 H_\pi^{-1}(x)\cdot W_1(x)\mathrm{d}x-\int_{\sigma}^1 H_{\pi,e}^{-1}(x)\cdot W_1(x)\mathrm{d}x|\\
    =&|\int_{\mathcal{P}}J(\pi,p)\cdot\mathbb{D}(p)\cdot
    W_1(h(\pi,p))\mathrm{d}p-\int_{\mathcal{P}}J(\pi,p)\cdot\mathbb{D}^e(p)\cdot W_1(h^e(\pi,p)) \mathrm{d}p|\\
    \le &\int_{\mathcal{P}}J(\pi,p)\cdot|\mathbb{D}(p)\cdot
    W_1(h(\pi,p))-\mathbb{D}^e(p)\cdot W_1(h^e(\pi,p))|\mathrm{d}p\\
    \le & J_0\cdot \int_{\mathcal{P}}|\mathbb{D}(p)\cdot
    W_1(h(\pi,p))-\mathbb{D}^e(p)\cdot W_1(h^e(\pi,p))|\mathrm{d}p\\
    \le& J_0\cdot[\int_{\mathcal{P}}\mathbb{D}(p)\cdot|    W_1(h(\pi,p))-W_1(h^e(\pi,p))|+|\mathbb{D}(p)-\mathbb{D}^e(p)|\cdot W_1(h^e(\pi,p))\mathrm{d}p]\\
\end{aligned}
\end{equation}

% Let $\eta(\delta)=\frac{\epsilon}{6J_0}$.
we consider that $W_1$ is Lipschitz continuous in $[\sigma,1]$, which means
\begin{equation}
    \forall~\frac{\epsilon}{6J_0},\exists \kappa_2=O(\epsilon)>0, |x_1-x_2|<\kappa_2\Rightarrow|W_1(x_1)-W_1(x_2)|\le\frac{\epsilon}{6J_0}
\end{equation}

So for the distance of $h$, we can get
\begin{equation}
\begin{aligned}
    &|h(\pi,p)-h^e(\pi,p)|\\
    =&|\int_{\mathcal{P}}\mathbb{D}(p)\cdot \mathds{1}[J(\pi,p)\ge J(\pi,p')]\mathrm{d}p'-\int_{\mathcal{P}}\mathbb{D}^e(p)\cdot \mathds{1}[J(\pi,p)\ge J(\pi,p')]\mathrm{d}p'|\\
    \le&\int_{\mathcal{P}}|\mathbb{D}(p)-\mathbb{D}^e(p)|\cdot \mathds{1}[J(\pi,p)\ge J(\pi,p')]\mathrm{d}p'\\
    \le&\int_{\mathcal{P}}|\mathbb{D}(p)-\mathbb{D}^e(p)|\cdot\mathrm{d}p'=D_{TV}(\mathbb{D},\mathbb{D}^e)
\end{aligned}
\end{equation}

If we guarantee $D_{TV}(\mathbb{D},\mathbb{D}^e)<\kappa_2$, then we have
\begin{equation}
    |W_1(h(\pi,p))-W_1(h^e(\pi,p))|\le\frac{\epsilon}{6J_0}
\end{equation}

Let $\kappa_3=\frac{\epsilon}{6\cdot J_0\cdot W(\sigma)}=O(\epsilon)$, then if guarantee $D_{TV}(\mathbb{D},\mathbb{D}^e)\le\kappa_3$, then we have
\begin{equation}
\begin{aligned}
    =&|\int_{\sigma}^1 H_\pi^{-1}(x)\cdot W_1(x)\mathrm{d}x-\int_{\sigma}^1 H_{\pi,e}^{-1}(x)\cdot W_1(x)\mathrm{d}x|\\
    \le& J_0\cdot[\int_{\mathcal{P}}\mathbb{D}(p)\cdot|    W_1(h(\pi,p))-W_1(h^e(\pi,p))|+|\mathbb{D}(p)-\mathbb{D}^e(p)|\cdot W_1(h^e(\pi,p))\mathrm{d}p]\\
    \le& J_0\cdot\frac{\epsilon}{6J_0}\cdot \int_{\mathcal{P}}\mathbb{D}(p)\mathrm{d}p+J_0\cdot W(\sigma)\cdot \int_{\mathcal{P}}|\mathbb{D}(p)-\mathbb{D}^e(p)|\mathrm{d}p\\
    \le& J_0\cdot\frac{\epsilon}{6J_0}+J_0\cdot W(\sigma)\cdot\kappa_3\\
    =&J_0\cdot \frac{\epsilon}{6J_0}+J_0\cdot W(\sigma)\cdot \frac{\epsilon}{6\cdot J_0\cdot W(\sigma)}\\
    =&\frac{\epsilon}{6}+\frac{\epsilon}{6}=\frac{\epsilon}{3}
\end{aligned}
\end{equation}
\paragraph{Conclusion}
~\par Let $\kappa=\min\{\kappa_1,\kappa_2,\kappa_3\}=\min\{O(\epsilon^d),O(\epsilon),O(\epsilon)\}=O(\epsilon^d)$, and we let $D_{TV}(\mathbb{D},\mathbb{D}^e)<\kappa$, so we can get 
\begin{equation}
    \begin{aligned}
    &|\mathcal{E}(\pi)-\mathcal{E}^e(\pi)|\\
    \le&|\mathcal{E}_1(\pi)-\mathcal{E}^e_1(\pi)|+|\mathcal{E}_2(\pi)-\mathcal{E}^e_2(\pi)|\\
    \le&\frac{\epsilon}{3}+\frac{\epsilon}{3}+|\int_0^1 H_\pi^{-1}(x)\cdot W_1(x)\mathrm{d}x-\int_0^1 H_{\pi,e}^{-1}(x)\cdot W_1(x)\mathrm{d}x|\\
    \le& \frac{\epsilon}{3}+\frac{\epsilon}{3}+\frac{\epsilon}{3}=\epsilon
    \end{aligned}
\end{equation}
\end{proof}
Back to \textbf{Theorem}~\ref{Theorem Empirical}, for $\forall~\epsilon$, let $\epsilon_0=\epsilon/2$, and use \textbf{Lemma~\ref{lemma2}} with $\epsilon_0$, we can get
\begin{equation}
    \begin{aligned}
    \mathcal{E}(\pi^e)&\ge E(\pi^e,f,\mathbb{D}^e)-\epsilon_0\\
    &\ge \mathcal{E}^e(\pi^e)\\&\ge \mathcal{E}(\pi)-\epsilon_0-\epsilon_0\\
    &=\mathcal{E}(\pi)-\epsilon
    \end{aligned}
\end{equation}
\label{Proof of Distribution Free}
\end{proof}
\subsection{Corollary~\ref{Corollary Empirical}}
\begin{proof}
Since the Update\_Policy learn an $\epsilon$-suboptimal policy, we have
\begin{equation}
\mathcal{E}^e(\hat{\pi}^e)\ge\mathcal{E}^e(\pi^e)-\epsilon.  
\end{equation}
And according to Lemma~\ref{lemma2} with $\epsilon$, we get
\begin{equation}
    \begin{aligned}
    \mathcal{E}(\hat{\pi}^e)\ge&\mathcal{E}^e(\hat{\pi}^e)-\epsilon\\
    \ge&\mathcal{E}^e(\pi^e)-2\epsilon\\
    \ge&\mathcal{E}^e(\pi^\ast)-2\epsilon\\
    \ge&\mathcal{E}(\pi^\ast)-3\epsilon\\
    =&\mathcal{E}(\pi^\ast)-\epsilon_1
    \end{aligned}
\end{equation}
\end{proof}
\subsection{Theorem ~\ref{Theorem distribution function free}}
In Algorithm~\ref{Algorithm distribution free}, for each trajectory $\xi_{j,k}$, we regard its environment parameter $p_{j,k}$ as a random variable. For $\forall~j,k$, the environment parameter $p_{j,k}$ is a random  variable. These environment parameter all follow a priori distribution $\mathbb{D}$ but is not independent to each other, for it is mild to assume the environment will not change abruptly. In Algorithm~\ref{Algorithm distribution free}, we sample these trajetories one by one. For ease of illustration, we mark the first trajectories in each cluster $\xi_{j,0}$ as the $\xi_{j}$ to represent the whole cluster, and $p_{j,0}$ as $p_j$.

Then we define the posterior distribution $\mathbb{D}_x^y$, which means the probability distribution of random variable $p_y$ after observing the $p_x$. Then we propose the following assumption
\begin{assumption}
The environment parameter change continuously between trajectories that the distance between to consecutive environment parameter is not too larger that
\begin{equation}
    \forall~j,k, |p_{j,k}-p_{j.k+1}|<\frac{\epsilon}{2\upsilon m_1},
\end{equation}
where $\upsilon$ is defined in Theorem~\ref{Theorem near optimal strengthem} and $m_1$ is the cluster size.

The posterior distribution probability of $\mathbb{D}_x^y$ will converge to its stable distribution $\mathbb{D}$ with the rate that
\begin{equation}
    \sum_{y=x}^\infty D_{TV}(\mathbb{D},\mathbb{D}_x^y)<\infty,
\end{equation}
which means the posterior converge rate is faster than harmonic series.
\label{Assumption_5_complete}
\end{assumption}
Then we prove the Theorem~\ref{Theorem distribution function free}.
\begin{proof}
\begin{lemma}
For given $\epsilon_0$ and confidence $\rho_0$, let
\begin{equation}
    n_2=\frac{-8\cdot R_{max}^2\cdot\ln\rho_0}{(1-\gamma)^2\cdot\epsilon_0^2}=O(-\frac{\ln\rho_0}{\epsilon_0^2})
\end{equation}

then let $\xi_{j,0}$ responding to the environment $p_j$, we will have
\begin{equation}
    |J(\pi,p_j)-J_j|\le \epsilon_0/2
\end{equation}
\label{lemma_trajectories_accurate}
\end{lemma}
\begin{proof}
~\par
Let trajectory $\xi_{j,k}$ responding environment parameter $p_{j,k}$. So we have $\mathbb{E}(J_{j,k})=J(\pi,p_{j,k})$. 

Then let 
\begin{equation}
    \begin{aligned}
        &\overline{J_{j,k}}=(\sum_{i=0}^{k-1}J_{j,i})/k\\
    \end{aligned}
\end{equation}

From the assumption~\ref{Assumption_5_complete}, we have
\begin{equation}
\begin{aligned}
    &\forall~l_1,l_2\le m_2,|p_{j,l_1}-p_{j,l_2}|\le \upsilon\epsilon_0/2\\
    \Rightarrow& |J(\pi,p_{j,l_1})-J(\pi,p_{j,l_2})|\le \epsilon_0/2
\end{aligned}
\end{equation}

So we get 
\begin{equation}
    |\mathbb{E}(\overline{J_{j,k}})-J(\pi,p_j)|\le \epsilon_0/2
\end{equation}

And from \textbf{Hoeffding's inequality}, we get
\begin{equation}
    Pr[\overline{J_{j,m-1}}-\mathbb{E}(\overline{J_{j,m-1}})|\ge \frac{\epsilon_0}{2}]\le \exp(-\frac{2\cdot(\frac{\epsilon_0}{2})^2\cdot n_2^2}{\sum_{i=0}^{n_2-1}|b_i-a_i|^2})=\exp(-\frac{2\cdot(\frac{\epsilon_0}{2})^2\cdot n_2}{\cdot(\frac{2R_{max}}{1-\gamma})^2})
\end{equation}

Let $\rho_0=\exp(-\frac{2\cdot\frac{\epsilon_0}{2}^2\cdot n_2}{\cdot\frac{2R_{max}}{1-\gamma}^2})$, then we have
\begin{equation}
    n_2= \frac{-8\cdot R_{max}^2\cdot\ln\rho_0}{(1-\gamma)^2\cdot\epsilon_0^2}
\end{equation}
\end{proof}

Also we have the following lemma

% Then for $\forall~p_0$, define $X_i=\mathds{1}\left[J(\pi,p_i)<J(\pi,p_0)\right]$. Since $p_i$ is a martingale, $X_i$ is also a martingale. And from $p_i\sim\mathbb{D}$, we have $\mathbb{E}(X_i)=h(\pi,p)$. Let $Y_i=X_i-h(\pi,p), S_i=\sum_{k=1}^i Y_i$. From $\mathbb{E}(Y_i=0)$. And from the \textbf{Theorem~\ref{theorem_empirical}}, for given $\epsilon$, $\exists \tau$. For given $\rho$ and $\tau$, and from the \textbf{Azuma–Hoeffding inequality} and $0\le Y_i\le 1$, we have
% \begin{equation}
%     Pr\left[|S_i|\ge n\cdot\tau \right]\le \exp{(-\frac{2\cdot (n\cdot\tau)^2}{n\cdot 1^2})}=\exp{(-2\cdot n\cdot\tau^2)}
% \end{equation}
% Let $\exp{(-2\cdot n\cdot\tau^2)}=\rho$, we have $n=\frac{-\ln\rho}{2\cdot\tau^2}$. 
% So from the proof \textbf{Theorem~\ref{theorem_empirical}}, let $\mathbb{D_i}=\sum_{k=1}^i\frac{\delta(p_k)}{i}$, we know $|H^{-1}_{\pi}(x)-H^{-1}_{\pi,\mathbb{D_i}}(x)|\le \tau$, then $|\mathcal{E}(\pi)-E(\pi,f,\mathbb{D_i})|\le \epsilon$
\begin{lemma}
Given block size $\delta=\upsilon\epsilon_2$, we divide environment set $\mathcal{P}$ into $N$ blocks.($N=(\frac{d(\mathcal{P})}{\delta})^d$) Let every $J_j$ represents the environment $p_j$ and the rate of $p_j$ in $\mathcal{P}_i$ is $m'_i$. Given confidence $\rho_1$ and error $\epsilon_2$, as long as batch size is larger than $n_1=O(\frac{-N^2\cdot\ln\rho_1}{\epsilon_2^2})$, we can promise that There is chance larger than $1-\rho_1$ that $\forall~1\le i\le N$, $|m'_i-m_i|\le \epsilon_2/N$.
\label{lemma_trajectories_in_block}
\end{lemma}
\begin{proof}
~\par
For the given $\epsilon_2$ and $\delta$. Then divide the environment domain $\mathcal{P}$ into $n$ blocks, for each block $\mathcal{P}_i$, its diameter is not larger than $\delta$. So the number of blocks is $N\le (\frac{d(\mathcal{P})}{\delta})^{k}$ with $\mathcal{P}\subseteq \mathbb{R}^k$. We define the random variables $X_n^i$ based on the block $\mathcal{P}_i$ that
\begin{equation}
    X_j^i=\mathds{1}\left[p_j\in \mathcal{P}_i\right]
\end{equation}
 Let $m_i=\int_{\mathcal{P}_i}\mathbb{D}(p)\mathrm{d}p$. We define a function $f$ that
\begin{equation}
    f(\overline{X^i})=\sum_{j=1}^{n_1} X_j^i-j\cdot m_i
\end{equation}
Then we define another random variable $Z_i$ that
\begin{equation}
    Z_t^i=\mathbb{E}\left[f(\overline{X^j})|\overline{X^i_t}\right]
\end{equation}
$\{Z_t^i\}$ is a \textbf{Doob sequence}. So $Z_t^i$ is a martingale with
\begin{equation}
    Z_0^i=\mathbb{E}\left[f(\overline{X^j})\right]=0,~~~Z_n^i=\mathbb{E}\left[f(\overline{X^j})|\overline{X^j}\right]=f(\overline{X^j})
\end{equation}
Considering the relationship between $Z_t^i$ and $Z_{t-1}^i$. Based on the \textbf{Assumption~\ref{Assumption_5_complete}}.
Let $\sum_{y=1}^\infty D_{TV}(\mathbb{D},\mathbb{D}_1^y)=B$,
% \begin{equation}
% \end{equation}
We can get 
\begin{equation}
    \begin{aligned}
        |Z_t^i-Z_{t-1}^i|=&\left|\mathbb{E}\left[\sum_{j=t}^{n_1} X_j^i|p_{t}\right]-\mathbb{E}\left[\sum_{j=t}^{n_1} X_j^i|p_{t-1}\right]\right|\\
        \le&\left|X_t^i-\mathbb{E}\left[X_{n_1}^i|p_{t-1}\right]\right|+\left|\mathbb{E}\left[\sum_{j=t+1}^{n_1}\left( X_j^i|p_{t}-X_{j-1}^i|p_{t-1}\right)\right]\right|\\
        \le& 1+\left|\sum_{j=t+1}^{n_1} D_{TV}(\mathbb{D}_t^j-\mathbb{D}_{t-1}^{j-1})\right|\\
        \le& 1+\left|\sum_{j=t+1}^{n_1} D_{TV}(\mathbb{D}_t^j-\mathbb{D})\right|+\left|\sum_{j=t+1}^{n_1} D_{TV}(\mathbb{D}_{t-1}^{j-1}-\mathbb{D})\right|\\
        \le&1+\left|\sum_{j=t+1}^\infty D_{TV}(\mathbb{D}_t^j-\mathbb{D})\right|+\left|\sum_{j=t+1}^\infty D_{TV}(\mathbb{D}_{t-1}^{j-1}-\mathbb{D})\right|\\
        \le&1+2B:=c
    \end{aligned}
\end{equation}
Since $d(\mathcal{P}_i)=\delta$. (We can regard $\mathcal{P}_i$ as a cube of edge length $\delta$). So we can get
\begin{equation}
    \left|Z_t^i-Z_{t-1}^i\right|\le c:=c_t^i
\end{equation}
And from \textbf{Azuma–Hoeffding inequality}, we can know
\begin{equation}
    Pr\left[|Z_n^i-Z_0^i|\ge\frac{\epsilon_2}{N}\cdot n_1\right]\le 2\cdot\exp{(\frac{-(\frac{\epsilon_2}{N}\cdot n)^2}{2\cdot\sum_{j=1}^{n_1}c^i_j})}=2\cdot\exp{(\frac{-\epsilon_2^2\cdot n_1}{2\cdot c^2\cdot N^2})}
\end{equation}
Let $\rho_1\ge2\cdot\exp{(\frac{-\epsilon_2^2\cdot n_1}{2\cdot c^2\cdot N^2})}$, we have 
\begin{equation}
    n\ge \frac{-2\cdot c^2\cdot N^2\cdot\ln(\rho_1/2)}{\epsilon_2^2}=O(\frac{-N^2\cdot\ln\rho_1}{\epsilon_2^2})
\end{equation}
From the definition we know $m'_i=\frac{1}{n}\sum_{j=1}^n \mathds{1}\left[p_i\in\mathcal{P}_i\right]$. Then we will have that we have no less than $1-\rho_0$ confidence that 
\begin{equation}
    |m'_i-m_i|\le \epsilon_2/N
\end{equation}
\end{proof}
Then we compare the $\tilde{\mathcal{E}}(\pi)$ and $\hat{\mathcal{E}}(\pi)$. In the Set\_Division Algorithm (Appendix~\ref{set division function}), we divide the $\mathcal{P}$ into $N$ blocks $\left(\mathcal{P}_1\sim\mathcal{P}_N,N\le (\frac{d(\mathcal{P})}{\delta})^k\right)$, so we have
\begin{equation}
\begin{aligned}
    \hat{\mathcal{E}}(\pi)=\sum_{j=1}^N w_{\alpha_i}\cdot &J(\pi,p_{\alpha_j})\\ w_{\alpha_j}=\int_{M_s}^{M_{s+1}}W(x)\mathrm{d}x,~~&M_s=\sum_{j=1}^{s-1}m_{\alpha_j}
\end{aligned}
\end{equation}
% Just means after the rearrange the $p_t$ is at the $s^{th}$ position.

We know $\int_0^1 W(x)\mathrm{d}x=1$, so $\sum_{j=1}^N w_{\alpha_j}=1.$ We define a function $g$ that 
\begin{equation}
    g(x)=J(\pi,p_i)~~with~~\sum_{j=1}^{i-1}w_{\alpha_j}\le x\le \sum_{j=1}^{i}w_{\alpha_j}
\end{equation}
$g(x)$ is a piecewise function on $[0,1]$, so it is Riemann intergable, and it has the property that
\begin{equation}
    \int_0^1 g(x)\mathrm{d}x=\sum_{j=1}^N w_{\alpha_j}\cdot J(\pi,p_t)
\end{equation}
We can change another form of the $\hat{\mathcal{E}}(\pi)$. Since we know $\forall p_i, |J(\pi,p_i)|\le J_0$, using the \textbf{Abel transformation}, we have

\begin{equation}
    \hat{\mathcal{E}}(\pi)=\int_{-J_0}^{J_0} m\left[g\ge y\right]\mathrm{d}y-J_0
\end{equation}
where $m$ is the measurement.
Similarly, we have
\begin{equation}
\begin{aligned}
    \tilde{\mathcal{E}}(\pi)=\sum_{j=1}^n w_{\alpha_j}\cdot J_t\\
    w_{\alpha_i}=\int_{\frac{s}{\delta}}^{\frac{s+1}{\delta}}W(x)\mathrm{d}x
\end{aligned}
\end{equation}
We also define a function $g'$ that 
\begin{equation}
    g'(x)=J_i~~with~~\sum_{j=1}^{i-1}w_{\alpha_i}\le x\le \sum_{j=1}^{i}w_{\alpha_i}
\end{equation}
$g(x)$ is a piecewise function on $[0,1]$, so it is Riemann intergable, and it has the property that
\begin{equation}
    \int_0^1 g'(x)\mathrm{d}x=\sum_{j=1}^N w_{\alpha_i}\cdot J_t
\end{equation}
And also use \textbf{Abel transformation}, we have
\begin{equation}
    \tilde{\mathcal{E}}(\pi)=\int_{-J_0}^{J_0} m\left[g'\ge y\right]\mathrm{d}y-J_0
\end{equation}
Then we compare $\forall y,m\left[g\ge y\right]$ and $m\left[g'\ge y\right]$. We have
\begin{equation}
\begin{aligned}
    m\left[g\le y\right]&=\sum_{j=1}^N m_t\cdot\mathds{1}\left[J(\pi,p_t)< y\right]\\
    m\left[g'\le y\right]&=\sum_{t=1}^n \frac{1}{n}\cdot\mathds{1}\left[J_t< y\right]
\end{aligned}
\end{equation}
I want to proof the following lemma that
\begin{lemma}
\label{lemma_m[g<y]}
For the relation between $m\left[g\le y\right]$ and $m\left[g'\le y\right]$, we have the confidence of $1-\rho_0-\rho_1$ that
\begin{equation}
    \begin{aligned}
        m\left[g'\ge y-2\epsilon_1\right]\le m\left[g\ge y\right]+\epsilon_2\\
        m\left[g'\ge y+2\epsilon_1\right]\ge m\left[g\ge y\right]-\epsilon_2
    \end{aligned}
\end{equation}
\end{lemma}
\begin{proof}
~\par 
Firstly, Based on the \textbf{Lemma~\ref{lemma_trajectories_accurate}},$\forall J_i$, we have $|J_i-J(\pi,p_{i,0})|\le\epsilon_1$ with confidence $1-\rho_0$. And based on the \textbf{Lemma~\ref{lemma_trajectories_in_block}},  for $\forall 1\le j\le n$. Then we have the following  with confidence of $1-\rho_1$, we can promise $|m'_{j}-m_{j}|\le \epsilon_2/N$. So with the confidence of $(1-\rho_0)\cdot(1-\rho_1)\ge 1-\rho_0-\rho_1$, we can promise the above property both hold.

Assume $p_{y1},p_{y2},\dots p_{ys}$ satisfies $J(\pi,p)< y$, $p_{yi}$ in block $\mathcal{P}_{yi}$. Since $\delta=\epsilon_1/j$, we have  $\forall p\in \bigcup_{i=1}^s \mathcal{P}_{yi}$, $J(\pi,p)< y+\epsilon_1$.  And we have $J_i\le(\pi,p_{i,0})+\epsilon_1$. As a result, if $p_{i,0}\in \bigcup_{i=1}^s \mathcal{P}_{yi}$, we will know $J_i\le J(\pi,p_{i,0})+\epsilon_1< y+2\epsilon_1$

\begin{equation}
    \begin{aligned}
        &\sum_{t=1}^n \frac{1}{n}\cdot\mathds{1}\left[J_t< y+2\epsilon_1\right]\le \sum_{t=1}^n \frac{1}{n}\cdot\mathds{1}\left[p_{t,0}\in \bigcup_{i=1}^s \mathcal{P}_{yi}\right]\\
       &=\sum_{i=1}^s m'_{yi}\le \sum_{i=1}^s (m_{yi}+\epsilon_2/N)\\
       &\le\sum_{i=1}^s m_{yi}+\epsilon_2/N\cdot N\le \sum_{t=1}^N m_t\cdot\mathds{1}\left[J(\pi,p_t)< y\right]+\epsilon_2\\
       &\Rightarrow m\left[g'\le y+2\epsilon_1\right]\le m\left[g\le y\right]+\epsilon_2\\
       &\Rightarrow m\left[g'\ge y+2\epsilon_1\right]\ge m\left[g\ge y\right]-\epsilon_2 
    \end{aligned}
\end{equation}

Similarly, assume $p_{y1},p_{y2},\dots p_{ys}$ satisfies $J(\pi,p)\ge y$, $p_{yi}$ in block $\mathcal{P}_{yi}$. Since $\delta=\epsilon_1/j$, we have  $\forall p\in \bigcup_{i=1}^s \mathcal{P}_{yi}$, $J(\pi,p)\ge y-\epsilon_1$.  And we have $J_i\ge(\pi,p_{i,0})-\epsilon_1$. As a result, if $p_{i,0}\in \bigcup_{i=1}^s \mathcal{P}_{yi}$, we will know $J_i\ge J(\pi,p_{i,0})-\epsilon_1\ge y-2\epsilon_1$

So
\begin{equation}
    \begin{aligned}
        &\sum_{t=1}^n \frac{1}{n}\cdot\mathds{1}\left[J_t\ge y-2\epsilon_1\right]\le \sum_{t=1}^n \frac{1}{n}\cdot\mathds{1}\left[p_{t,0}\in \bigcup_{i=1}^s \mathcal{P}_{yi}\right]\\
       &=\sum_{i=1}^s m'_{yi}\le \sum_{i=1}^s (m_yi+\epsilon_2/N)\\
       &\le\sum_{i=1}^s m_{yi}+\epsilon_2/N\cdot N\le \sum_{t=1}^N m_t\cdot\mathds{1}\left[J(\pi,p_t)\ge y\right]+\epsilon_2\\
       &\Rightarrow m\left[g'\ge y-2\epsilon_1\right]\le m\left[g\ge y\right]+\epsilon_2\\
    \end{aligned}
\end{equation}
\end{proof}

Using the $\textbf{Lemma~\ref{lemma_m[g<y]}}$, we can finish the proof of Theorem.
First can bound the $\tilde{E}(\pi)$
\begin{equation}
    \begin{aligned}
        \tilde{\mathcal{E}}(\pi)&=\int_{-J_0}^{J_0} m\left[g'\ge y\right]\mathrm{d}y-J_0\\ 
        &=\int_{-J_0}^{J_0} m\left[g'\ge y\right]\mathrm{d}(y+2\epsilon)-J_0\\
        &\overset{y'=y+2\epsilon_1}{=}\int_{-J_0+2\epsilon}^{J_0+2\epsilon_1} m\left[g'\ge y'-2\epsilon_1\right]\mathrm{d}(y')-J_0\\
        &\ge \int_{-J_0+2\epsilon_1}^{J_0+2\epsilon_1} (m\left[g\ge y\right]-\epsilon_2)\mathrm{d}(y)-J_0\\
        &=\int_{-J_0}^{J_0+2\epsilon_1} m\left[g\ge y\right]\mathrm{d}(y)-\int_{-J_0}^{-J_0+2\epsilon_1} m\left[g\ge y\right]\mathrm{d}(y)-2J_0\cdot\epsilon_2-J_0\\
        &\ge\int_{-J_0}^{J_0}m\left[g\ge y\right]\mathrm{d}(y)-J_0-2\epsilon_1-2J_0\cdot\epsilon_2\\
        &=\hat{\mathcal{E}}(\pi)-2(\epsilon_1+\epsilon_2\cdot J_0)
    \end{aligned}
\end{equation}

Similarly, we also have
\begin{equation}
    \begin{aligned}
        \tilde{\mathcal{E}}(\pi)&=\int_{-J_0}^{J_0} m\left[g'\ge y\right]\mathrm{d}y-J_0\\ 
        &=\int_{-J_0}^{J_0} m\left[g'\ge y\right]\mathrm{d}(y-2\epsilon)-J_0\\
        &\overset{y'=y+2\epsilon_1}{=}\int_{-J_0+2\epsilon}^{J_0+2\epsilon_1} m\left[g'\ge y'+2\epsilon_1\right]\mathrm{d}(y')-J_0\\
        &\le \int_{-J_0-2\epsilon_1}^{J_0-2\epsilon_1} (m\left[g\ge y\right]+\epsilon_2)\mathrm{d}(y)-J_0\\
        &=\int_{-J_0}^{J_0} m\left[g\ge y\right]\mathrm{d}(y)+\int_{-J_0-2\epsilon_1}^{-J_0}m\left[g\ge y\right]\mathrm{d}(y)-\int_{J_0-2\epsilon}^{J_0} m\left[g\ge y\right]\mathrm{d}(y)-2J_0\cdot\epsilon_2-J_0\\
        &\le\int_{-J_0}^{J_0}m\left[g\ge y\right]\mathrm{d}(y)-J_0+2\epsilon_1+2J_0\cdot\epsilon_2\\
        &=\hat{\mathcal{E}}(\pi)+2(\epsilon_1+\epsilon_2\cdot J_0)
    \end{aligned}
\end{equation}

And from the \textbf{Theorem~\ref{Theorem near optimal strengthem}}, we have $\left|\hat{\mathcal{E}}(\pi)-\mathcal{E}(\pi)\right|\le \epsilon/2$. And we set the parameters that
\begin{equation}
    \epsilon_1=\frac{\epsilon}{8},\epsilon_2=\frac{\epsilon}{8J_0},\rho_0=\rho_1=\frac{\rho}{2}
\end{equation}
Then we have the confidence $1-\rho_0-\rho_1=1-\rho$ that 
\begin{equation}
    |\mathcal{E}(\pi)-\tilde{\mathcal{E}}(\pi)|=|\mathcal{E}(\pi)-\hat{\mathcal{E}}(\pi)|+|\hat{\mathcal{E}}(\pi)-\tilde{\mathcal{E}}(\pi)|\le \epsilon/2+\epsilon/2=\epsilon
    \label{fomulate}
\end{equation}
Under this condition, we have
\begin{equation}
    \begin{aligned}
        n_1&=O(\frac{-\ln\rho_0}{\epsilon_1^2})=O(\frac{-\ln\rho}{\epsilon^2})\\
        n_2&=O(\frac{-N^2\cdot\ln\rho_1}{\epsilon_2^2})=O(\frac{-(d(\mathcal{P})/\delta)^{2d}\cdot\ln\rho}{\epsilon^2})=O(\frac{-\ln\rho}{\epsilon^{2d+2}})
    \end{aligned}
\end{equation}
We define the optimal policy in metric $\tilde{E}$ that
\begin{equation}
    \tilde{\pi}^\ast=\argmax_{\pi}\tilde{\mathcal{E}}(\pi).
\end{equation}
For the given optimality-requirement $\epsilon=2\epsilon_0$, we can have the Update\_Policy to get $\epsilon_0$-suboptimal policy that
\begin{equation}
    \tilde{\mathcal{E}}(\tilde{\pi})>\tilde{\mathcal{E}}(\tilde{\pi}^\ast)-\epsilon_0.
\end{equation}
Then we use the formulation~\ref{fomulate} with $\frac{\epsilon_0}{2}$.
Finally, we can conclude that
\begin{equation}
    \begin{aligned}
    \mathcal{E}(\tilde{\pi})&>\tilde{\mathcal{E}}(\tilde{\pi})-\frac{\epsilon_0}{2}  >\tilde{\mathcal{E}}(\tilde{\pi}^\ast)-\frac{3\epsilon_0}{2}\\
    &>\tilde{\mathcal{E}}(\pi^\ast)-\frac{3\epsilon_0}{2}    >\mathcal{E}(\pi^\ast)-2\epsilon_0\\
    &=\mathcal{E}(\pi^\ast)-\epsilon.
    \end{aligned}
\end{equation}
\end{proof}
\section{Set Division Algorithm}
\label{set division function}
% \centering
% \resizebox{0.75\textwidth}{!}
% \begin{minipage}{0.75\textwidth}
\begin{algorithm}[ht]
\caption{Set Division Algorithm}
\label{Algorithm set division}
    \normalsize
    % % \INPUT Environment Set $\mathcal{P}$, Block size $\delta$
    \tcp{Initialization.}
    Initialize block set $\mathcal{B}$, environment space dimension $d$, and diameter upper bound $\delta$\;\label{a3l1}
    % $\mathbb{S}_\mathcal{P}\leftarrow \emptyset$\;\label{adl
    $\delta'=\delta/\sqrt{d}$\;\label{a3l2}
    \ForEach{Dimension $i=1,2,\dots,d$\label{a3l3}}
    {
    $L_i\leftarrow\inf_{p^i}\{p=(p^1,p^2,\dots p^d)\in \mathcal{P}\}$\;\label{a3l4}
    $R_i\leftarrow\sup_{p^i}\{p=(p^1,p^2,\dots p^d)\in \mathcal{P}\}$\;\label{a3l5}
    $n_i\leftarrow\lceil (R_i-L_i)/\delta'\rceil$\;\label{a3l6}
    % $l_i\leftarrow L_i$\;\label{a3l7}
    }
    
    \tcp{Set Division.}
    \ForEach{Dimension $i=1$ to $d$\label{a3l7} }
    {
    \ForEach{Block Index $t_i=1$ to $n_i$\label{a3l8}}
    {\ForEach{Dimension $i=1$ to $d$\label{a3l9}}
        {
        $l_i\leftarrow L_i+\delta'\cdot(t_i-1)$\;\label{a3l10}
        $r_i\leftarrow L_i+\delta'\cdot t_i$\;\label{a3l11}
        }
        $S\leftarrow\left[l_1,r_1\right]\times\left[l_2,r_2\right]\dots\times\left[l_d,r_d\right])$\;\label{a3l12}
        $S\leftarrow S\cap \mathcal{P}$\;\label{a3l13}
        $\mathcal{B}.append(S)$\;\label{a3l14}
    }
    }
    $\{\mathcal{P}_1,\mathcal{P}_2,\cdots,\mathcal{P}_n\}\leftarrow\mathcal{B}$\;\label{a3l15}
\end{algorithm}
% \end{minipage}
% }
The Algorithm~\ref{Algorithm set division} divides the set into several cubes with edge length less than $\frac{\delta}{\sqrt{d}}$, then the diameter of the cube is less than $\delta$.
Firstly, Algorithm~\ref{Algorithm set division} initializes the block set $\mathcal{B}$, environment space dimension $d$, and the diameter upper bound $\delta$ (line~\ref{a3l1}). Then it sets the edge length of each block as $\frac{\delta}{\sqrt{d}}$. After that, we initial the range of each dimension of the environment set (line~\ref{a3l4}-\ref{a3l5}), calculates the number of division times in each dimension (line~\ref{a3l6}),. Then Algorithm~\ref{Algorithm set division} begins to divides the environment set into blocks (line~\ref{a3l7}-\ref{a3l13}). For each block, it calculates the lower bound and upper bound of each dimension (line~\ref{a3l10}-\ref{a3l11}), and uses the lower bounds and upper bounds to get the block (line~\ref{a3l12}). Besides, Algorithm~\ref{Algorithm set division} takes the intersection of the obtained block and $\mathcal{P}$ to guarantee $\mathcal{S}\subset \mathcal{P}$ (line~\ref{a3l13}), and adds $\mathcal{S}$ into the block set $\mathcal{B}$ (line~\ref{a3l14}). Finally, Algorithm~\ref{Algorithm set division} returns all of the blocks (line~\ref{a3l15}).
\section{Additional Experimental Results}
\subsection{Random Seed Settings and Parameter Settings for Training}
\label{para_training}

In the experiments, we choose 6 different random seeds in $\{2,4,6,8,10,12\}$ for each task. In training, we train models for UOR-RL and baselines under these 6 random seeds; and correspondingly in testing, we compare the performance of the models under these 6 seeds as well.

In training, according to the specific process of each method, the environment parameters for the baseline algorithms and UOR-RL are selected as described below. For DR-U, EPOpt and MRPO, according to the process of the baseline algorithms, the environment parameters are set to be uniformly distributed, as in Table \ref{env_para_u}. And for DR-G, DB-UOR-RL and DF-UOR-RL, consistent with that during testing, the environment parameters are sampled following the Gaussian distributions truncated over the range given in Table \ref{env_para}.

\begin{table}[htbp!]
    \caption{Environment Parameter Settings for Training.}
    \centering
    % \resizebox{\textwidth}{!}{
    \begin{tabular}{c|c|c|c}
    %\hline
    \Xhline{1.25pt}
    \textbf{Task} & \textbf {Parameters} & \textbf{Range $\mathcal{P}$} & \textbf{Distribution $\mathbb{D}$}\\
    \Xhline{1.25pt}
    % \multirow{2}{*}{Walker2d}  & Density $\in$ [750,1250] & $\mu = 1000$ ; $\sigma = 83.3$\\ 
    % & Friction $\in$ [0.5, 1.1] & $\mu = 0.8$ ; $\sigma = 0.1$\\ \hline
    \multirow{2}{*}{Reacher} & Body size & [0.008,0.05]& $\mathcal U(0.008,0.05)$\\  
    & Body length & [0.1,0.13] & $\mathcal U(0.1,0.13)$\\ \hline
    \multirow{2}{*}{Hopper} &  Density & [750,1250] & $\mathcal U(750,1250)$\\  
    & Friction & [0.5, 1.1] & $\mathcal U(0.5,1.1)$\\ \hline
    \multirow{2}{*}{Half Cheetah} &  Density & [750,1250] & $\mathcal U(750,1250)$\\  
    & Friction & [0.5, 1.1] & $\mathcal U(0.5,1.1)$\\ \hline
    \multirow{2}{*}{Humanoid} &  Density & [750,1250] & $\mathcal U(750,1250)$\\  
    & Friction & [0.5, 1.1]& $\mathcal U(0.5,1.1)$\\ \hline
    \multirow{2}{*}{Ant} &  Density & [750,1250] & $\mathcal U(750,1250)$\\  
    & Friction & [0.5, 1.1]& $\mathcal U(0.5,1.1)$\\ \hline
    \multirow{2}{*}{Walker 2d} &  Density & [750,1250] & $\mathcal U(750,1250)$\\  
    & Friction & [0.5, 1.1]& $\mathcal U(0.5,1.1)$\\ 
    %\hline
    \Xhline{1.25pt}
    \end{tabular}
    \label{env_para_u}
\end{table}
% }

In addition, we list the final hyper-parameter settings for training UOR-RL and baselines in Table \ref{para_uorrl} and \ref{para_baseline}.

\begin{table}[htbp!]
    \caption{Final Hyper-parameter Settings for Training UOR-RL.}
    \centering
    % \resizebox{\textwidth}{!}{
    \begin{tabular}{c|c|c|c|c|c|c|c}
    %\hline
    \Xhline{1.25pt}
    \textbf{Task} & \textbf{Algorithms} & \textbf{LR of Actor} & \textbf{LR of Critic} & \textbf{Blocks} & \textbf{Optimizer} & \textbf{Total Iterations}\\
    \Xhline{1.25pt}
    \multirow{2}{*}{Reacher} & UOR-RL-DB & $5.0\times 10^{-4}$ & $6.0\times 10^{-4}$ & 100 & Adam & 1000\\  
    & UOR-RL-DF & $5.0\times10^{-4}$ & $6.0\times10^{-4}$ & 100 & Adam & 1000\\ \hline
    \multirow{2}{*}{Hopper} & UOR-RL-DB & $4.0\times 10^{-4}$ & $5.0\times 10^{-4}$ & 100 & Adam & 1000\\  
    & UOR-RL-DF & $4.0\times10^{-4}$ & $5.0\times10^{-4}$ & 100 & Adam & 1000\\ \hline
    \multirow{2}{*}{Half Cheetah} & UOR-RL-DB & $6.0\times 10^{-4}$ & $8.0\times 10^{-4}$ & 100 & Adam & 1000\\  
    & UOR-RL-DF & $6.0\times10^{-4}$ & $8.0\times10^{-4}$ & 100 & Adam & 1000\\ \hline
    \multirow{2}{*}{Humanoid} & UOR-RL-DB & $3.0\times 10^{-4}$ & $3.0\times 10^{-4}$ & 100 & Adam & 1000\\  
    & UOR-RL-DF & $3.0\times10^{-4}$ & $3.0\times10^{-4}$ & 100 & Adam & 1000\\ \hline
    \multirow{2}{*}{Ant} & UOR-RL-DB & $3.0\times 10^{-4}$ & $3.0\times 10^{-4}$ & 100 & Adam & 1000\\  
    & UOR-RL-DF & $3.0\times10^{-4}$ & $3.0\times10^{-4}$ & 100 & Adam & 1000\\ \hline
    \multirow{2}{*}{Walker 2d} & UOR-RL-DB & $3.0\times 10^{-4}$ & $8.0\times 10^{-4}$ & 256 & Adam & 1000\\  
    & UOR-RL-DF & $3.0\times10^{-4}$ & $8.0\times10^{-4}$ & 256 & Adam & 1000\\ \hline
    \Xhline{1.25pt}
    \end{tabular}
    \label{para_uorrl}
\end{table}
% }

\begin{table}[htbp!]
    \caption{Final Hyper-parameter Settings for Training baselines.}
    \centering
    % \resizebox{\textwidth}{!}{
    \begin{tabular}{c|c|c|c|c|c|c}
    %\hline
    \Xhline{1.25pt}
    \textbf{Task} & \textbf{Algorithms} & \textbf{LR} & \textbf{Minibatch Numbers} & \textbf{Optimizer} & \textbf{Total Episodes}\\
    \Xhline{1.25pt}
    \multirow{4}{*}{Reacher} 
    & DR-U & $5.0\times 10^{-4}$ &No Minibatch & Adam &$5.0\times 10^{4}$\\  
    & DR-G & $5.0\times10^{-4}$ & No Minibatch& Adam & $5.0\times 10^{4}$\\
    & EPOpt & $5.0\times10^{-4}$ &No Minibatch & Adam &$5.0\times 10^{4}$\\
    & MRPO & $5.0\times10^{-4}$ & No Minibatch& Adam & $5.0\times 10^{4}$\\
    \hline
    \multirow{4}{*}{Hopper} 
    & DR-U & $5.0\times 10^{-4}$ &No Minibatch & Adam &$5.0\times 10^{4}$\\  
    & DR-G & $5.0\times10^{-4}$ & No Minibatch& Adam & $5.0\times 10^{4}$\\
    & EPOpt & $5.0\times10^{-4}$ &No Minibatch & Adam &$5.0\times 10^{4}$\\
    & MRPO & $5.0\times10^{-4}$ & No Minibatch& Adam & $5.0\times 10^{4}$\\\hline
    \multirow{4}{*}{Half Cheetah}
    & DR-U & $5.0\times 10^{-4}$ &No Minibatch & Adam &$5.0\times 10^{4}$\\  
    & DR-G & $5.0\times10^{-4}$ & No Minibatch& Adam & $5.0\times 10^{4}$\\
    & EPOpt & $5.0\times10^{-4}$ &No Minibatch & Adam &$5.0\times 10^{4}$\\
    & MRPO & $5.0\times10^{-4}$ & No Minibatch& Adam & $5.0\times 10^{4}$\\\hline
    \multirow{4}{*}{Humanoid}
    & DR-U & $5.0\times 10^{-4}$ &128 & Adam &$5.0\times 10^{4}$\\  
    & DR-G & $5.0\times10^{-4}$ & 128& Adam & $5.0\times 10^{4}$\\
    & EPOpt & $5.0\times10^{-4}$ &128 & Adam &$5.0\times 10^{4}$\\
    & MRPO & $5.0\times10^{-4}$ & 128& Adam & $5.0\times 10^{4}$\\\hline
    \multirow{4}{*}{Ant}
     & DR-U & $5.0\times 10^{-4}$ &No Minibatch & Adam &$5.0\times 10^{4}$\\  
    & DR-G & $5.0\times10^{-4}$  & No Minibatch& Adam & $5.0\times 10^{4}$\\
    & EPOpt & $5.0\times10^{-4}$  &No Minibatch & Adam &$5.0\times 10^{4}$\\
    & MRPO & $5.0\times10^{-4}$ & No Minibatch& Adam & $5.0\times 10^{4}$\\\hline
    \multirow{4}{*}{Walker 2d}
    & DR-U & $5.0\times 10^{-4}$  &256 & Adam &$5.0\times 10^{4}$\\  
    & DR-G & $5.0\times10^{-4}$  & 256& Adam & $5.0\times 10^{4}$\\
    & EPOpt & $5.0\times10^{-4}$  &256 & Adam &$5.0\times 10^{4}$\\
    & MRPO & $5.0\times10^{-4}$  & 256& Adam & $5.0\times 10^{4}$\\\hline
    \Xhline{1.25pt}
    \end{tabular}
    \label{para_baseline}
\end{table}
% }

\subsection{The Supplementary Heat Maps}
\label{other_heatmaps}
\begin{figure*}[htbp!]
    \centering
    \setlength{\abovecaptionskip}{-0.1cm} 
    \includegraphics[width=\textwidth]{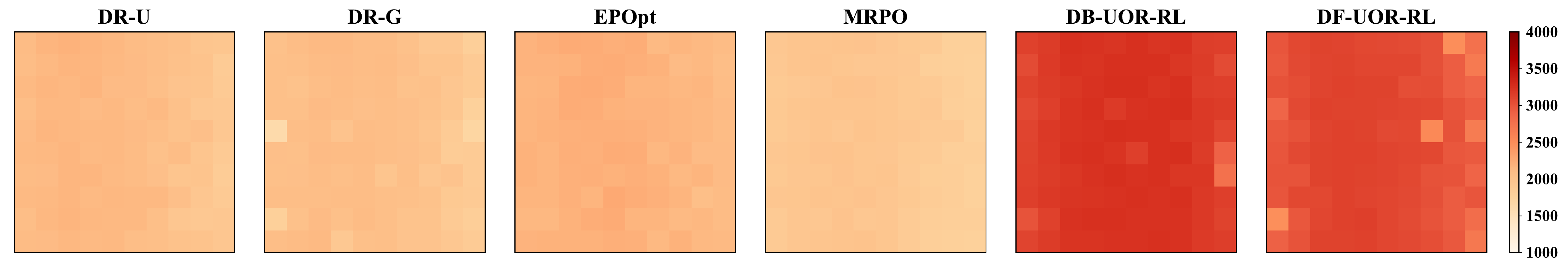}
    \caption{Heat map of $\mathcal{E}_1$ in sub-ranges (Ant). The x-axis and y-axis denote friction and density, respectively. The ranges of these two parameters are chosen as in Table \ref{env_para}, and are evenly divided into 10 sub-ranges for each one.}
    \label{Ant-v0}
\end{figure*}

\begin{figure*}[htbp!]
    \centering
    \setlength{\abovecaptionskip}{-0.1cm} 
    \includegraphics[width=\textwidth]{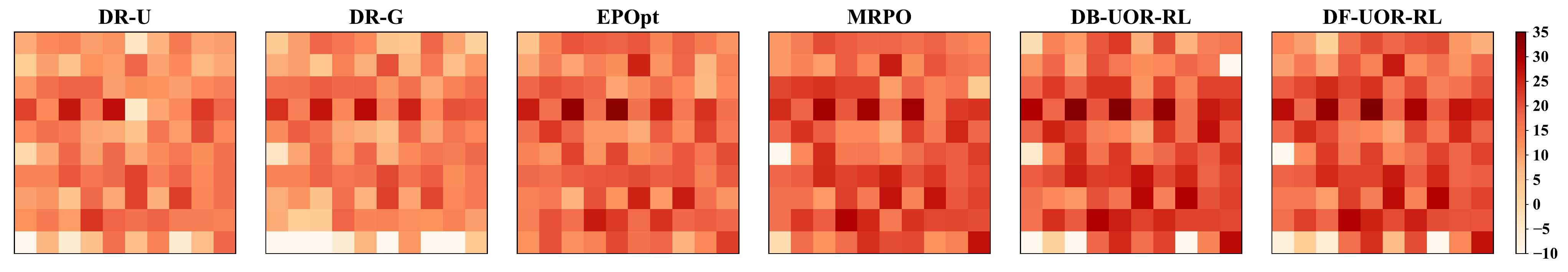}
    \caption{Heat map of $\mathcal{E}_1$ in sub-ranges (Reacher). The x-axis and y-axis denote body length and body size, respectively. The ranges of these two parameters are chosen as in Table \ref{env_para}, and are evenly divided into 10 sub-ranges for each one.}
    \label{Reacher-v0}
    \vspace{-0.2cm}
\end{figure*}

\begin{figure*}[htbp!]
    \centering
    \setlength{\abovecaptionskip}{-0.1cm} 
    \includegraphics[width=\textwidth]{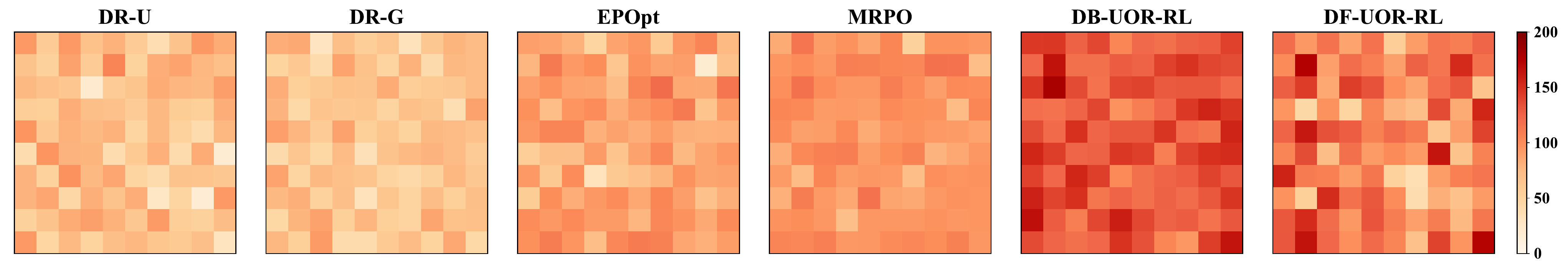}
    \caption{Heat map of $\mathcal{E}_1$ in sub-ranges (Humanoid). The x-axis and y-axis denote body length and body size, respectively. The ranges of these two parameters are chosen as in Table \ref{env_para}, and are evenly divided into 10 sub-ranges for each one.}
    \label{Humanoid-v0}
     \vspace{-0.2cm}
     
\end{figure*}

\begin{figure*}[htbp!]
    \centering
    \setlength{\abovecaptionskip}{-0.1cm} 
    \includegraphics[width=\textwidth]{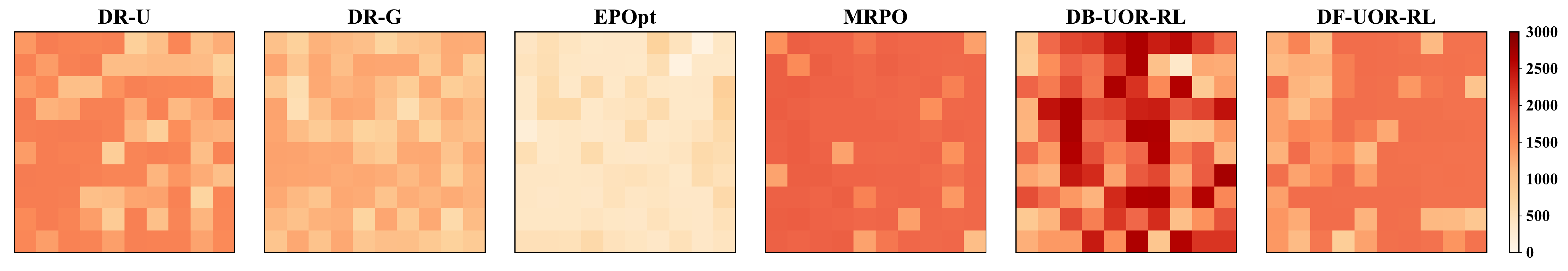}
    \caption{Heat map of $\mathcal{E}_1$ in sub-ranges (Walker 2d). The x-axis and y-axis denote body length and body size, respectively. The ranges of these two parameters are chosen as in Table \ref{env_para}, and are evenly divided into 10 sub-ranges for each one.}
    \vspace{-0.2cm}
    \label{Walker2d-v0}
\end{figure*}

Due to space limitations, we only show the heat maps of Half Cheetah and Hopper in the main text. Here, we add the heat maps for the remaining tasks(Ant, Humanoid, Reacher and Walker 2d), as shown in Figures \ref{Ant-v0} to \ref{Walker2d-v0}.

\subsection{Training Curves}

\label{train_cruve}
Training curves of the baseline algorithms and UOR-RL are shown in Figure \ref{ant_train}-\ref{reacher_train}. In the same task, the model of UOR-RL are trained in three different robustness degree $k$, while the baseline models are the same but showing the performance under different metrics. In the process of training the model, the abort condition of training is that the model has reached convergence or has been trained for 1000 iterations.

\begin{figure}[htbp!]
\centering
\setlength{\abovecaptionskip}{-0.05cm} 
\includegraphics[width=0.32\textwidth]{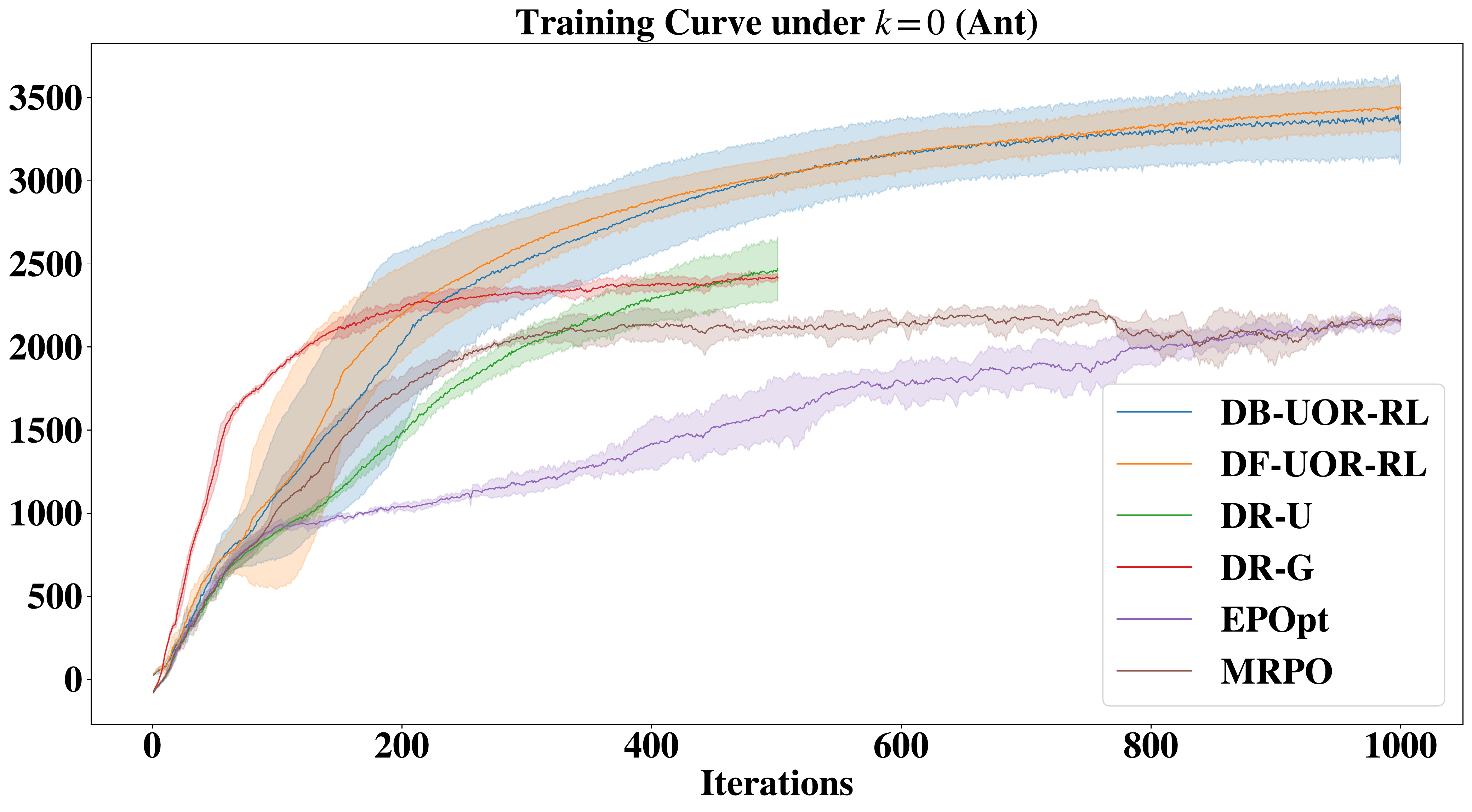}
\
\includegraphics[width=0.32\textwidth]{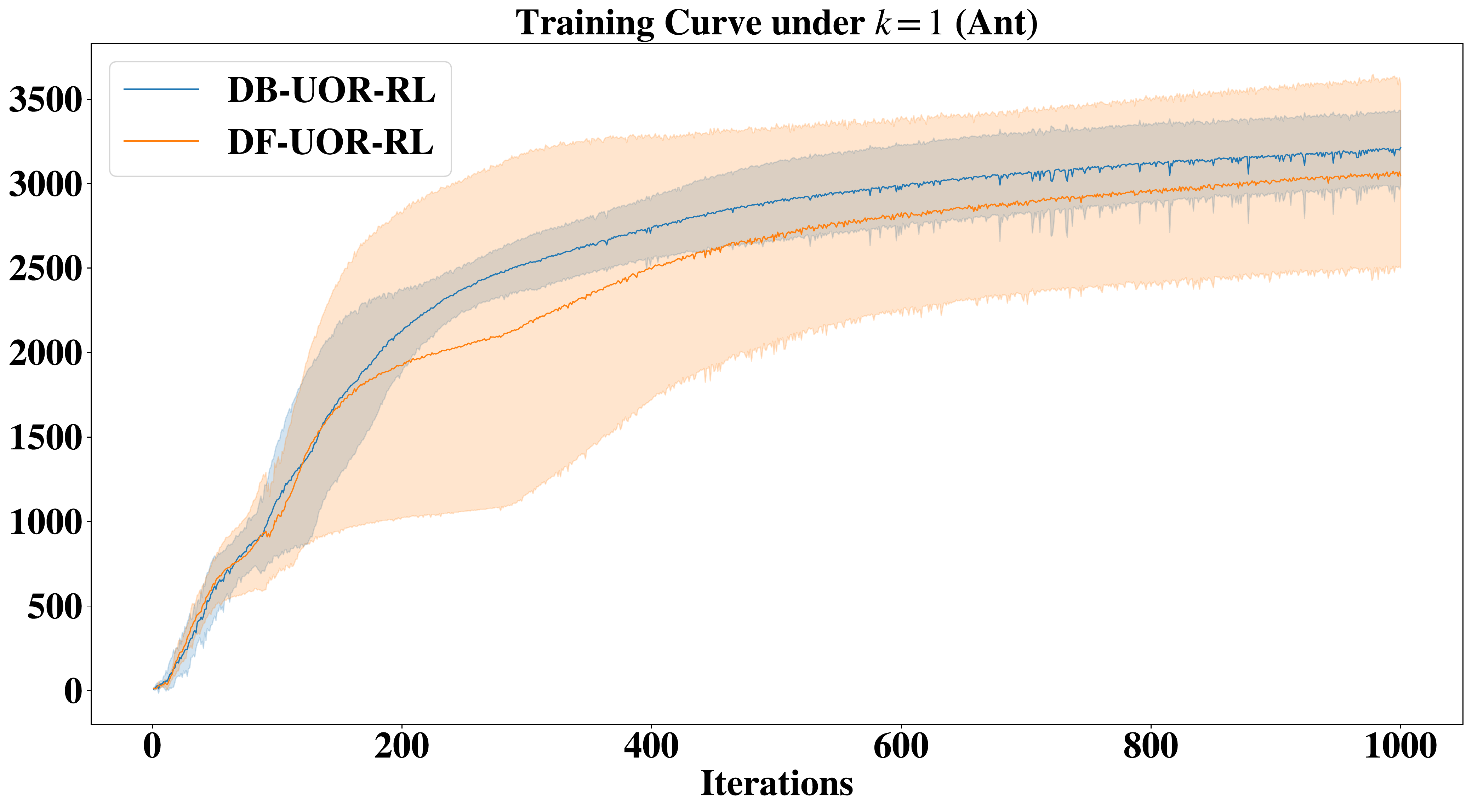}
\
\includegraphics[width=0.32\textwidth]{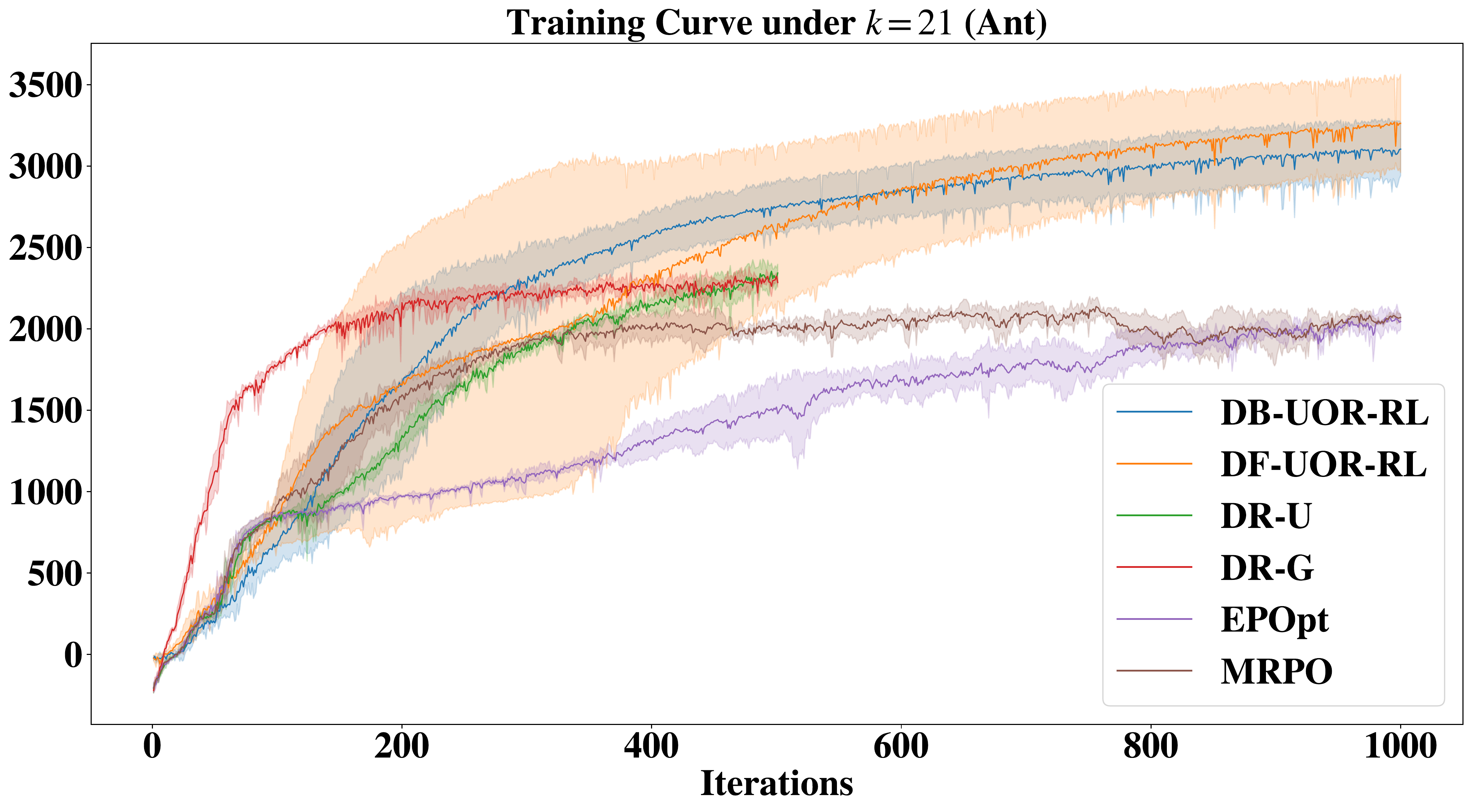}
\caption{Training Curves (Ant). In the three graphs from left to right, the UOR-RL algorithms are trained under $k=0,1$ and 21, and the y-axis  represent the average return of all trajectories, $\mathcal{E}_1$ and the average return of worst 10\% trajectories respectively. }
\label{ant_train}
\end{figure}

\begin{figure}[htbp!]
\centering
\setlength{\abovecaptionskip}{-0.05cm} 
\includegraphics[width=0.32\textwidth]{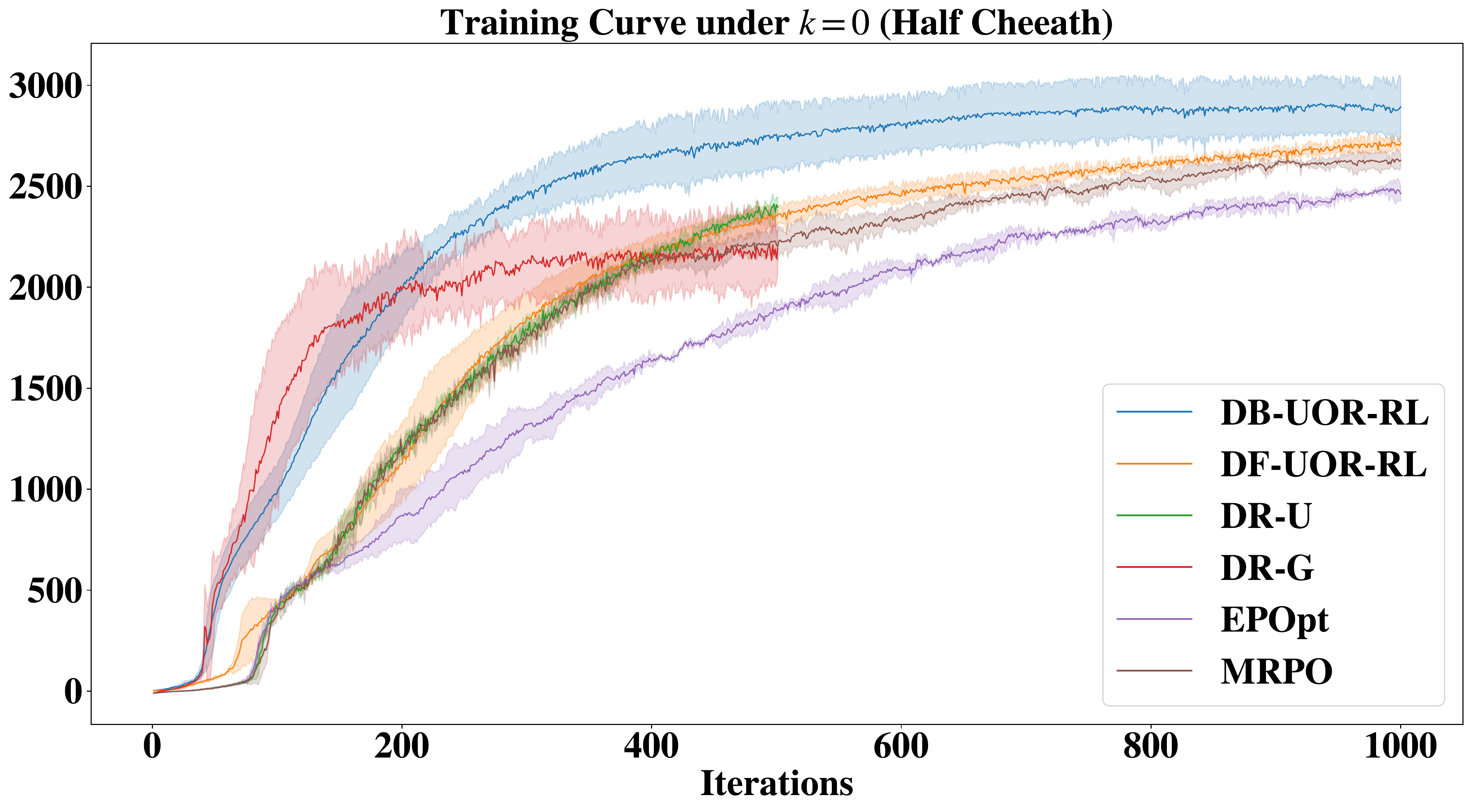}
\
\includegraphics[width=0.32\textwidth]{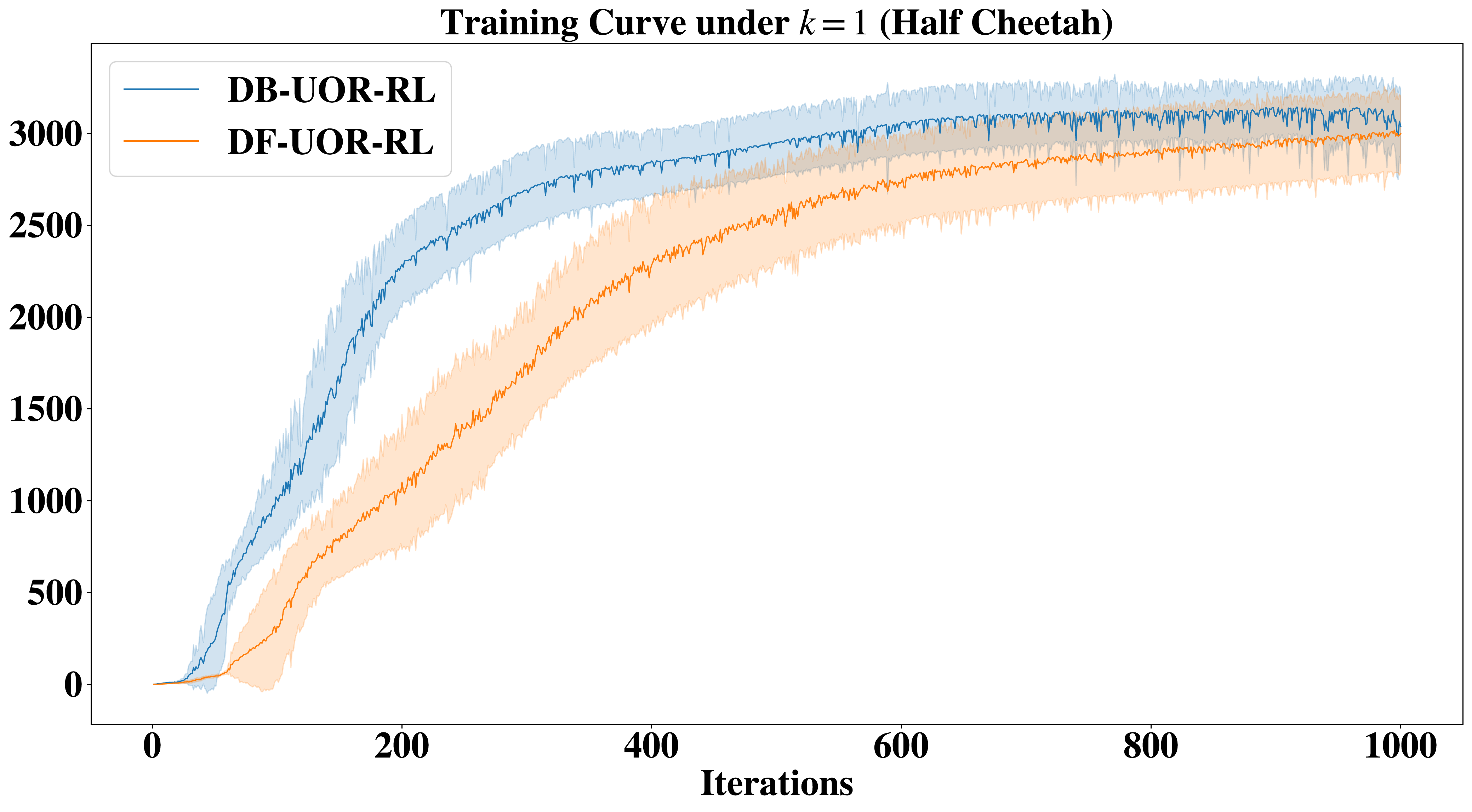}
\
\includegraphics[width=0.32\textwidth]{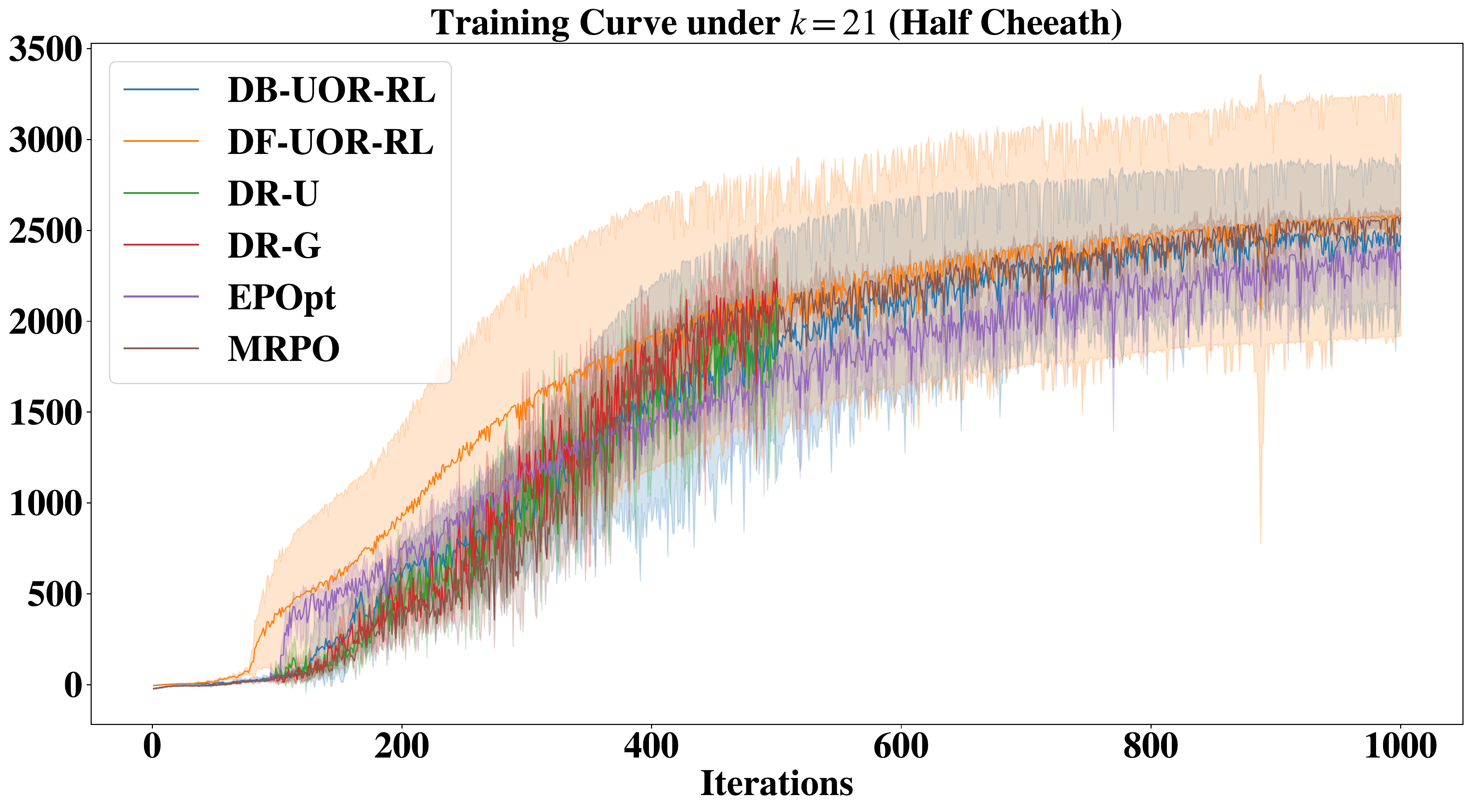}
\caption{Training Curves (Half Cheetah). In the three graphs from left to right, the UOR-RL algorithms are trained under $k=0,1$ and 21, and the y-axis  represent the average return of all trajectories, $\mathcal{E}_1$ and the average return of worst 10\% trajectories respectively. }
\vspace{-0.8cm}
\label{half_train}
\end{figure}

\begin{figure}[htbp!]
\centering
\includegraphics[width=0.32\textwidth]{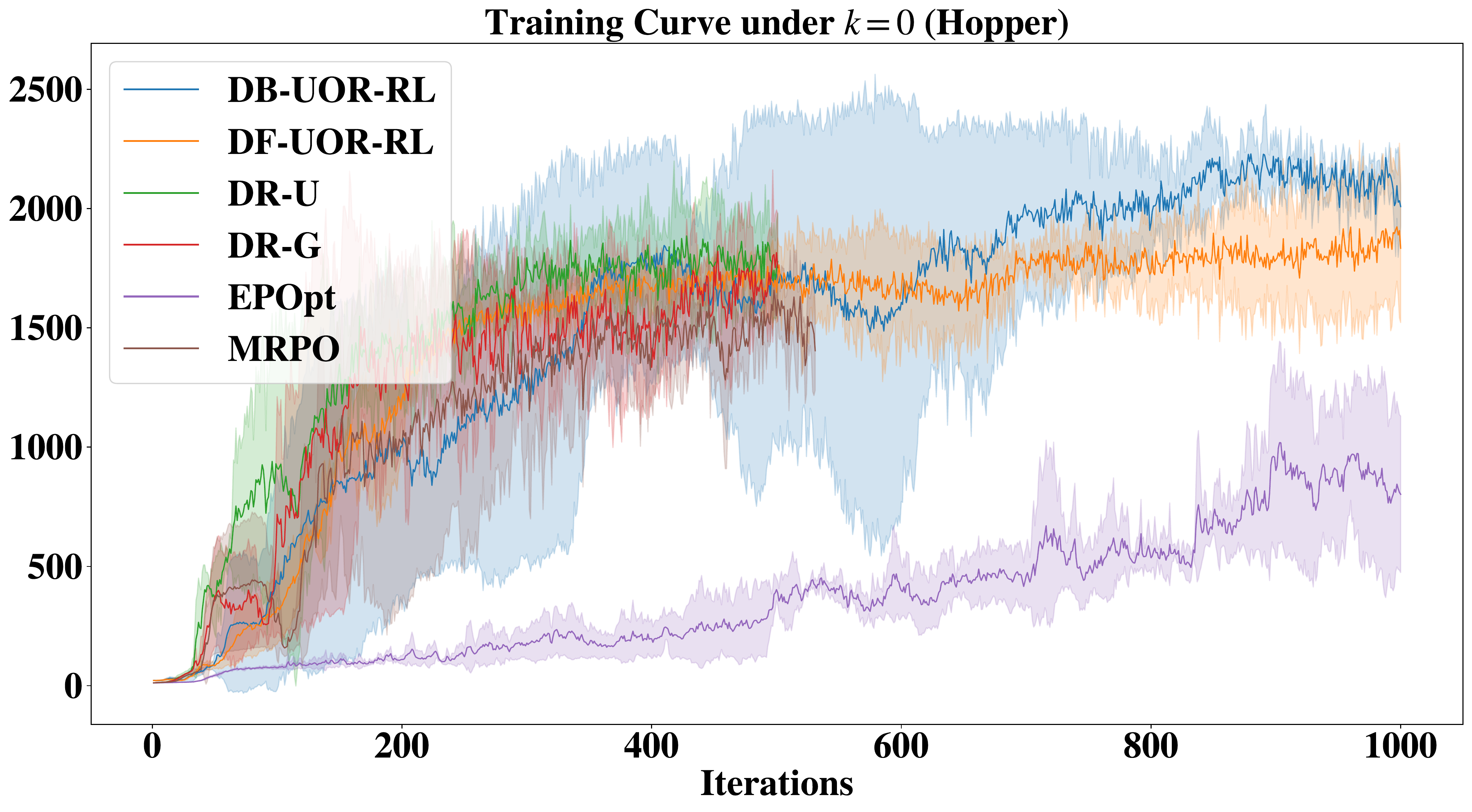}
\
\includegraphics[width=0.32\textwidth]{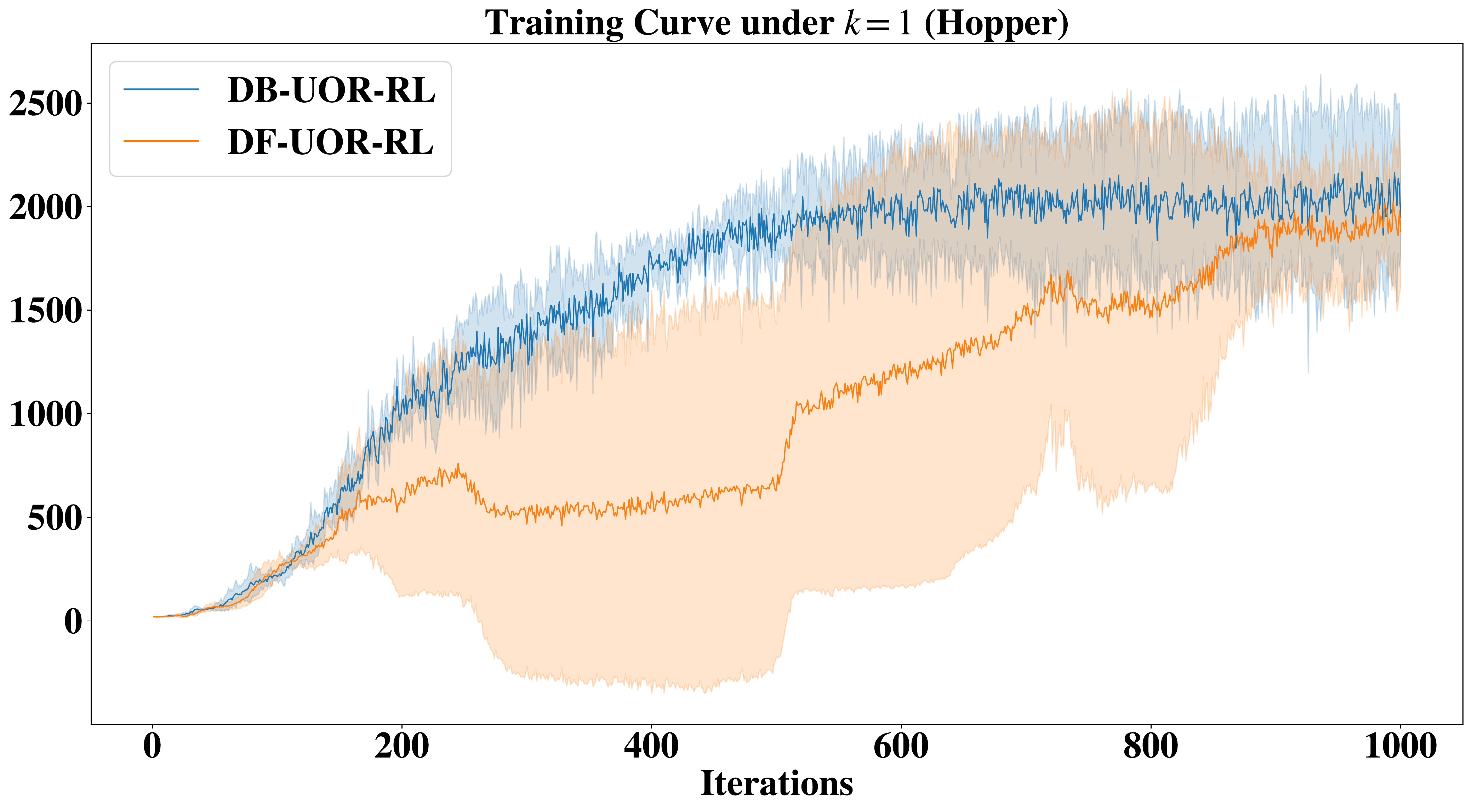}
\
\includegraphics[width=0.32\textwidth]{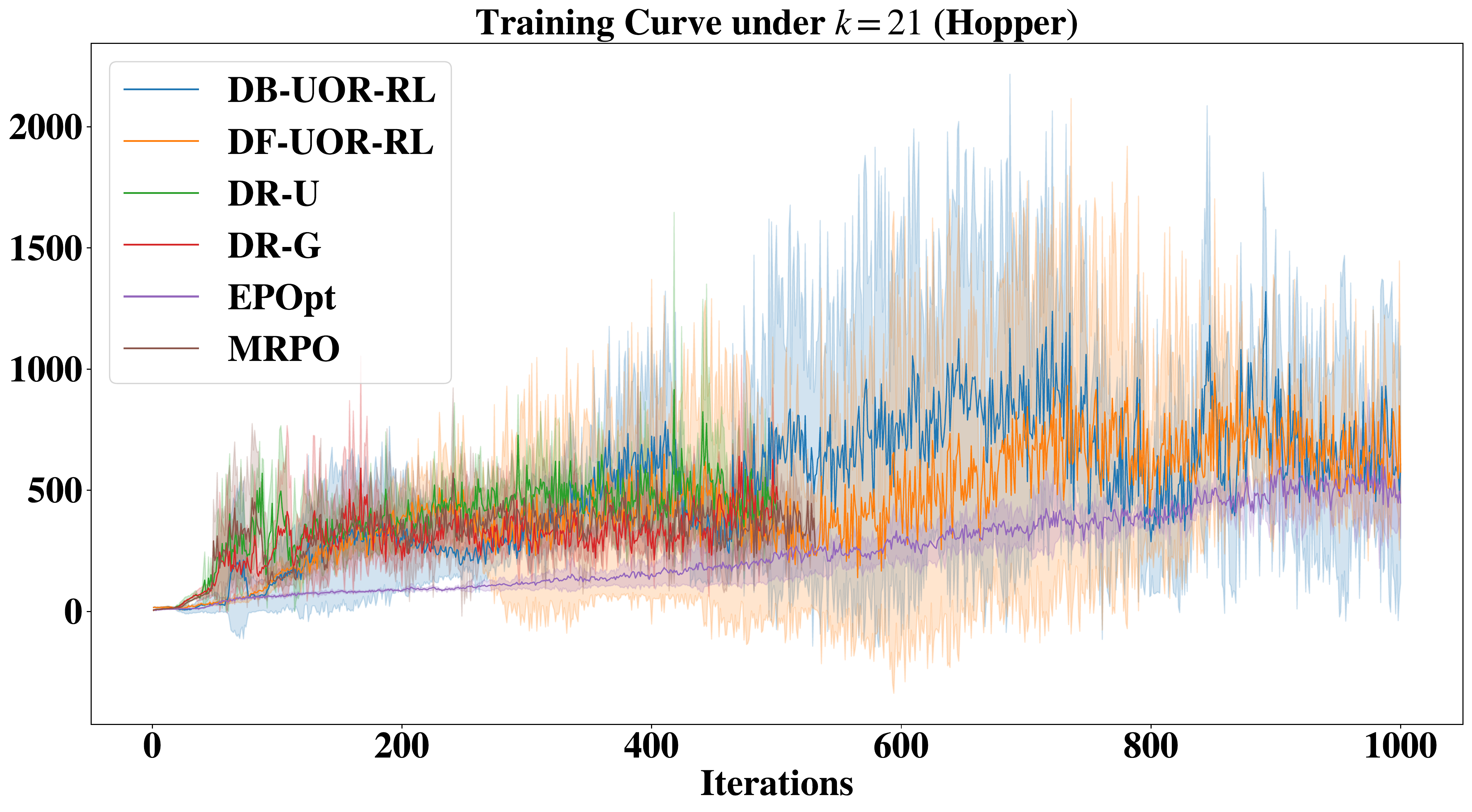}
\caption{Training Curves (Hopper). In the three graphs from left to right, the UOR-RL algorithms are trained under $k=0,1$ and 21, and the y-axis  represent the average return of all trajectories, $\mathcal{E}_1$ and the average return of worst 10\% trajectories respectively. }
% \vspace{-0.5cm}
\label{hopper_train}
\end{figure}

\begin{figure}[htbp!]
\centering
\includegraphics[width=0.32\textwidth]{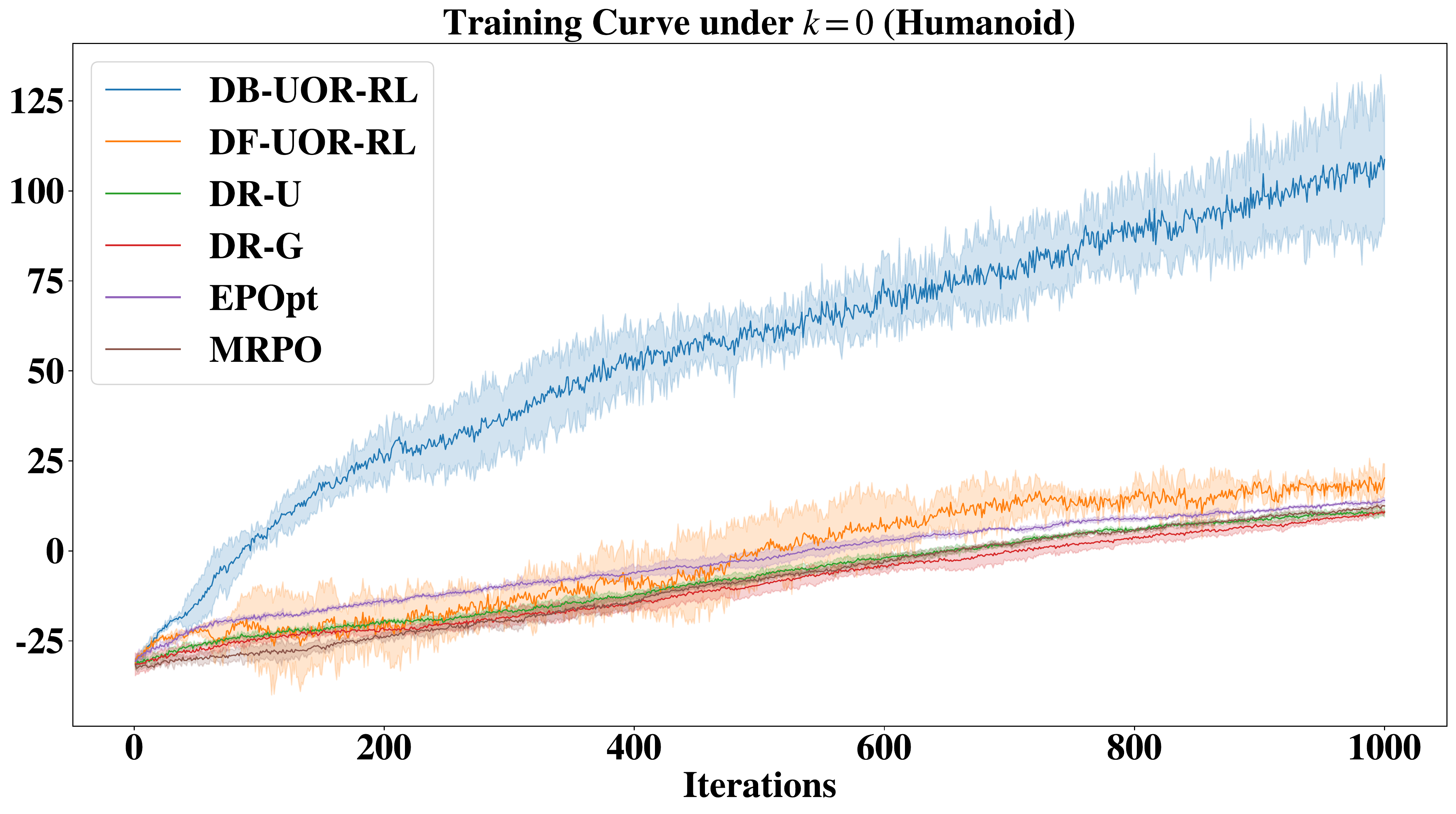}
\
\includegraphics[width=0.32\textwidth]{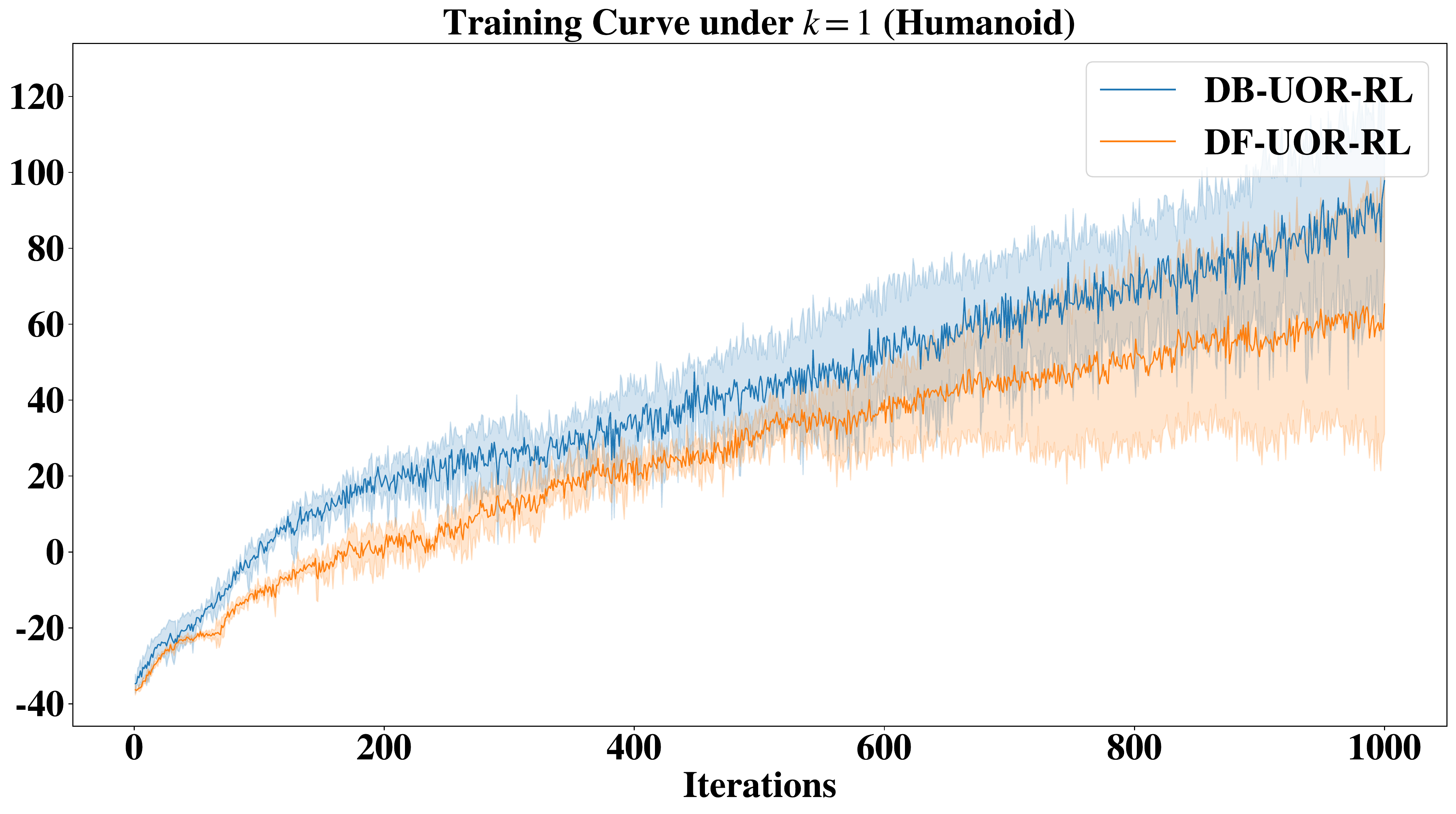}
\
\includegraphics[width=0.32\textwidth]{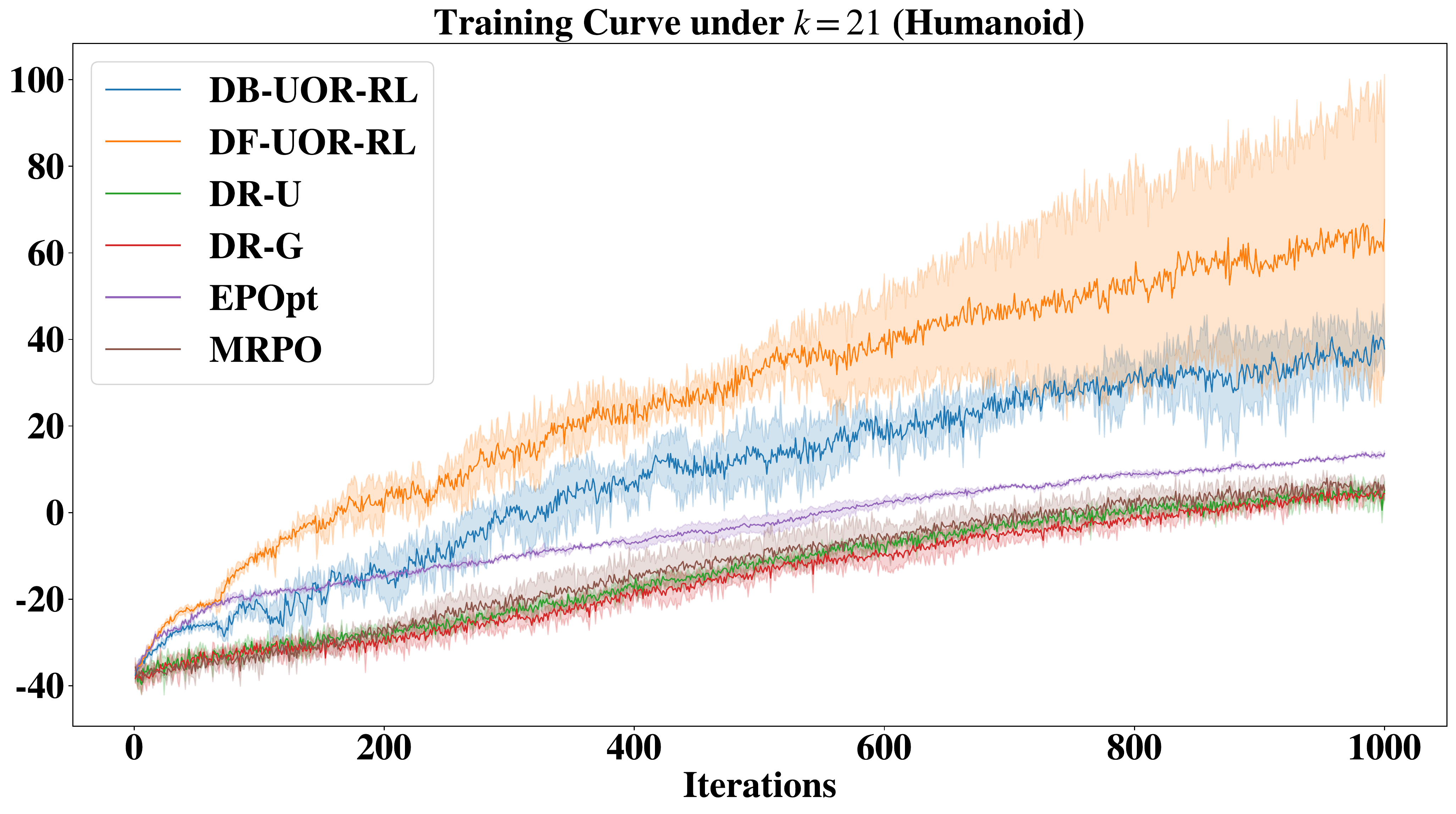}
\caption{Training Curves (Humanoid). In the three graphs from left to right, the UOR-RL algorithms are trained under $k=0,1$ and 21, and the y-axis  represent the average return of all trajectories, $\mathcal{E}_1$ and the average return of worst 10\% trajectories respectively. }
% \vspace{-0.5cm}

\label{humanoid_train}
\end{figure}

\begin{figure}[htbp!]
\centering
\includegraphics[width=0.32\textwidth]{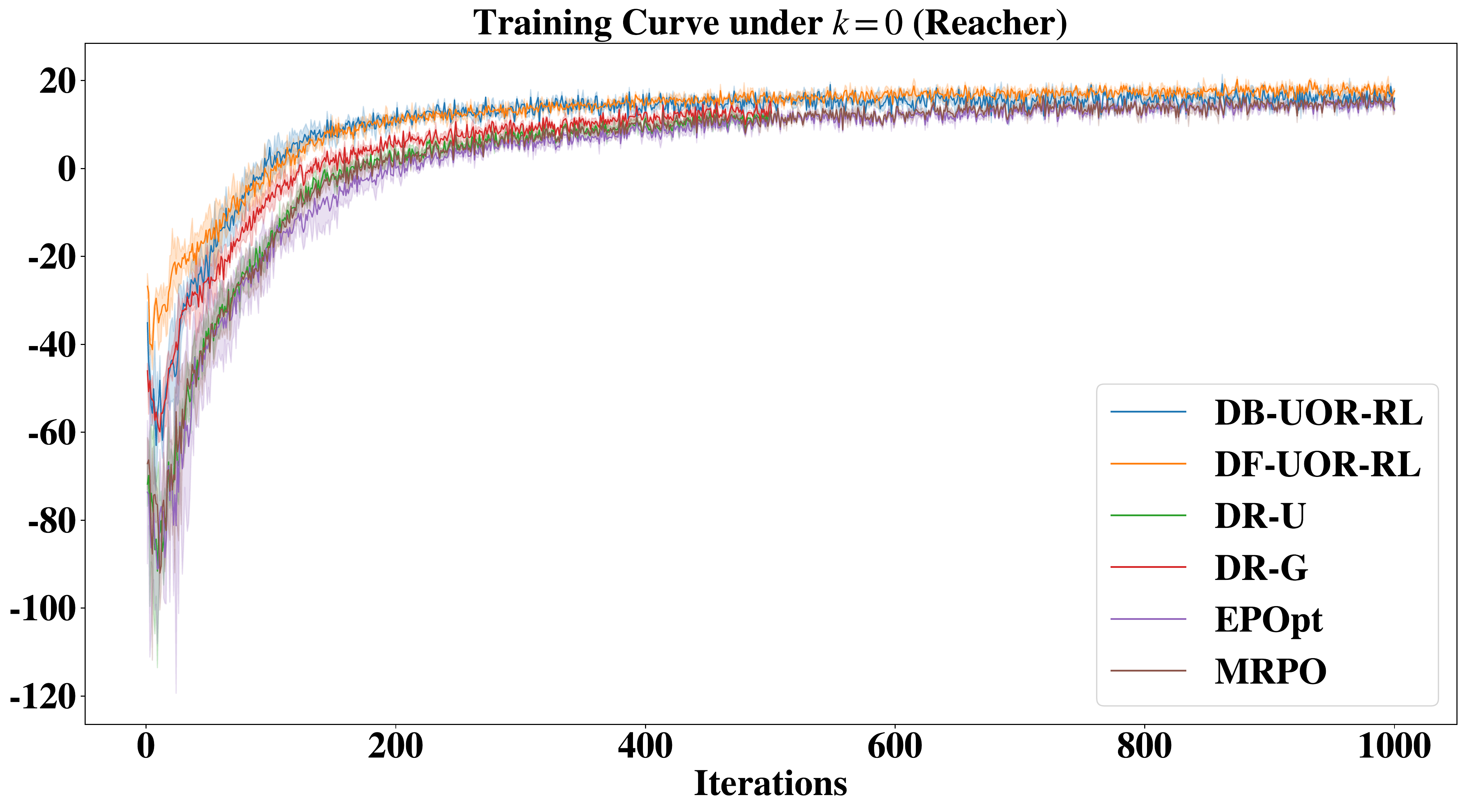}
\
\includegraphics[width=0.32\textwidth]{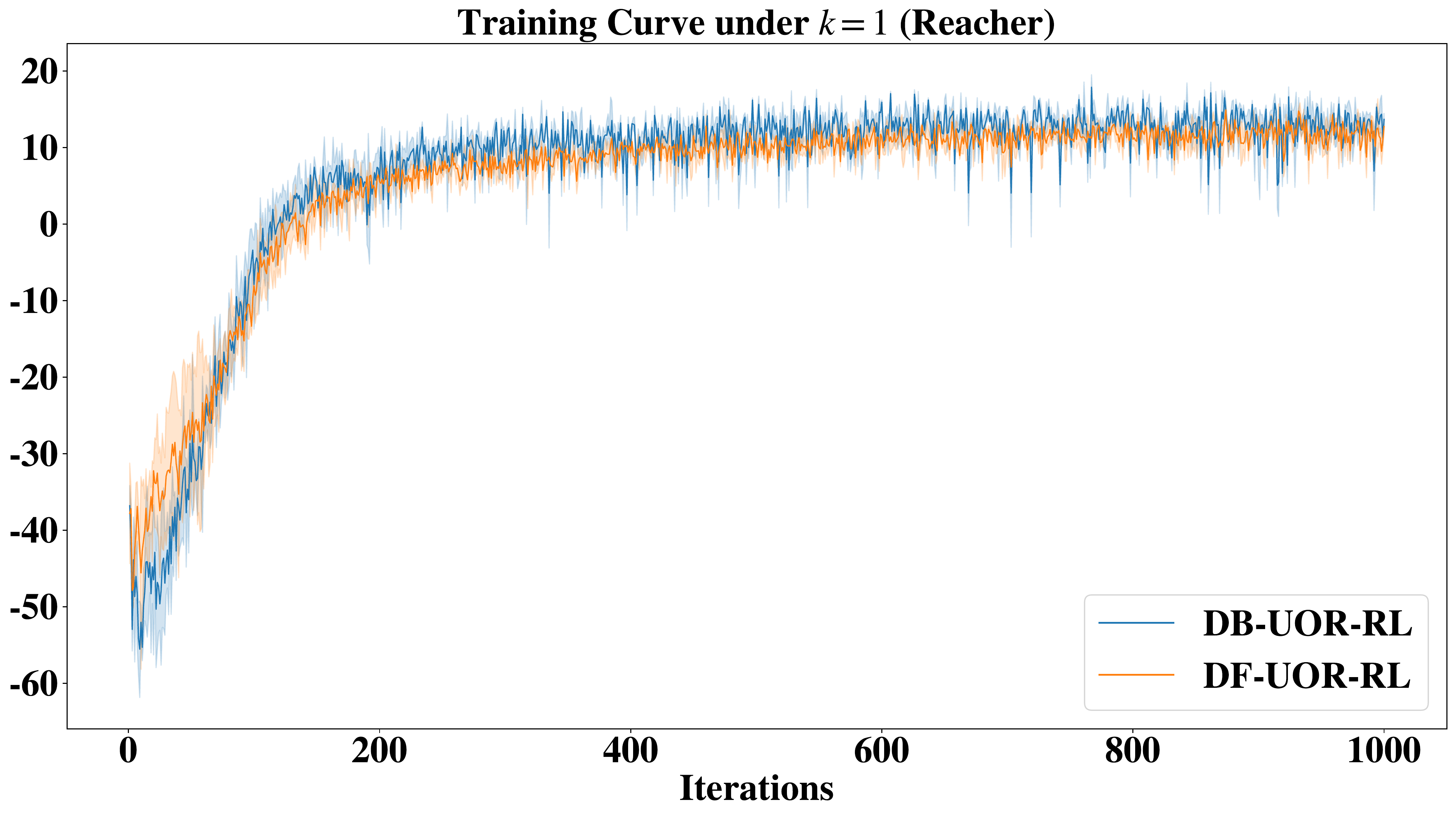}
\
\includegraphics[width=0.32\textwidth]{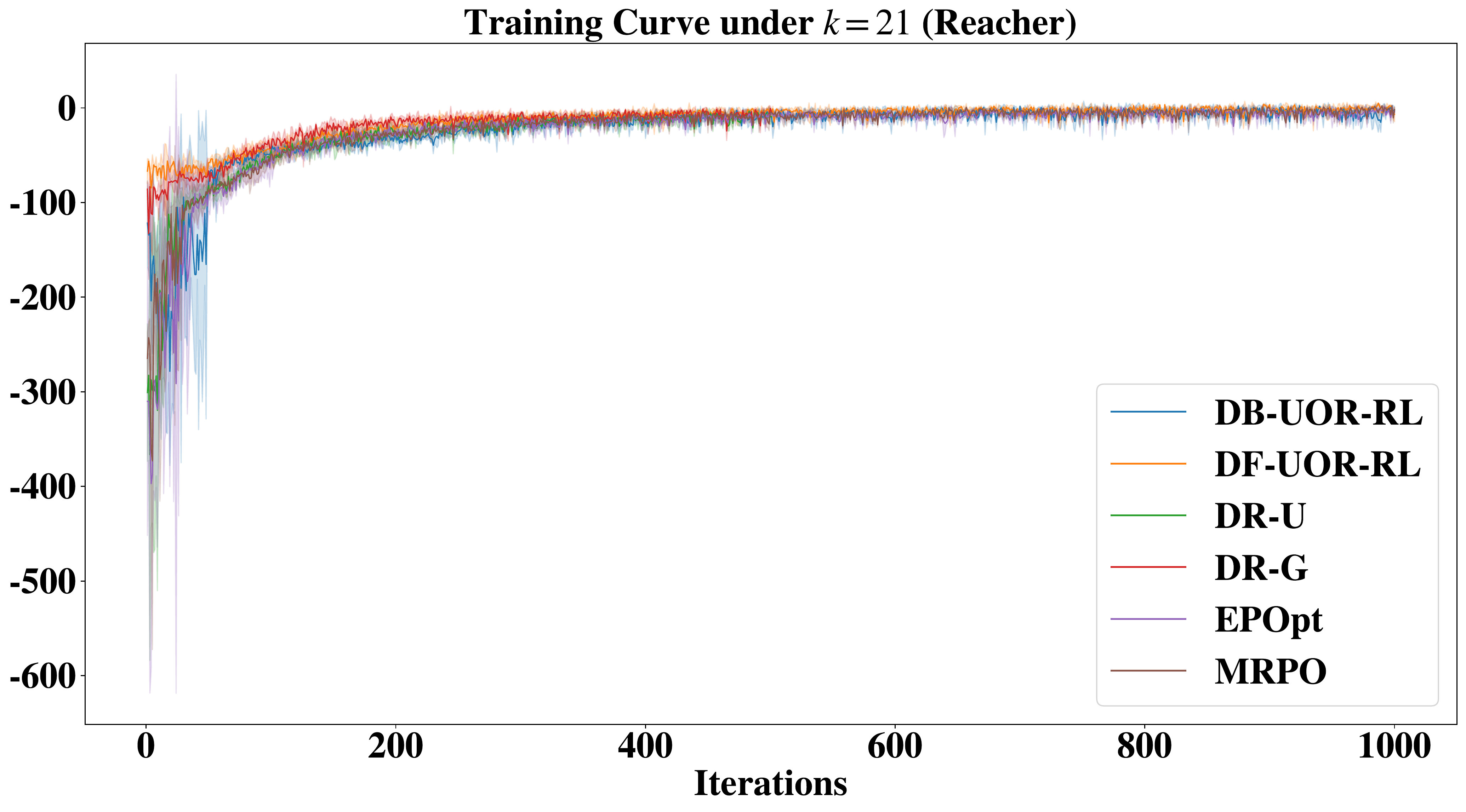}
\caption{Training Curves (Reacher). In the three graphs from left to right, the UOR-RL algorithms are trained under $k=0,1$ and 21, and the y-axis  represent the average return of all trajectories, $\mathcal{E}_1$ and the average return of worst 10\% trajectories respectively. }
% \vspace{-0.5cm}
\label{reacher_train}
\end{figure}

\begin{figure}[htbp!]
\centering
\includegraphics[width=0.32\textwidth]{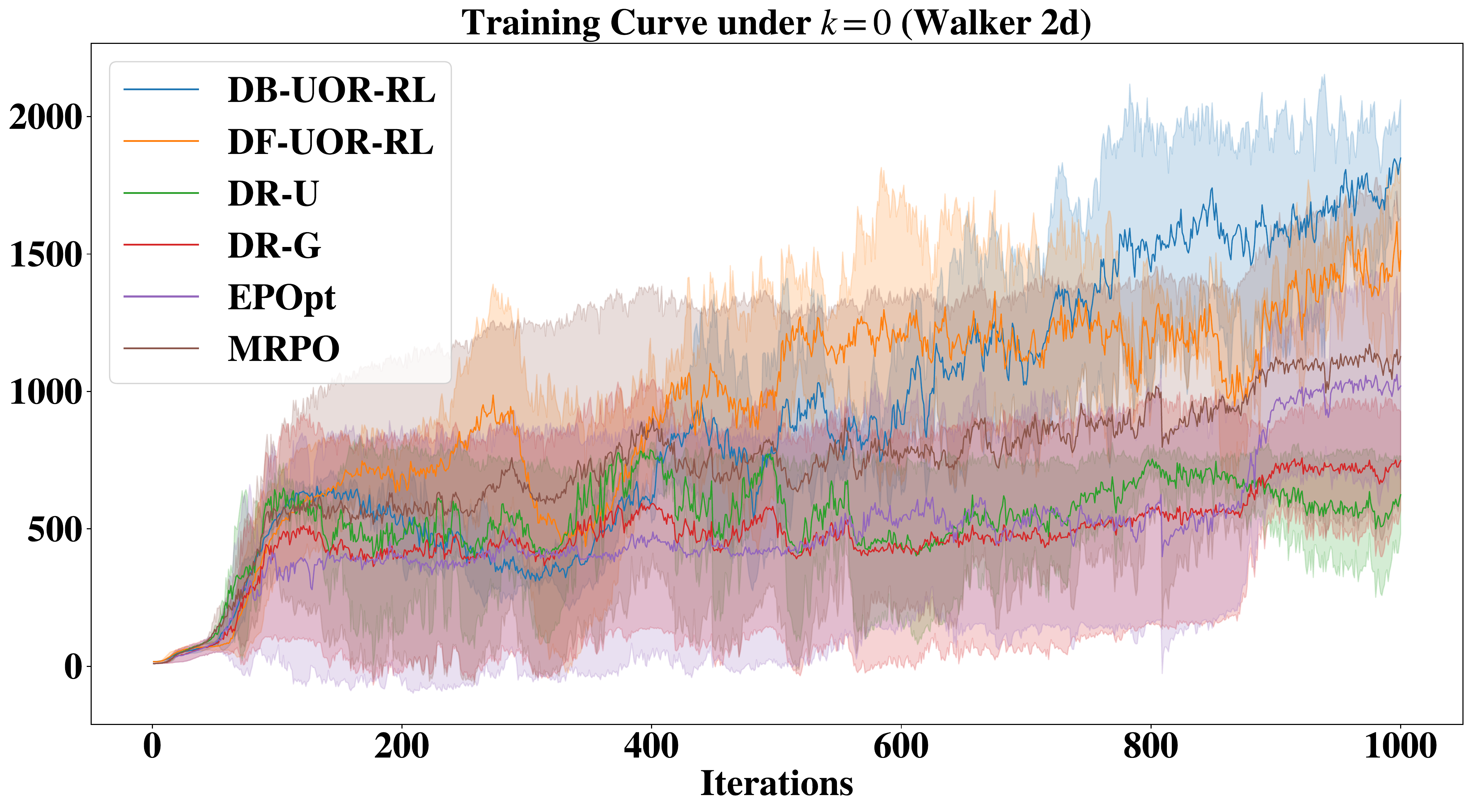}
\
\includegraphics[width=0.32\textwidth]{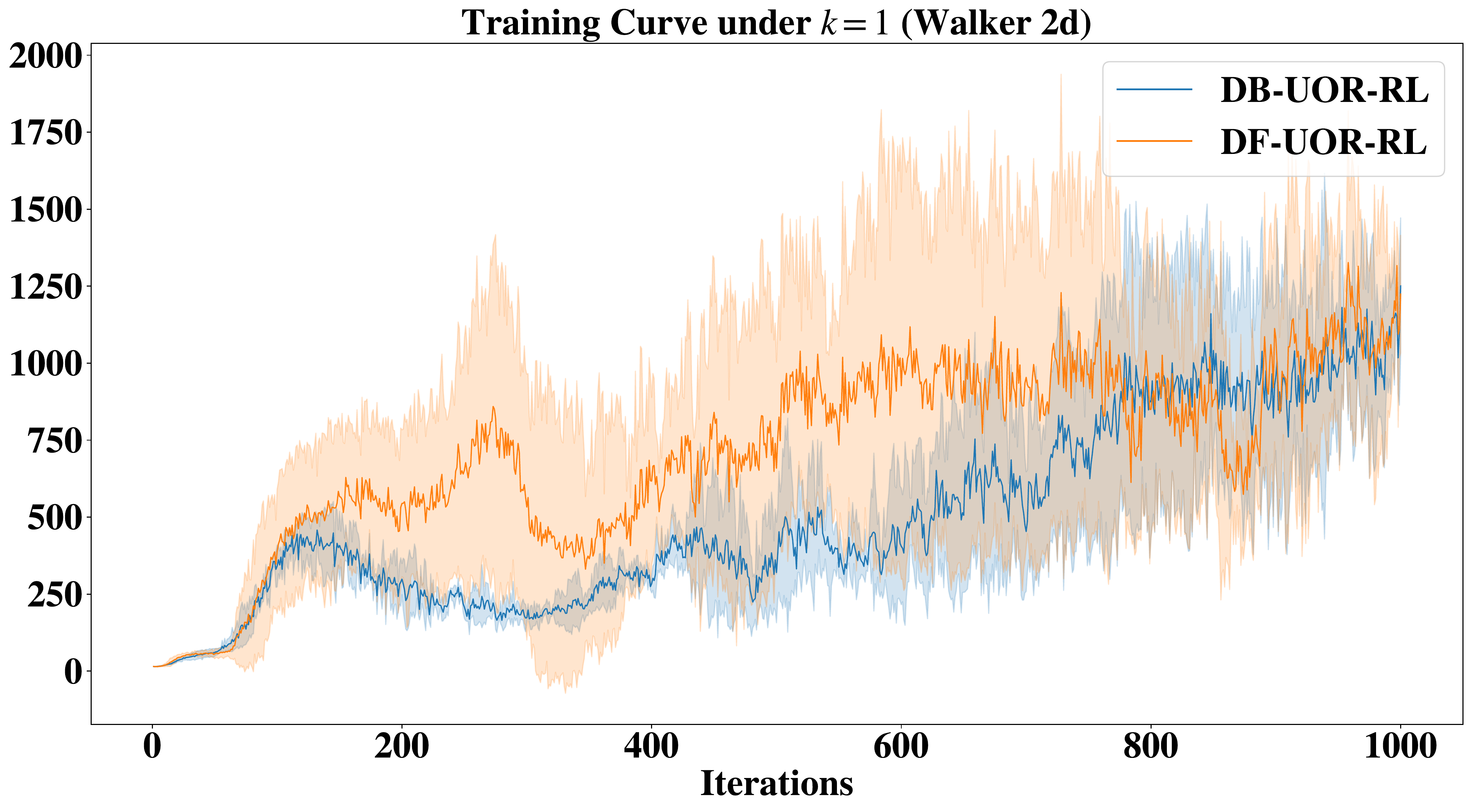}
\
\includegraphics[width=0.32\textwidth]{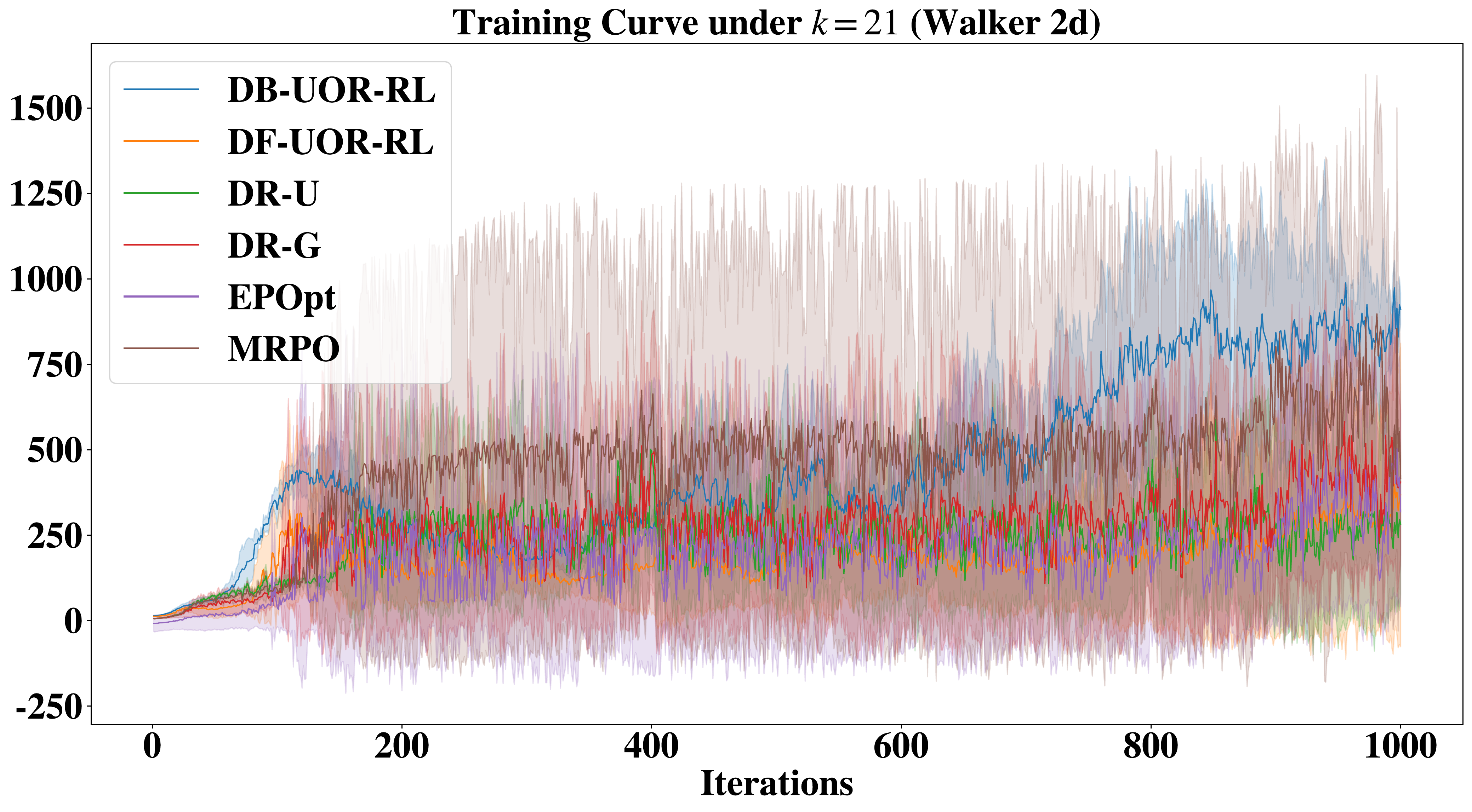}
\caption{Training Curves (Walker 2d). In the three graphs from left to right, the UOR-RL algorithms are trained under $k=0,1$ and 21, and the y-axis  represent the average return of all trajectories, $\mathcal{E}_1$ and the average return of worst 10\% trajectories respectively. }
% \vspace{-0.5cm}

\label{walker2d_train}
\end{figure}
\end{document}